\documentclass[11pt]{article}
\usepackage{graphicx}
\usepackage{hyperref}
\usepackage{amsmath, amssymb, enumerate}
\usepackage{natbib} 
\usepackage{cancel}
\usepackage{algorithm,tikz}
\usepackage{algpseudocode}
\usepackage{appendix}
\newcommand*\circled[1]{\tikz[baseline=(char.base)]{
            \node[shape=circle,draw,inner sep=2pt] (char) {#1};}}
\usepackage[applemac]{inputenc}
\usepackage[T1]{fontenc}
\usepackage{amsfonts}
\usepackage{float}
\usepackage{tikz,authblk}
\usetikzlibrary{shapes.geometric,arrows}
\usepackage{mathtools}

\usepackage{parskip}
\usepackage{multirow}
\usepackage{amsthm}
\newtheorem{theorem}{Theorem}
\newtheorem{lemma}{Lemma}

\newtheorem{definition}{Definition}
\newtheorem{assumption}{}[]

\newtheorem{remark}{Remark}
\newcommand{\E}{\mathbb{E}}
\DeclareMathOperator{\tr}{Tr}

\newcommand{\Slinv}{(\Sigma + \lambda I)^{-1}}
\newcommand{\ts}{\textsuperscript}
\newcommand{\fh}{\hat{f}}
\newcommand{\Sigmah}{\hat{\Sigma}}
\newcommand{\ahat}{\hat{a}}
\newcommand{\pih}{\hat{\pi}}
\newcommand{\tauh}{\hat{\tau}}

\newcommand{\argmax}{\text{arg\ max}}
\newcommand{\argmin}{\text{arg\ min}}

\newcommand\norm[1]{\left\lVert#1\right\rVert}
\newcommand\normx[1]{\Vert#1\Vert}
\begin{document}
\title{Kernel $\epsilon$-Greedy for Multi-Armed Bandits with Covariates}
\author{Sakshi Arya\textsuperscript{1} \text{and} Bharath K. Sriperumbudur\textsuperscript{2}}
\affil{\textsuperscript{1}Department of Mathematics, Applied Mathematics and Statistics,\\
Case Western Reserve University, Cleveland, OH 44106, USA.}
\affil{\textsuperscript{2}Department of Statistics,  
Pennsylvania State University\\
University Park, PA 16802, USA.\\
\texttt{sxa1351@case.edu, bks18@psu.edu}}

%  \author{Sakshi Arya and Bharath K. Sriperumbudur}
%  \affil{Department of Statistics,  
% Pennsylvania State University\\
% University Park, PA 16802, USA.\\
% \texttt{arya.sakshi44@gmail.com,bks18@psu.edu}}
\date{}
\maketitle
% \textcolor{red}{
% THINGS TO DO:
% \begin{itemize}
%     \item Change results to add the contribution of the initial randomization in the estimation and regret results. It could be like a $Ct_0$ kind of addition (assuming that $t_0$ is fixed and small as compared to $T$).
%     \item Real data analysis for Yahoo clicks dataset.
%     \item Compare $\gamma_t$ (mutual information) values for particular kernel choices in the literature and the regret rates achieved as a result with out results.
%     \item Add the condition (trivial but still) $\epsilon_t < \frac{1}{L-1} $ for all $t$.
% \end{itemize}
% }
% \textcolor{red}{Note:
% \begin{itemize}
%     \item Anything in red is the edit/correction I have made.
%     \item Anything in blue is highlighting what we had previously in our original manuscript, just so we can either remove that and replace it by the version in red, or keep it as is. 
%     \item Anything in cyan is a comment/question for Bharath.
% \end{itemize}}

\begin{abstract}
  We consider 
the $\epsilon$-greedy strategy for the multi-arm bandit with covariates (MABC) problem, where 
the mean reward functions are assumed to lie in a reproducing kernel Hilbert space (RKHS). We propose to estimate the unknown mean reward functions using an online weighted kernel ridge regression estimator, and show the resultant estimator to be consistent under appropriate decay rates of
the exploration probability sequence, $\{\epsilon_t\}_t$, and 
regularization parameter, $\{\lambda_t\}_t$. 
Moreover, we show that for any choice of kernel and the corresponding RKHS, we achieve a sub-linear regret rate depending on the intrinsic dimensionality of the RKHS. Furthermore, we achieve the optimal regret rate of $\sqrt{T}$ under a margin condition for finite-dimensional RKHS.
    %upon restricting the RKHS to a finite dimensional space and assuming the margin condition.
\end{abstract}
\textbf{MSC 2010 subject classification:} Primary: 62L10; Secondary: 62G05, 68T05.\\
\textbf{Keywords and phrases:} Contextual bandits, reproducing kernel Hilbert space, covariance operator, $\epsilon$-greedy, cumulative regret, inverse probability weighting, kernel ridge regression
\setlength{\parskip}{4pt}

\section{Introduction}
Sequential decision-making in real time is increasingly becoming important in various applications, such as clinical trials \citep{bather1985allocation, villar2015multi}, news article recommendation \citep{li2010contextual}, and mobile health \citep{nahum2017just}. In all such problems, the decision-maker is faced with several alternatives, from which they have to make a series of choices (referred to as \emph{arms}) sequentially, based on the information available at any given time. In doing so, the decision-maker takes into account additional information or covariates (characteristics) that help in making informed decisions. As far as our knowledge extends, no prior research has explored the utilization of kernel methods within a nonparametric Multi-Armed Bandit with Covariates (MABC) framework \citep{rigollet2010nonparametric}. This approach differs from the conventional kernel bandits framework, primarily in terms of the underlying modeling framework. In a treatment allocation problem, this can be described as follows: Given finitely many competing treatments for a disease, the decision-maker (physician) chooses the treatment best suited for individual patients as they arrive, and each allocated treatment results in a \textit{reward} (outcome). While doing so, the decision-maker takes into account the patient's covariates and information available about previous patients with the same disease, with the eventual goal of maximizing the total reward accumulated over a period of time. The technical challenge in achieving this is two-fold: 1) learning the relationship between the covariates and optimal arms, and, 2) balancing the \textit{exploration-exploitation trade-off}, which arises due to the sequential (or online) nature of the problem. In other words, in a sequential setup, at each time point, the physician has to effectively identify the best treatment (exploration) and treat patients as effectively as possible during the trial (exploitation). Since there has been a substantial number of advancements in the contextual bandit problem in recent years, we refer the reader to \citet{lattimore2020bandit} for a detailed description of recent developments in this area and \citet{tewari2017ads} for a comprehensive survey of both parametric and nonparametric methods in both contextual bandits and MABC.

In this paper, we consider a kernelized MABC framework, where the relationship between the rewards and covariates for each arm is modeled by functions in a reproducing kernel Hilbert space (RKHS), and study a popular heuristic in multi-armed bandit problems known as the annealed  $\epsilon$-greedy strategy \citep{sutton2018reinforcement}. This strategy allocates arms based on a randomized strategy to balance the exploration-exploitation trade-off, with a careful reduction in exploration over time. For example, in a two-armed bandit, the $\epsilon$-greedy strategy chooses the current best-performing arm with probability $1-\epsilon$ and makes a random decision with a small probability $\epsilon$. In the annealed version of the algorithm, $\epsilon$ is a non-increasing function of time. This algorithm falls in the broad category of `algorithms with myopic exploration', which are easy to implement and could result in good empirical performance in some situations with appropriate exploration probability choices \citep{bietti2021contextual, mnih2015human}. In practice, they are often selected as the top choices due to their simplicity. However, they have not been studied extensively in the literature.  Recently, \citet{dann2022guarantees} studied the $\epsilon$-greedy strategy in the more general reinforcement learning setup and provided theoretical guarantees in terms of what they define as the myopic exploration gap, which is a problem-dependent quantity. Using $\epsilon$-greedy strategy, they achieve the optimal regret rate of $\tilde{O}(T^{2/3})$ for contextual bandits, which matches the rate we obtain for our algorithm when the RKHS is finite-dimensional. \citet{chen2021statistical} study the $\epsilon$-greedy strategy for the linear bandit problem (parametric MABC with a linear regression framework) and establish the regret rate of $\tilde{O}(\sqrt{T})$, which our work recovers when the RKHS is finite-dimensional and under the margin condition.

\subsection{Contributions}
The main contribution of this work is in developing a theoretical understanding of kernelized $\epsilon$-greedy algorithm in a nonparametric multi-armed bandits framework. More specifically: 
\begin{itemize}
 \item We propose an inverse probability weighted kernel ridge (IPWKR) regression type of online estimator for the mean reward functions in Section~\ref{sec: IPWR_estimator}, whose implementation details are provided in Section~\ref{sec: implementation}. Such estimators have been studied in the context of linear bandits for mitigating the problem of estimation bias in adaptively collected data \citep{dimakopoulou2019balanced}. Using IPWKR estimator, we propose a kernel $\epsilon$-greedy algorithm for multi-armed bandits with covariates and provide regret bounds.
 We highlight here that the inverse probability weights appearing in IPWKR are deterministic known quantities that involve user-determined exploration probabilities.
 \item In Section~\ref{sec: estimation_error}, we establish upper bounds on the estimation error (see Theorem~\ref{thm: theorem1}) for the proposed IPWKR estimator, which we specialize to the setting of finite-dimensional RKHS in Theorem~\ref{thm: theorem_finiteRKHS}. These results hold for specific conditions on the exploration probability sequence $\{\epsilon_t\}_t$ and choices of the regularization parameter sequence $\{\lambda_t\}_t$. As a comparison, in the linear bandit framework, \citet{chen2021statistical} propose an online weighted least squares (WLS) estimator similar to our proposed IPWKR estimator but without regularization. In fact, when the kernel is linear, our proposed estimator can be seen as a dualized representation of their online WLS estimator if $\lambda=0$. In Section \ref{sec: estimation_error} (see Theorem~\ref{thm: theorem_finiteRKHS} and Remark~\ref{rem:0}), we show that our consistency result is stronger than that in \citep[Proposition 4.1]{chen2021statistical} in the sense that we achieve consistency in estimation for large ranges of the exploration probabilities, $\{\epsilon_t\}_t$, i.e.,  for decaying choices of $\{\epsilon_t\}_t$ faster than those considered in \citet{chen2021statistical}. Interestingly, for these choices of $\{\epsilon_t\}_t$, we obtain the same convergence rate as in \citet{chen2021statistical}---obtained for linear MABC---even when the true regression functions are non-linear, as long as the RKHS is finite-dimensional. Also, we highlight that, compared to the existing literature \citep{valko2013finite, zenati2022efficient}, in a finite-dimensional setting, our analysis provides an {explicit time-dependent choice} for regularization parameter, $\{\lambda_t\}_t$, circumventing the need to tune it when implementing the algorithm.
     \item In Theorem~\ref{thm: regret_infinite_case_pre} of Section \ref{sec: regret_analysis}, we establish finite-time regret bounds for kernel $\epsilon$-greedy strategy for the MABC problem with finitely many arms when the mean reward functions are assumed to be in an RKHS. These regret bounds are sub-linear for all choices of bounded, positive definite, and symmetric kernels. Compared to the results in the literature, this is a significant result since the 
     % This is promising as in the existing literature for kernel bandits, the regret upper bound for the 
     Mat\'ern kernel can lead to linear regret for some choices of the kernel parameters \citep{vakili2021open, scarlett2017lower}. In fact, our finite-time regret bound matches the state-of-the-art upper bound for $\epsilon$-greedy strategy when the RKHS is finite-dimensional (Theorem~\ref{thm: regret_finite_case_pre}), both with and without the \emph{margin condition}---the margin conditions ensures that there is sufficient gap between the rewards for different arms---, which are $\tilde{O}(T^{1/2})$  and $\tilde{O}(T^{2/3})$, respectively. When the RKHS is infinite-dimensional, the bounds are controlled by the intrinsic dimensionality of the RKHS---in turn, controlled by the decay rate of the eigenvalues of the covariance operator---and a source condition, which captures the smoothness of the target.
 \end{itemize}

\subsection{Related work}
Multi-Armed Bandits with Covariates (MABC) and contextual bandits in a parametric framework, especially linear bandits have been extensively studied in the bandit literature \citep{lattimore2020bandit}. In a linear bandit framework, every arm corresponds to a known, finite-dimensional context, and its expected reward is assumed to be an unknown linear function of its context.
Upper Confidence Bound (UCB) and Thompson sampling strategies are the most commonly used bandit algorithms. While the UCB uses the \textit{optimism in the face of uncertainty} idea where the exploration-exploitation trade-off is balanced by creating confidence sets on the unknown reward functions, Thompson sampling is a randomized algorithm and a popular heuristic based on Bayesian ideas. Both these types of algorithms are extremely popular, well-studied, and enjoy tight regret guarantees \citep{dani2008stochastic,li2010contextual,NIPS2011_e1d5be1c,pmlr-v22-agarwal12}. On the other hand, there are relatively fewer works on $\epsilon$-greedy strategy except for the epoch-greedy version of \cite{langford2007epoch} and the mostly exploration-free algorithms based on \cite{bastani2018mostly}, which have gained a lot of popularity recently. In a nonparametric framework, $\epsilon$-greedy strategies and modifications of the same have been studied by \citet{yang2002randomized, qian2016kernel, ARYA2020108818, qian2023adaptive}.

In the contextual bandit framework, there are also two types of setups that are considered in the literature: (a) continuous action space linear bandits: action (arm) space is the same as the context (covariate) space \citep{NIPS2011_e1d5be1c,dani2008stochastic,rusmevichientong2010linearly} and (b) finite action space linear bandits: finite action (arm) set which is different from the context space \citep{li2010contextual, chen2021statistical}.  
Under these two steps, infinite-dimensional extensions of contexts and/or arms have been considered in both the frequentist and Bayesian perspectives. In the frequentist perspective, \cite{valko2013finite} propose a \textit{KernelUCB} algorithm for finite action space, which is obtained by kernelizing the \textit{LinUCB} and \textit{SupLinUCB} algorithms of \cite{li2010contextual, chu2011contextual}, and \cite{auer2002finite}. They give a
bound on the regret in terms of a data-dependent quantity, the effective dimension, $\tilde{d}$. In the Bayesian perspective, \cite{srinivas2010gaussian} propose the Gaussian Process (GP)-UCB for the context-free stochastic bandit problem, which assumes that the reward function is drawn from a GP prior. \cite{krause2011contextual} generalize the GP-UCB algorithm by taking context information into account in the decision-making process. \cite{gopalan2014thompson} study the stochastic multi-armed bandit problem with continuous action space and propose an improved version of GP-UCB, which they call \textit{IGP-UCB}. In this work, they also propose a nonparametric version of Thompson sampling, \textit{GP-Thompson sampling}. The regret bounds in all of the Gaussian process bandits line of work are in terms of the quantity, $\gamma_t$,  which is the maximum information gain at time $t$ and depends on the choice of kernels that define the underlying RKHS containing the reward functions. Moreover, in \cite{valko2013finite}, the regret rates are controlled by the intrinsic dimension,  $\tilde{d}$, which is shown to be closely related to $\gamma_t$.  \cite{zhou2020neural} propose a neural-net-based algorithm, called NeuralUCB, which uses a neural network-based random feature mapping to construct an upper confidence bound (UCB). More recently, \cite{zenati2022efficient} propose an efficient contextual UCB type  algorithm for computational efficiency. This algorithm relies on incremental Nystr\"om approximations of the joint kernel embedding of contexts and actions. Furthermore, \cite{janz2020bandit} and \cite{vakili2021information} provide improved regret bounds for GP bandits with M\'atern kernels. \cite{scarlett2017lower} and \cite{cai2021lower} provide lower bounds for the Gaussian Process bandit optimization problem with squared exponential and M\'atern kernel for the contextual bandit problem with continuous action space. Another approach for kernelized contextual bandits is related to experimental design and also aims at optimal pure exploration in kernel bandits as has been studied by \cite{camilleri2021high} and \cite{zhu2021pure}. A partial list of some of the other unrelated work in nonparametric MABC/contextual bandit problem includes \citet{yang2002randomized, rigollet2010nonparametric, magureanu2014lipschitz, hu2020smooth, kleinberg2008multi, slivkins2011contextual, zhou2020neural, liu2023adaptation, whitehouse2023sublinear}. 

While UCB and Thompson sampling algorithms for kernel contextual bandits have received considerable attention in the recent past, to the best of our knowledge, the kernelized version of the $\epsilon$-greedy algorithm has not been studied previously. 
 % \textcolor{blue}{Our setup is similar to the agnostic setup of \cite{valko2013finite} for a contextual bandit framework with finitely many arms, where we propose a kernelized version of the $\epsilon$-greedy algorithm, which can be seen as a nonparametric extension to the linear $\epsilon$-greedy algorithm studied by \cite{chen2021statistical}}.
{Note that our approach is inspired by the nonparametric bandit framework studied by \citep{rigollet2010nonparametric, yang2002randomized}, diverging from the conventional `kernel bandits' modeling framework. The key distinctions encompass two fundamental aspects:
1) we operate under the assumption of independent and identically distributed (i.i.d.) covariates. This differs from the standard `kernel bandits' framework, where the covariates or contexts depend on the arms.
2) In our nonparametric model, we adopt distinct regression models for each of the $K$ arms, resulting in the need to estimate $K$ different mean reward functions. In contrast, the `kernel bandits' framework \citep{valko2013finite} features a singular model with a sole mean reward function `$f$' encapsulating the relationship between rewards and arm-dependent contexts.
The suitability of either model depends on the intended application. Our framework finds its motivation in scenarios that involve treatment decision-making. Here, the context encapsulates patient characteristics that remain consistent across treatments. Modeling the relationship between responses and covariates separately for each treatment aligns with the nature of decision-making in this practical setting.}
 
  Our main contribution is a theoretical analysis of this kernelized $\epsilon$-greedy approach validated by some numerical results. The regret bounds we achieve are different from the previous line of work, as our results do not depend on the data-dependent quantities such as $\tilde{d}$ in \cite{valko2013finite} or maximum information gain $\gamma_t$ as in the GP-bandits line of work. One can quantify this information gain for a specific kernel choice and find the corresponding regret rate, see \cite{scarlett2017lower}. Therefore, for some choices of kernel parameters for the M\'atern kernel, it has been observed that one could obtain a linear cumulative regret. On the other hand, our results depend on the intrinsic dimensionality of the RKHS stemming from the assumption on the rate of the  eigenvalue decay for the covariance operator and the source condition for the space in which the true mean reward functions are assumed to belong. As a result, we always obtain a sub-linear regret rate irrespective of the kernel choice. Note that, the $\epsilon$-greedy algorithm that we propose uses an inverse probability-weighted online kernel ridge regression estimator. The inverse probability weighting is reminiscent of the inverse propensity score weighted algorithms considered in parametric frameworks to handle model misspecifications.  These are known as balanced bandit algorithms \citep{dimakopoulou2019balanced} where  each observation is divided by its propensity score to correct for estimation bias.
\cite{dimakopoulou2019balanced} propose the balanced linear UCB and balanced Thompson sampling algorithms, establish their corresponding regret rates, and assess empirical performances under model misspecifications. \cite{bogunovic2021misspecified} study misspecified Gaussian Process bandit optimization, but for continuous action space and they establish regret bounds in terms of the amount of model misspecification.  \cite{chen2021statistical} consider an `balanced' (weighted) linear $\epsilon$-greedy strategy for handling model-misspecification with weighted least squares estimator. Our work can be thought of as an extension to their work offering flexible and nonparametric modeling for the relationship between the rewards and covariates, where we study a kernelized version of the weighted linear $\epsilon$-greedy algorithm. 
%\vspace{-0.6cm}
\subsection{Organization }
%\vspace{-0.2cm}
The rest of the paper is organized as follows. In Section \ref{sec:notations}, we define the notations used in the rest of the paper. In Section \ref{sec: setup}, we introduce the setting of MABC and the online regression framework, along with the definition of regret that is used to assess the performance of the proposed algorithm. The kernel $\epsilon$-greedy strategy and the online kernel ridge-regression estimator are presented in Section \ref{sec: algorithm}. In Section \ref{sec: implementation}, we provide an implementable version of the regression estimator that is then employed in the proposed algorithm for empirical evaluation on synthetic data sets in Section \ref{sec: simulations}. The convergence rates for the estimation error and regret are provided in Sections \ref{sec: estimation_error} and \ref{sec: regret_analysis}, respectively, for both the finite and infinite-dimensional settings. %In Section \ref{sec: regret_analysis}, we present the regret rates for both the infinite-dimensional and finite-dimensional RKHSs. 
Under an additional assumption of the margin condition, improved regret bounds are presented for the proposed algorithm in Section \ref{sec: margincondition}. Finally, all the proofs are provided in Section \ref{sec: proofs}.

%%%%%%%%%%%%%%%%%%%%%%%%%%%%%%%%%%%%% Section Background %%%%%%%%%%%%%%%%%%%%%%%%%%%%%%%%%%%%%%%%%%%%%%%%%%%%%%
\section{Background and problem setup} \label{sec: background}
In this section, we introduce the notations, followed by the problem setup of the sequential decision-making framework of MABC.
\subsection{Notations } \label{sec:notations}  
 For a Hilbert space $\mathcal{H}$,  $\langle f, g \rangle_{\mathcal{H}}$ denotes the inner product of $f,g\in \mathcal{H}$. We denote $\Vert\cdot\Vert$ or $\Vert\cdot\Vert_\mathcal{H}$ to denote the corresponding norm in $\mathcal{H}$. For $h \in \mathcal{H}$, we use $\normx{h}_\mathcal{H} = \sqrt{\langle h,h \rangle_{\mathcal{H}}}$ to denote the RKHS norm and $\Vert A\Vert_\infty$ denotes the operator norm of a bounded operator $A$. For operators $A$ and $B$ on $\mathcal{H}$, $A \preceq B$ if and only if $B - A$ is a positive definite operator. For two real numbers $x$ and $y$, $x \lesssim y$  denotes that  $x$ is less than or equal to $y$ up to a constant factor, and $\otimes$ denotes the tensor product. The notation $X \perp Y$ for two random variables, $X$ and $Y$, translates to $X$ is independent of $Y$. $\tilde{O}(\cdot)$ denotes the order of approximation (big-$O$) with some additional constant terms or terms of logarithmic order in time. 
  
%%%%%%%%%%%%%%%%%%%%%%%%%%%%%%%%%%%%%%%%%%%%%%%%%%%%%%%
\subsection{Problem Setup}
\label{sec: setup}
In the MABC problem with finitely many \textit{arms}, the decision-maker has $L \in \mathbb{N}$ competing choices of arms (or actions), say $\mathcal{A}:= \{1, \hdots, L\}$, and have to choose arms sequentially over time, $\mathcal{T} = \{1,\hdots, T\}$, while using the contextual information available at each time point, with $T$ being the horizon. The contextual information can be thought of as patient characteristics in a treatment allocation problem or user information in a recommender system application. That is, at each time point $t \in \mathcal{T}$, the decision-maker observes contextual information (covariate), $X_t \in \mathbb{R}^{d}$, from an underlying probability distribution $\mathcal{P}_X$. Now, based on the available information until time $t$, the decision-maker then chooses an arm $a_t$ from a finite set of arms, $\mathcal{A} = \{1, \hdots, L\}$. Choosing (pulling) the arm $a_t$ results in a \textit{reward}, $y_t \in \mathbb{R}$. The reward can be thought of as a quantitative outcome of assigning the arm at that time. For instance, it could mean the amount of benefit caused by assigning a particular treatment to a patient at a given time. 

 In order to make informed decisions, the decision-maker must understand the relationship between the rewards, $\{y_t\}_{t=1}^T$, and covariates, $\{X_t\}_{t = 1}^T$, for each arm in $\{1,\hdots,L\}$. This relationship, usually stochastic in nature, can naturally be formulated as a regression problem. 
 In this work, we assume a nonparametric regression framework to model this relationship. 
 For each arm, $i \in 1,\hdots, L$, we consider the following regression model:
\begin{align}
y_t = f_i(X_t)+ e_t, 
\label{eq: linear_model}
\end{align}
where the corresponding mean reward function $f_i \in \mathcal{H}$, with $\mathcal{H}$ being a reproducing kernel Hilbert space (RKHS) with reproducing kernel $k: \mathcal{X} \times \mathcal{X} \rightarrow \mathcal{H}$, where $\mathcal{X}$ is a separable topological space. We assume that $k$ is bounded and continuous, i.e., there exists a $\kappa > 0$, such  that $\sup_{x\in\mathcal{X}}k(x,x) \leq \kappa$. Note that, by the reproducing property, $f_i(x_t) = \langle f_i, k(\cdot, x_t) \rangle_{H}$.  We make the following model assumptions: 
% \textcolor{red}{I prefer we use $(\mathcal{A}_1)$, $(\mathcal{A}_2)$ for assumptions as it is easy to refer in the results than saying we assume Assumption 1. Its a matter of style. If you like what you have, no worries. Just leave as is.} \textcolor{cyan}{Let me know if this is ok, I can modify it if needed.}
\begin{assumption}
    \label{assump:independenterrors} Errors $\{e_t\}_t$ conditioned on an arm, i.e., $\{e_t|i\}_t$ are i.i.d.~random variables with mean 0 and finite variance,  $\sigma^2 := \E(e_t^2|i) < \infty$ for $i = 1,\hdots, L$.
\end{assumption}

\begin{assumption}
\label{assump: errorsindependentX}  $e_t \perp x_t| a_t = i$ for all $t = 1,\ldots,T$, and $i =  1,\ldots, L$.
\end{assumption}

Note that the above distributional assumptions on the error are weaker than those made in  \cite{chen2021statistical}, where the errors are assumed to be sub-Gaussian in addition to satisfying the assumptions \ref{assump:independenterrors} and \ref{assump: errorsindependentX}. 

In Section \ref{subsec: kernel_eps_greedy_orig}, we propose an allocation strategy or an algorithm, $\mathcal{B}$, for choosing the next arm based on the sequence of past information on arms played, the covariates, and the rewards obtained, respectively.
We will build online estimators for the mean reward functions, $f_i, i = 1,\hdots L$, in Section \ref{sec: IPWR_estimator}. Then, we use these estimators to make optimal decisions about arms in a sequential manner. Next, we formulate the notion of \textit{regret}, which is a standard way to assess the performance of MABC algorithms.

For covariate $X_t = x_t$, let, 
$$a^*_t:= \argmax_{a \in \{1,\hdots,L\}} f_a(x_t)$$ 
be the true best arm at time $t$ and $f_{a^*_t}(x_t)$ the corresponding best function value. 
Then, given the previously observed contexts, arms and rewards $\{(X_s, a_s, y_s),$ $s=1,\hdots, t-1\}$ and the current context $X_t = x_t$, a standard goal in an MABC problem is to choose an action $a_t$ in order to minimize the regret (see Definition \ref{def: Regret}) after $T$ rounds. Let $\mathcal{F}_t = \sigma \langle a_1, x_1, y_1, a_2, x_2, y_2, \hdots, a_t, x_t, y_t\rangle$ denote the sigma-field generated by all information on the covariates observed, arms pulled, and rewards obtained, respectively, until time $t$.

\begin{definition}
\label{def: Regret}
\textit{The instantaneous regret at time $t$ is $r_t(\mathcal{B}) := f_{a^*_t}(x_t) - f_{a_t}(x_t)$, where $a^*_t$ is the optimal arm at time $t$ and $a_t$ is the arm chosen by the bandit algorithm, $\mathcal{B}$, at time $t$. The cumulative regret $R_T(\mathcal{B})$ with horizon $T$ is defined as:
\begin{align*}
	R_T(\mathcal{B}):=  \sum_{t=1}^T (f_{a^*_t}(X_t)  - f_{a_t}(X_t)).
	\end{align*}}
\end{definition}	
Note that the regret, as defined above, is a random quantity. Thus, we are interested in finding an upper bound for regret in probability or in expectation.
Since our method aims at providing model flexibility by allowing us to discover a non-linear relationship between expected rewards and covariates in an RKHS while trying to achieve MABC designs which are less prone to problems of bias, we also study the estimation error. Let $\hat{f}_{a}$ denote the proposed estimator for $f_a$. The estimation error at time $t$ is defined as $\normx{\hat{f}_{a_t} - f_{a_t}}_\mathcal{H}$, where $a_t$ is the arm chosen by the algorithm.

%%%%%%%%%%%%%%%%%%%%%%%%%%%%%%%%%%%%%%%%%%%%%%%% Algorithm %%%%%%%%%%%%%%%%%%%%%%%%%%%%%%%%%%%%%%%%%%%%%%%%%%%%%%%%%
\section{Kernel \texorpdfstring{$\epsilon$}--greedy algorithm \& IPWKR estimator}
\label{sec: algorithm}
A simple policy to make sequential decisions in the MABC framework is to be greedy and choose the arm yielding the highest estimated reward for that covariate. However, this could lead to under-exploration of some arms, thus adversely affecting the performance of the algorithm. A way around this is to use $\epsilon$-greedy, a randomized version of the greedy algorithm, which chooses the best arm with a large probability, i.e., $(1-\epsilon)$, and explores the remaining arms with a small probability, i.e., $\epsilon$. In the following, we propose a kernel $\epsilon$-greedy algorithm that sequentially makes decisions about which arms to play for the MABC problem as described in Section \ref{sec: setup}.

\subsection{Kernel \texorpdfstring{$\epsilon$}--greedy algorithm}
\label{subsec: kernel_eps_greedy_orig}
The proposed algorithm is a kernelized version of the popular $\epsilon$-greedy algorithm for the MABC problem. Let $\{\epsilon_t\}_t$ be a sequence of non-increasing probabilities, such that $\epsilon_t \rightarrow 0$ as $t \rightarrow \infty$. We denote $\hat{a}_t$ to be the arm chosen by the proposed algorithm at time $t$, as it depends on all previous data. Below, we describe the kernel $\epsilon$-greedy strategy. 
\begin{enumerate}
	\item \textbf{Initialize.} \label{alg:initialization} Randomly select among the $L$ arms up to time $t_0$ for $t = 1,2,\hdots, t_0$, such that at least one reward per arm is obtained by time $t_0$.
	\item \textbf{Estimate $f_i$.} \label{alg: estimate_f_step} At time $t_0$, construct regression estimators for the $L$ arms and denote them by $\hat{f}_{i,t_0}, i = 1,\hdots, L$.
	\item \label{alg: most_promising_arm_step} \textbf{Most promising arm at time $t$.} For $t = t_0+1$, observe covariate $X_t = x_t$ and define:
	\begin{align*}
		A_t = \argmax_{i \in \mathcal{A}} \hat{f}_{i,t-1}(x_t),  %\label{best_performing_arm}
	\end{align*}
	be the arm corresponding to the highest estimated value at the current covariate.
	\item \label{alg:eps_greedy_step} \textbf{$\epsilon$-greedy step.} For a non-increasing probability sequence $\{\epsilon_t\}_t$, the arm pulled is given by the following randomized scheme:
	\begin{align}
		\hat{a}_t = 
		\begin{cases}
A_t & \ \text{with probability } 1- \epsilon_t\\
\{1,\hdots,L\}\backslash A_t & \ \text{with probability } \frac{\epsilon_t}{L -1}
		\end{cases}. \label{eq: La_hat_def}
		\end{align}
  \item \label{alg: update_estimates}  \textbf{Update the estimators.} Corresponding to the arm pulled at time $t = t_0 + 1$, observe reward $Y_t$  and update $\hat{f}_{\hat{a}_t,t}$. For the remaining arms, $i \neq \hat{a}_t$, $\hat{f}_{i,t} = \hat{f}_{i,t-1}$.
  \item Repeat steps \ref{alg: most_promising_arm_step}-\ref{alg: update_estimates} for $t = t_0 + 2$ and so on up to time $T$.
\end{enumerate}
Note that step \ref{alg:initialization} is an initialization step, where we randomly assign the $L$ arms until time $t_0$, such that at least one reward is observed per arm by then. In step \ref{alg: estimate_f_step}, we construct regression estimators for each arm using the information gathered during the initialization phase. Step \ref{alg: estimate_f_step} is presented as a generic step as we do not describe how the regression estimator is constructed. We propose an inverse probability weighted kernel ridge regression estimator in Section \ref{sec: IPWR_estimator} and study the above algorithm for that specific estimator. In step \ref{alg: most_promising_arm_step}, at time $t = t_0 + 1$, we evaluate estimated mean reward functions at the covariate $X_{t}$ for each arm using the estimators constructed in step \ref{alg: estimate_f_step} and find the arm $A_t$ that maximizes the estimated mean reward. Note that, at this instant, we face the \textit{exploration-exploitation dilemma}. That is, we can either choose the most promising arm, $A_t$, based on the data available so far or explore the remaining arms, $a \neq A_t$. We use the $\epsilon$-greedy strategy in step \ref{alg:eps_greedy_step} to balance this trade-off. This is a randomization scheme where we choose the best promising arm $A_t$ with a larger probability $1-\epsilon_t$ and explore the other arms $a \neq A_t$ with the remaining probabilities $\epsilon_t/(L-1)$. We also assume that $\epsilon_t \leq (L-1)/L$ for $t > t_0$, so that $1 - \epsilon_t \geq \epsilon_t/(L-1)$. Note that, the exploration probabilities $\{\epsilon_t\}_t$ are chosen to be a decreasing sequence of probabilities converging to 0 as $t \rightarrow \infty$, hence exploiting more with time. This is because as we accumulate more data, we gain more confidence in our estimates for the mean rewards for each of the arms. Then, the same process is repeated sequentially until we hit the time horizon $T$. In Section \ref{sec: IPWR_estimator}, we propose an online kernel regression estimator, which we use in this algorithm. Note that we have a `hat' on the proposed arm notation, $\hat{a}_t$, to highlight that the choice of the arm is data dependent.

%%%%%%%%%%%%%%%%%%%%%%%%%%%% Estimator %%%%%%%%%%%%%%%%%%%%%%%%%%%%%%%%
\subsection{Inverse probability weighted kernel ridge regression estimator}
\label{sec: IPWR_estimator}
In this section, we propose an online version of the kernel ridge regression estimator for the mean reward functions $f_i, i = 1,\hdots, L$. Recall, $\mathcal{F}_t = \sigma \langle \hat{a}_1, x_1, y_1, \hat{a}_2,$\\$x_2, y_2, \hdots, \hat{a}_t, x_t, y_t\rangle$ denotes the sigma-field generated by all information on the covariates observed, arms pulled, and rewards obtained, respectively, until time $t$.  To build an online kernel ridge regression estimator, we solve the following optimization problem with Tikhonov regularization,
\begin{align}
\hat{f}_{i,t} = \argmin_{f_i \in \mathcal{H}} \sum_{s=1}^{t} \dfrac{I\{\hat{a}_s = i\}}{P(\hat{a}_s = i| \mathcal{F}_{s-1},X_s)} (Y_s - \langle f_i, k(\cdot, X_s) \rangle_{\mathcal{H}})^2 + \lambda \normx{f_i}_\mathcal{H}^2, \label{eq: weighted_optimization}
\end{align}
where $\lambda>0$ is the regularization parameter, and  
$I\{\hat{a}_s = i\}$ denotes the indicator function which is 1 if the arm chosen by the algorithm at time $s$ is $i$ and is 0 otherwise. Note, the optimization in \eqref{eq: weighted_optimization} weighs each observation $(x_s, a, y_{a,s})$ in the history of arm $a$ by the inverse propensity score $w_{a,s} = 1/P(a|\mathcal{F}_{s-1}, X_s)$ and uses weighted regression to obtain
the estimate $\hat{f}_{i,t}$. We do this for handling biased estimation in sequential decision-making, particularly in the context of MABC. In practice, if the sequential decision-making algorithm tends to favor some arms over others due to the underlying exploration or exploitation strategy, the estimation of the reward function can become biased because the algorithm will observe some arms much more frequently than others \citep{dimakopoulou2019balanced}. This introduces bias in estimating the true reward function for the less frequently chosen arms. To address this issue, inverse propensity weighting \citep{imbens2015causal, dimakopoulou2019balanced} is used. The core idea is as follows: at each time step $s$, the observation $(X_s, a_s, Y_s)$  is weighted by the inverse probability that arm $a_s$ was selected, given the context $X_s$
and history $\mathcal{F}_{s-1}$. This weight corrects for the bias introduced by non-uniform arm selections. Essentially, it up-weights observations from arms that are less frequently selected.
 Using the same ideas as in kernel ridge regression, it is easy to verify that the solution to \eqref{eq: weighted_optimization}  is given by:
\begin{align}
	\hat{f}_{i,t} &= \left(\dfrac{1}{t} \sum_{s=1}^t \dfrac{I\{\hat{a}_s = i\}}{P(\hat{a}_s = i| \mathcal{F}_{s-1}, X_s)} k(\cdot, X_s)\otimes k(\cdot, X_s) + \lambda I  \right)^{-1} \nonumber \\
    & \quad \quad \quad \times \dfrac{1}{t} \sum_{s=1}^t \dfrac{I\{\hat{a}_s = i\}}{P(\hat{a}_s = i| \mathcal{F}_{s-1},X_s)} k(\cdot, X_s) Y_s \label{eq: Lbetahat_i},
\end{align}
for $i = 1,\hdots, L$. We will refer to this as the online Inverse Probability Weighted Kernel Ridge (IPWKR) regression estimator.  Note that, in the above algorithm, the probability weights $P(\hat{a}_s = i| \mathcal{F}_{s-1}, X_s)$ are known at time $s$. 
More specifically, since we know $A_s$ and $\epsilon_s$ at time step $s$, $P(\hat{a}_s = A_s| \mathcal{F}_{s-1}, X_s) = 1- \epsilon_s$ and $P(\hat{a}_s = a| \mathcal{F}_{s-1}, X_s) = \epsilon_s/(L-1)$ for $a \neq A_s$, therefore $\hat{f}_{i,s}$ for $i = 1, \hdots, L$ and for $s = 1, \hdots, t$ are data-determined estimates and can be used for estimation at the $(t + 1)$\ts{th} time. 
Note that the definitions of the arm pulled (see \eqref{eq: La_hat_def}) and the online IPWKR estimator in \eqref{eq: Lbetahat_i} depend on each other. Since the data are not independent, the consistency of the estimator does not follow immediately from the classical tools in kernel methods. Therefore, one of our contributions is in analyzing the estimation error associated with the online IPWKR estimator, and establishing its consistency and rate of convergence.
% From the proposed estimator, it is easy to see that a natural candidate for the estimator of the covariance operator is:
% \begin{align}
%  \hat{\Sigma}_{i,t} = \dfrac{1}{t} \sum_{s=1}^t \frac{I\{\hat{a}_s = i\}}{P(\hat{a}_s = i| \mathcal{F}_{s-1}, X_s)} k(\cdot, X_s)\otimes k(\cdot, X_s), \label{eq: LSigma_hat_it}
% \end{align}
% which can be shown (see Lemma~\ref{lem: Lamma_unbiased_covariance}) to be an unbiased estimator of the covariance operator $\Sigma := \E (k(\cdot, X_s)\otimes k(\cdot, X_s))$. We highlight that the unbiasedness of the covariance estimator is critical in our analysis. %\textcolor{red}{I suggest that the following Lemma be moved to an appendix.} \textcolor{cyan}{done.}
In the following, we rewrite the proposed IPWKR %of Section \ref{subsec: kernel_eps_greedy_orig} 
in an implementable form. %at.

\subsection{Implementation of kernel  \texorpdfstring{$\epsilon$}--greedy algorithm}
\label{sec: implementation}
In this section, we devise an implementable version of IPWKR. To this, we define the following.
\begin{itemize}
	\item Let $S_{t, X}: \mathcal{H} \rightarrow \mathbb{R}^t$ be the sampling operator, such that $$S_{t, X} f = \frac{1}{\sqrt{t}}[f(X_1), \hdots, f(X_t)]^\top.$$
	\item The reconstruction operator is given by $S_{t, X}^{*}: \mathbb{R}^t \rightarrow \mathcal{H}$, where $$S_{t, X}^* \underline{\alpha} = \frac{1}{\sqrt{t}} \sum_{s=1}^t \alpha_s k(\cdot, X_s),\,\,\,\underline{\alpha} \in \mathbb{R}^t.$$
	\item $ S_{t,X} S_{t,X}^* = \frac{K_t}{t}: \mathbb{R}^t \rightarrow \mathbb{R}^t$, where $K_t$ is the kernel/Gram matrix, i.e., $[K_t]_{ij}=k(X_i,X_j),\,\,i,j=1,\ldots,t$.
\end{itemize}  
%We express the empirical covariance operator in \eqref{eq: LSigma_hat_it} in terms of these operators. 
Let $\Lambda_{i,t}$ be a diagonal matrix in $\mathbb{R}^{t\times t}$ with diagonal elements given by, $$\left\{w_{s,i} := \frac{I\{\hat{a}_s = i\}}{P(\hat{a}_s = i| \mathcal{F}_{s-1},X_s)}, s = 1,\hdots, t\right\}.$$ Then note that, $$\hat{\Sigma}_{i,t}:=\dfrac{1}{t} \sum_{s=1}^t \frac{I\{\hat{a}_s = i\}}{P(\hat{a}_s = i| \mathcal{F}_{s-1}, X_s)} k(\cdot, X_s)\otimes k(\cdot, X_s) = S_{t,X}^* \Lambda_{i,t} S_{t,X},$$ and
\begin{align*}
\hat{\Sigma}_{i,t} f = \dfrac{1}{t} \sum_{s=1}^t w_{s,i} f(X_s) k(\cdot, X_s) \ \text{for all} \ i = 1,\hdots, L.
\end{align*}
Let ${Y}_t = (y_1, \hdots, y_t)^\prime$. Then, the proposed estimator in \eqref{eq: Lbetahat_i} can be written as:
\begin{align*}
\hat{f}_{i,t} &= \dfrac{1}{\sqrt{t}} (S_{t,X}^* \Lambda_{i,t} S_{t,X} + \lambda I)^{-1} \dfrac{1}{\sqrt{t}}\sum_{s=1}^t w_{s,i} k(\cdot, X_s) y_s\\
&= \dfrac{1}{\sqrt{t}} (S_{t,X}^* \Lambda_{i,t} S_{t,X} + \lambda I)^{-1} S_{t,X}^* \Lambda_{i,t} {Y}_t \\
&= \dfrac{1}{\sqrt{t}} S_{t,X}^* (\Lambda_{i,t} S_{t,X} S_{t,X}^* + \lambda I)^{-1} \Lambda_{i,t} Y_t,
\end{align*}
where the last equality follows from the fact that,
\begin{align*}
(S_{t,X}^* \Lambda_{i,t} S_{t,X} + \lambda I)^{-1} S_{t,X}^* = S_{t,X}^* (\Lambda_{i,t} S_{t,X} S_{t,X}^* + \lambda I)^{-1}.
\end{align*}
Therefore we obtain
\begin{align*}
\hat{f}_{i,t} &= \dfrac{1}{\sqrt{t}} S_{t,X}^* \left(\Lambda_{i,t} \frac{K_t}{t} + \lambda I\right)^{-1} \Lambda_{i,t} Y_t, \ \text{for} \ i = 1,\hdots, L. 
\end{align*}
Then, using the definition of $S_{t,X}^*$, the estimated reward function value at $X_{t+1}$ for arm $i$ is given by:
\begin{align}
\hat{f}_{i,t} (X_{t+1}) &= \dfrac{1}{t} \bar{k}_{t+1}^\top \left(\Lambda_{i,t} \frac{K_t}{t} + \lambda_t I\right)^{-1} \Lambda_{i,t} Y_t \nonumber\\
&=  \bar{k}_{t+1}^\top (\Lambda_{i,t} {K_t} + t \lambda_t I)^{-1} \Lambda_{i,t} Y_t \ \text{for} \ i = 1,\hdots, L, \label{eq: fhat_implementable_estimate}
\end{align}
where, $\bar{k}_{t+1} = \left(k(X_1, X_{t+1}), k(X_2, X_{t+1}), \hdots, k(X_t, X_{t+1})\right)^\top$. Note that \eqref{eq: fhat_implementable_estimate} involves only inverting a $t \times t$ matrix and therefore this version of the estimator is implementable. In order to facilitate faster computation, we use SVD for finding the inverse in \eqref{eq: fhat_implementable_estimate}. A sketch of the kernel $\epsilon$-greedy algorithm that utilizes the implementable version of the estimator can be found in Algorithm \ref{algorithm: kernel_eps_greedy_implementable}.
\begin{algorithm}[t]
\begin{algorithmic}[1]
\State Randomly select arms $\hat{a}_1, \hat{a}_2, \hdots, \hat{a}_{t_0} \in \mathcal{A} = \{1,\hdots, L\}$ until each arm is selected at least once.
\For{$t = t_0 + 1, \hdots, T$}
\State Estimate $\hat{f}_{i,t-1}(X_t)$, for each $i = \mathcal{A}$ using \eqref{eq: fhat_implementable_estimate}.
\State Calculate the best-performing arm so far: $A_t = \argmax_{i \in \mathcal{A}} \hat{f}_{i,t-1}(X_t)$.
\State For a non-increasing exploration probability sequence $\{\epsilon_t, t  \geq 1\}$, the arm pulled is given by:
\begin{align*}
\hat{a}_t = 
\begin{cases}
A_t & \ \text{with probability} \ 1 - \epsilon_t\\
\{1, \hdots, L\} \backslash A_t & \ \text{with probability} \ \frac{\epsilon_t}{L-1}
\end{cases}.
\end{align*}
\State Observe reward $Y_t$ corresponding to $\hat{a}_t$. 
\State For $i = \hat{a}_t$, update $\hat{f}_{i,t}$ using \eqref{eq: fhat_implementable_estimate} and use $\hat{f}_{i,t} = \hat{f}_{i,t-1}$ for $i \in \mathcal{A} \backslash \hat{a}_t$.
\EndFor
\end{algorithmic}
\caption{Kernel $\epsilon$-greedy algorithm}
\label{algorithm: kernel_eps_greedy_implementable}
\end{algorithm}
%%%%%%%%%%%%%%%%%%%%%%%%%%%%%%%%%%%%%%%%%%% Estimation error results %%%%%%%%%%%%%%%%%%%%%%%%%%%%%%%%%%%%%%%%%%%%%%
\section{Estimation error: Convergence rates}
\label{sec: estimation_error}
 In this section, we present the theoretical results for the proposed algorithm, for which the proofs can be found in Section \ref{sec: proofs}.
We make the following assumptions throughout this paper whenever we are working under the assumption that the mean reward functions lie in an RKHS $\mathcal{H}$.
\begin{assumption}
\label{assump: Levaluedecay_rep}
$\eta_i(\Sigma) \le \bar{C} i^{-\alpha}, \alpha > 1$ where 
$\eta_i(\Sigma)$ denotes the $i^{th}$ eigenvalue of $\Sigma = \E(k(\cdot, X_s)\otimes k(\cdot, X_s))$ and $\bar{C}\in(0,\infty)$.
\end{assumption}
\begin{assumption}
\label{assump: Lsmoothness_cond_rep}
For all $i = 1,\hdots,L$, $f_i\in \text{Ran}(\Sigma^{\gamma_i})$, $0 < \gamma_i \leq \frac{1}{2}$, i.e., there exists $h \in \mathcal{H}$ such that $f_i = \Sigma^{\gamma_i} h$ for $i = 1,\hdots,L$.
\end{assumption}
Note that, \ref{assump: Levaluedecay_rep} implies that the effective dimension, $$N_{\Sigma, 1}(\lambda):= \tr \left( \Slinv \Sigma \right) \lesssim \lambda^{-1/\alpha},$$ which controls the complexity of $\mathcal{H}$ and \ref{assump: Lsmoothness_cond_rep} determines the smoothness of the true mean reward functions \citep{caponnetto2007optimal}.  
Next, in Theorem~\ref{thm: theorem1}, which is proved in Section \ref{proof: proof1}, we present an upper bound on the estimation error in probability.
\begin{theorem}
\label{thm: theorem1}
Suppose \ref{assump:independenterrors}--\ref{assump: Lsmoothness_cond_rep} hold and $\{\epsilon_s\}^t_{s=1}$ is such that for any $\delta>0$ and $t\ge 1$, 
\begin{align}
\lambda_{i,t} = \left[\dfrac{1}{\delta t^2} \left(\sum_{s=1}^t \dfrac{1}{\epsilon_s}\right) \right]^{\alpha/(2\gamma_i \alpha + \alpha + 1)},\,\,i=1,\ldots,L, \label{eq: Lthm_cond_lambda}
\end{align}
satisfies \begin{align}\lambda_{i,t} \ge \left[\dfrac{4 (L-1)\kappa A_1(\bar{C},\alpha)}{\delta t^2} \left(\sum_{s=1}^t \dfrac{1}{\epsilon_s}\right)\right]^{\alpha/(1+\alpha)}.\label{Eq:constraint-lambda}\end{align}
Then, the following holds with probability at least $1-2\delta$:
\begin{align}
\normx{ \hat{f}_{i,t} - f_i }_\mathcal{H} \le %A(\sigma^2,L,C_i,\bar{C},\gamma_i,\alpha)
2\sqrt{2}\max\{C_0,C_i\}
\left[\dfrac{1}{\delta t^2} \left( \sum_{s=1}^t \frac{1}{\epsilon_s}\right)  \right]^{\gamma_i \alpha/ (2\gamma_i \alpha + \alpha + 1)}, \,\, i =1,\hdots,L, \label{eq: Lthm_1_result}
\end{align}
where $C_0=\sqrt{\sigma^2(L-1)A_1(\bar{C},\alpha)}$, $A_1(\bar{C},\alpha)=\bar{C}^{-1/\alpha}\int^\infty_0 (1+x^\alpha)^{-1}\,dx$, and $C_i = \normx{\Sigma^{-\gamma_i} f_i}_\mathcal{H}$.
\end{theorem}
Note that the upper bound on the estimation error depends on the exploration probability sequence $\{\epsilon_s\}_{s = 1}^t$, and for consistent estimation, we require $\sum_{s=1}^t \epsilon^{-1}_s = o(t^2)$ as $t\rightarrow\infty$. Under this requirement of $\sum_{s=1}^t \epsilon^{-1}_s = o(t^2)$ as $t\rightarrow\infty$, it is easy to verify that \eqref{Eq:constraint-lambda} holds. 

In order to compare Theorem~\ref{thm: theorem1} with Proposition 4.1 in \cite{chen2021statistical}, we assume that the underlying RKHS, $\mathcal{H}$, is finite-dimensional, and make the following assumption instead of \ref{assump: Lsmoothness_cond_rep}:
\begin{assumption}
\label{assump: min_evalue}
The minimum eigenvalue of $\Sigma$, denoted as $\eta_{\emph{min}}(\Sigma)$ satisfies 
$\eta_{\emph{min}}(\Sigma) > \eta$ for some $\eta > 0$. 
\end{assumption}
Since \cite{chen2021statistical} study linear MABC which is equivalent to our approach when the kernel is linear, i.e., the corresponding RKHS is finite-dimensional, we would like to specialize Theorem~\ref{thm: theorem1} to finite-dimensional RKHS. This assumption of finite-dimensional RKHS is imposed through \ref{assump: min_evalue}, which is also assumed in \cite{chen2021statistical}.
\begin{theorem}
\label{thm: theorem_finiteRKHS}
Suppose \ref{assump:independenterrors}, \ref{assump: errorsindependentX} and \ref{assump: min_evalue} hold. For any $\delta>0$ and $t\ge 1$, suppose $\{\epsilon_s\}^t_{s=1}$ satisfies \begin{align}\frac{1}{t^2}\sum^t_{s=1}\frac{1}{\epsilon_s}\le \frac{\delta\eta}{4(L-1)d\kappa},\label{constraint-thm2}\end{align} where $d:=\emph{dim}(\mathcal{H})$. Then, for  $t\ge 1$,  %exploration probability sequence $\{\epsilon_s\}_s$, 
with the choice of
\begin{align*}
\lambda_{i,t} = \left[\dfrac{1}{ t^2} \left(\sum_{s=1}^t \dfrac{1}{\epsilon_s}\right) \right]^{1/2},\,\,i=1,\ldots,L, 
%\label{eqF: thm_cond_lambda}
\end{align*}
% and 
% the exploration probability sequence $\sum_{s=1}^t \frac{1}{\epsilon_s} = o(t^2)$, 
the following holds with probability at least $1-2\delta$:
 \begin{align*}
\normx{\hat{f}_{i,t} - f_i }_\mathcal{H} \leq  \frac{4}{\sqrt{\delta}}\max\{\tilde{C}_0,\tilde{C}_i\} \left[\dfrac{1}{ t^2} \left( \sum_{s=1}^t \frac{1}{\epsilon_s}\right)  \right]^{1/2},\ i = 1, \hdots, L. %\label{eqF: thm_1_result}
\end{align*}
Moreover, 
% \begin{align}
%     \E[\normx{\hat{f}_{i,t} - f_i}_{\mathcal{H}}] \leq 2 \max\{\tilde{C}_0, \tilde{C}_i\} \left[\dfrac{1}{t^2} \sum_{s=1}^t \dfrac{1}{\epsilon_s}\right]^{1/2},\label{eq: estimation_bound_finiteRKHS_inexpectation}
%  \end{align}
 %and 
 for $0\le \zeta < 1$,
 \begin{align}
    \E[\normx{\hat{f}_{i,t} - f_i}_{\mathcal{H}}^{1+\zeta}] & \leq B(\tilde{C}_0,\tilde{C}_i,\zeta,\eta)
    % \left(\dfrac{2}{1-\zeta}\right) \max\{\tilde{C}_0, \tilde{C}_i\}^{1+\zeta} 
    \left[\dfrac{1}{t^2} \sum_{s=1}^t \dfrac{1}{\epsilon_s} \right]^{\frac{1+\zeta}{2}},\,\,i=1,\ldots,L, \label{eq: estimation_bound_zeta_inexpectation} 
    % &\leq \textcolor{red}{ 8 \sqrt{2} \max\{\tilde{C}_0, \tilde{C}_i \}
    % % \left(\dfrac{2}{1-\zeta}\right) \max\{\tilde{C}_0, \tilde{C}_i\}^{1+\zeta} 
    % \left[\dfrac{1}{t^2} \sum_{s=1}^t \dfrac{1}{\epsilon_s} \right]^{\frac{1+\zeta}{2}},\,\,i=1,\ldots,L, \label{eq: estimation_bound_zeta_inexpectation} }
 \end{align}
 where $\tilde{C}_0 := \sqrt{\frac{(L-1) d\sigma^2}{\eta }}$ and $\tilde{C}_i:= \frac{\normx{f_i}_\mathcal{H}}{\eta}$. Here $B(\tilde{C}_0,\tilde{C}_i,\zeta,\eta)$ is a constant that depends only on its arguments and not on $t$ and $\{\epsilon_s\}_s$.
\end{theorem}
Note that in the above result, the choice of $\lambda_{i,t}$ is independent of $i$ and $\delta$, unlike in the infinite-dimensional case of Theorem~\ref{thm: theorem1}. Moreover, unlike in Theorem~\ref{thm: theorem1}, the choice of $\lambda_{i,t}$ in Theorem~\ref{thm: theorem_finiteRKHS} has an exponent of $\frac{1}{2}$ instead of the term depending on $\gamma_i$ and $\alpha$ in \eqref{eq: Lthm_1_result}. Clearly, the estimators are consistent if $\sum_{s=1}^t \epsilon^{-1}_s = o(t^2)$ as $t\rightarrow\infty$. Under this requirement on $\{\epsilon_s\}_s$, it is easy to verify that \eqref{constraint-thm2} holds.
The proof of Theorem \ref{thm: theorem_finiteRKHS} is provided in Section \ref{proof: proof2}. 

\begin{remark}\label{rem:0}
\cite{chen2021statistical} %(Section 4) 
proposed a weighted online least squares estimator similar to the IPWKR estimator for the linear MABC problem with 2 arms ($L = 2$) and a finite-dimensional context space. However, unlike our estimator, they study an unregularized online weighted least squares (WLS) estimator. Under the assumptions of (i) bounded covariates, (ii) reliability of linear approximation, and (iii) minimum eigenvalue is bounded from below, they provide a bound (see Proposition 4.1) on the estimation error, which can be summarized to behave as 
\begin{align}
\tilde{O} \left(\left[\dfrac{d^2 \log(4d/\delta)}{t \epsilon_t^4} \right]^{1/2}\right) \label{eq: rate_Chenetal_estimation}
\end{align}
with probability at least $1-\delta$. On the other hand, our method provides a bound of \begin{align}
  \tilde{O} \left(\left[\dfrac{d}{\delta} \dfrac{1}{t^2} \Big(\sum_{s=1}^t \dfrac{1}{\epsilon_s}\Big) \right]^{1/2} \right),\label{eq: rate_finiteRKHS_effecDim}
 \end{align}
 which also holds with probability $1-\delta$. By comparing \eqref{eq: rate_Chenetal_estimation} and \eqref{eq: rate_finiteRKHS_effecDim}, it can be observed that the bound in \cite{chen2021statistical} holds with sub-Gaussian concentration while ours holds with polynomial concentration. Another key distinction is the requirement of $\frac{1}{\epsilon_t} = o(t^{1/4})$ in \eqref{eq: rate_Chenetal_estimation} vs. $\sum_{s=1}^t \epsilon^{-1}_s = o(t^2)$ in \eqref{eq: rate_finiteRKHS_effecDim} for the estimators to be consistent, as $t\rightarrow\infty$. Clearly, a sufficient condition for our estimator to be consistent is $\frac{1}{\epsilon_t} = o(t)$. 
Therefore, our result is stronger than that of \cite{chen2021statistical}, as the exploration probability can decay at a faster rate while still guaranteeing the consistency of our estimator. This is promising as for situations where the reward functions are relatively easier to learn, more exploitative strategies can significantly reduce the regret. Another key highlight of our work is that these improved results are obtained for any $f_i$ belonging to a finite-dimensional RKHS, that is not necessarily linear.
%\end{enumerate} 
\end{remark}

%%%%%%%%%%%%%%%%%%%%%%%%%%%%%%%%%%%%%%%%%%% Regret analysis %%%%%%%%%%%%%%%%%%%%%%%%%%%%%%%%%%%%%%%%%%%%%%%%%%
\section{Regret analysis}
\label{sec: regret_analysis}
In this section, we construct upper bounds for the regret defined in Definition \ref{def: Regret} for the algorithm proposed in Section \ref{sec: algorithm} that involves using the IPWKR estimator studied in Section \ref{sec: IPWR_estimator}. 
Recall, $a^*_t = \argmax_{a \in \mathcal{A}} f_a(x_t),$ where $f_a(x_t) = \langle f_a, k(\cdot, x_t) \rangle_\mathcal{H}$. Let $A_t = \argmax_{a \in \mathcal{A}} \fh_a(x_t)$ and note that by definition of $A_t$, $\fh_{a^*_t}(X_t) \leq \fh_{A_t}(X_t)$. 
% For an initialization phase, $t_0$, the proposed algorithm is run from time $t_0 + 1$ through $T$. 
Then, the cumulative regret for the proposed algorithm up to some time horizon $T$ is given by,
\begin{align}
R_T &= \sum_{t=1}^{T} f_{a^*_t}(X_t) - f_{\hat{a}_t}(X_t)  = \sum_{t=1}^{t_0} f_{a^*_t}(X_t) - f_{\hat{a}_t}(X_t) + \sum_{t = t_0 + 1}^{T}f_{a^*_t}(X_t) - f_{\hat{a}_t}(X_t) \nonumber \\
% \end{align}
% \begin{align}
& = \sum_{t=1}^{t_0} \langle  f_{a^*_t} - f_{\hat{a}_t}, k(\cdot, X_t) \rangle_\mathcal{H} + \sum_{t = t_0 + 1}^{T }f_{a^*_t}(X_t) - f_{\hat{a}_t}(X_t) \nonumber\\
& \leq  \sum_{t=1}^{t_0} \kappa \normx{f_{a^*_t} - f_{\hat{a}_t}}_{\mathcal{H}} +  \sum_{t = t_0 + 1}^{T}f_{a^*_t}(X_t) - f_{\hat{a}_t}(X_t)  \leq \Lambda t_0 +  \sum_{t = t_0 + 1}^{T }f_{a^*_t}(X_t) - f_{\hat{a}_t}(X_t) \nonumber \\
&= \Lambda t_0 +  \sum_{t = t_0 + 1}^{T } \left[ f_{a^*_t}(X_t) - \fh_{a^*_t}(X_t) + \fh_{a^*_t}(X_t) - f_{A_t}(X_t) + f_{A_t}(X_t) - f_{\ahat_t}(X_t) \right] \nonumber\\
& \leq \Lambda t_0 +  \sum_{t = t_0 + 1}^{T} \left[ f_{a^*_t}(X_t) - \fh_{a^*_t}(X_t) + \fh_{A_t}(X_t) - f_{A_t}(X_t) + f_{A_t}(X_t) - f_{\ahat_t}(X_t) \right]\nonumber\\
% \end{align}
% \begin{align}
&\leq  \Lambda t_0 + 2 \underbrace{ \sum_{t = t_0 + 1}^{T }  \sup_{a \in \mathcal{A}} |(f_a(X_t) - \fh_{a,t}(X_t)) |}_{\text{Cumulative estimation error}} + \underbrace{ \sum_{t = t_0 + 1}^{T} |f_{A_t}(X_t) - f_{\ahat_t}(X_t)| }_{\text{Randomization error}} ,\label{eq: regret_breakup}
\end{align}
where $\Lambda:=\sup\{\normx{f_{a} - f_{a^\prime}}_\mathcal{H}:a,a'\in\mathcal{A},a\ne a'\}$, and for simplicity, we did not put the algorithm within parenthesis for $R_T$ (as in Definition~\ref{def: Regret}) though all the presented results are for the proposed kernel $\epsilon$-greedy algorithm. Note that the first term in \eqref{eq: regret_breakup} is the regret incurred due to the random initialization up to time $t_0$. 
We call the second term on the right-hand side in \eqref{eq: regret_breakup} as \textit{cumulative estimation error}, as it measures the error in estimating the function accumulated over time, and the third term as \textit{randomization error} since it measures the error incurred due to the randomization scheme ($\epsilon$-greedy in step \ref{alg:eps_greedy_step} of the proposed algorithm in Section \ref{sec: algorithm}). 
% Then, we separately construct upper bounds for both to get the resulting upper bound for the regret. 
Note that, the initialization phase regret can be trivially bounded by $O(t_0)$, but the regret incurred during the post-initialization phase dominates over the regret incurred over the initialization phase. Hence, without the loss of generality, we set $t_0 = 0$ in the following results and in the corresponding proofs.
% These can then be easily modified to include the initialization regret.

\begin{theorem}
\label{thm: regret_infinite_case_pre}
 Suppose \ref{assump:independenterrors}--\ref{assump: Lsmoothness_cond_rep} hold, $\sup_{\substack{a, a^\prime \in \mathcal{A} \\ a \neq a^\prime}}\normx{f_{a} - f_{a^\prime}}_\mathcal{H}<\infty$, and $\{\epsilon_s\}^T_{s=1}$ is such that for any $\delta>0$, and $T\ge 1$, 
 % and for the following choice of $\lambda_{i,t}$ for $i \in \{1,\hdots,L\}$, $0<t\le T$,
\begin{align}
\lambda_{i,t} = \left[\dfrac{L}{\delta t^2} \left(\sum_{s=1}^t \dfrac{1}{\epsilon_s} \right)  \right]^{\alpha/(2\gamma_i \alpha + \alpha + 1)},\,\, i \in \{1,\hdots,L\},\,\,0<t\le T,\label{eq:choice_of_lam_regret}
\end{align}
satisfies $$\lambda_{i,t} \ge \left[\dfrac{4L (L-1)\kappa A_1(\bar{C},\alpha)}{\delta t^2} \left(\sum_{s=1}^t \dfrac{1}{\epsilon_s}\right)\right]^{\alpha/(1+\alpha)}\,\,i \in \{1,\hdots,L\},\,\,0<t\le T.$$ Define   $$  \Delta_t = \frac{L}{\delta t^2} \sum_{s=1}^t \frac{1}{\epsilon_s}.$$
%  and $(\epsilon_s)_s$ is such that for any $\delta>0$, there exists a constant $\tilde{C}$ (not depending on $t$) such that $\frac{1}{\delta t^2}\sum^t_{s=1}\epsilon^{-1}_s\le  \tilde{C}$ for all $t\ge 1$. Then for any $\delta>0$, $1\le p\le\infty$, $T\ge 1$ and for the following choice of $\lambda_{i,t}$ for $i \in \{1,\hdots,L\}$, $0<t\le T$,
% \begin{align}
% \lambda_{i,t} = \left[\dfrac{L}{\delta t^2} \left(\sum_{s=1}^t \dfrac{1}{\epsilon_s} \right)  \right]^{\alpha/(2\gamma_i \alpha + \alpha + 1)}, \label{eq:choice_of_lam_regret}
% \end{align}
 % and for $\{\epsilon_s\}_s$ chosen such that $\sum_{s=1}^t \frac{1}{\epsilon_s} = o(t^2)$, 
 Then, for $\delta>0$, $1\le p\le\infty$ and $T\ge 1$, the following holds with probability 
 % we get that for $\delta > 0$, with probability 
 at least $1-2\delta$:
\begin{align}
R_T &\leq \kappa \Theta\sum_{t=1}^T \left(I\{\Delta_t < 1\}  \Delta_t^{\frac{(\min_{i \in \mathcal{A}} \gamma_i) \alpha}{2(\min_{i \in \mathcal{A}} \gamma_i) \alpha + \alpha + 1}} + I\{\Delta_t \geq 1\}  \Delta_t^{\frac{(\max_{i \in \mathcal{A}} \gamma_i) \alpha}{2(\max_{i \in \mathcal{A}} \gamma_i) \alpha + \alpha + 1}}\right)  \nonumber \\
& \quad + \kappa T^{1-\frac{1}{p}}\left[\sum_{t=1}^T \dfrac{\epsilon_t}{L-1} + \left\{\dfrac{1}{\delta} \sum_{t=1}^T \dfrac{\epsilon_t}{L-1} \right\}^{1/2} \right]^{1/p} \sup_{\substack{a, a^\prime \in \mathcal{A} \\ a \neq a^\prime}}\normx{f_{a} - f_{a^\prime}}_\mathcal{H}, \label{eq: final_regret_bound_thm3}
\end{align} where
%  $$  \Delta_t = \frac{L}{\delta t^2} \sum_{s=1}^t \frac{1}{\epsilon_s},$$and 
$\Theta=4\sqrt{2}\max\{C_0,\max_{i\in\mathcal{A}}C_i\}$ with $C_0$, $C_i$ and $A_1(\bar{C},\alpha)$ being defined in 
% A(\sigma^2,L,C_i,\bar{C},\gamma_i,\alpha)$ with $(A(\sigma^2,L,C_i,\bar{C},\gamma_i,\alpha))_{i\in\mathcal{A}}$ being defined in 
Theorem~\ref{thm: theorem1}.
% , \ C_0 = \sqrt{2\sigma^2 (L-1)},\  C_* = \max_{1\leq i \leq L} \normx{\Sigma^{-\gamma_i} f_i}_\mathcal{H}, $$  and $\{\epsilon_t\}_t$ is the pre-specified exploration probability sequence.
\end{theorem}
The first term in the r.h.s.~of \eqref{eq: final_regret_bound_thm3} corresponds to the estimation error accumulated by time $T$ and the second term corresponds to the randomization error. The cumulative estimation error bound follows from the estimation error analysis in Theorem~\ref{thm: theorem1}, after applying a union bound over the arms. Note that for a given sequence of exploration probability, the cumulative estimation error and the randomization error behave inversely to each other. 
% We use the Chebyshev's inequality for the concentration bound in the randomization error as shown in the proof in Section \ref{sec: proofs}.
Further, note that if $\sum^t_{s=1}\epsilon^{-1}_s = o(t^2)$ as $t\rightarrow\infty$, then clearly $\Delta_t < 1$ for some large enough $t$. Let $t_1 = \min_{t}I\{\Delta_t < 1\}$. Then, we can choose the initialization phase endpoint to be  $\tilde{t}_0 = \max\{t_0, t_1\}$, which means only the first term in the accumulated estimation error would contribute to the regret bound. This reflects that the estimation error incurred due to the arm with the lowest smoothness parameter for the respective reward function dominates. 
% That is,
% \begin{align}
% R_T &\lesssim O(\tilde{t}_0) + \sum_{t=\tilde{t}_0 + 1}^T  \Delta_t^{\frac{(\min_{i \in \mathcal{A}} \gamma_i) \alpha}{2(\min_{i \in \mathcal{A}} \gamma_i) \alpha + \alpha + 1}}  \nonumber \\
% & \quad \quad + \kappa \left[\sum_{t=1}^T \dfrac{\epsilon_t}{L-1} + \left\{\dfrac{1}{\delta} \sum_{t=1}^T \dfrac{\epsilon_t}{L-1} \right\}^{1/2} \right]^{1/p} T^{1-\frac{1}{p}} \times \sup_{\substack{a, a^\prime \in \mathcal{A} \\ a \neq a^\prime}}\normx{f_{a} - f_{a^\prime}}_\mathcal{H}. \label{eq: final_regret_bound_thm3_modified}
% \end{align}
\begin{remark}\label{rem:1}
Suppose $\epsilon_t = t^{-\beta}$ for some $0 < \beta < 1$ and $t \in \mathbb{N}$, which satisfies the requirement that $\sum^t_{s=1}\epsilon^{-1}_s \lesssim t^{\beta+1}=o(t^2)$ as $t\rightarrow\infty$. Then under the assumptions of Theorem~\ref{thm: regret_infinite_case_pre}, for large enough $T$, \eqref{eq: final_regret_bound_thm3} reduces to 
\begin{equation}
% R_T \leq C_2(L, \kappa, \delta, w, \sigma^2, \beta) T^{(\beta - 1)w + 1} +  C_0(\beta, \delta) T^{1- \beta} \sup_{\substack{a, a^\prime \in \mathcal{A} \\ a \neq a^\prime}}\normx{f_{a} - f_{a^\prime}}_\mathcal{H}
R_T \lesssim 
%\dfrac{2\kappa \max\{C_0, C_*\}}{(1 + (\beta - 1)w)} \left(\dfrac{L}{\delta (\beta + 1)}\right)^{w} 
T^{(\beta - 1)w + 1} + 
%\sup_{\substack{a, a^\prime \in \mathcal{A} \\ a \neq a^\prime}}\normx{f_{a} - f_{a^\prime}}_\mathcal{H} \left(\dfrac{1 + \sqrt{\delta}}{\sqrt{\delta} (1-\beta) (L-1)}\right)  
T^{1-\frac{\beta}{p}},\nonumber
\end{equation}
for the choice of $\lambda_{i,t}$ as in \eqref{eq:choice_of_lam_regret}  and $$w = {\frac{(\min_{i \in \mathcal{A}} \gamma_i) \alpha}{2(\min_{i \in \mathcal{A}} \gamma_i) \alpha + \alpha + 1}}.$$ 
% $C_0 = \sqrt{2\sigma^2 (L-1)}$ and $C_* = \max_{1\leq i \leq L} \normx{\Sigma^{-\gamma_i} f_i}_\mathcal{H}$. See Section~\ref{Sec:proof-rem} for a proof of \eqref{Eq:regret-rem}.
Since $1\le p\le \infty$ is arbitrary, the best regret is achieved at $p=1$, yielding 
\begin{equation}
R_T \lesssim T^{(\beta - 1)w + 1} + T^{1-\beta}.\label{Eq:regret-rem}
\end{equation}
In \eqref{Eq:regret-rem}, the first term corresponds to an upper bound on the cumulative estimation error, which is increasing in $\beta$ and the second term corresponds to an upper bound on the randomization error, which is decreasing in $\beta$. 
% Note that, there is a trade-off between the two terms on the right-hand side. While the first term, i.e., the bound on the estimation error is decreasing over $T$, the second term, i.e., the bound on the randomization error is increasing in $T$. Therefore, by 
By balancing these two terms, the optimal choice of $\beta$ is given by
\begin{align*}
	 \beta = \frac{w}{w+1}=\dfrac{(\min_{i\in\mathcal{A}}\gamma_i) \alpha}{3 (\min_{i \in \mathcal{A}} \gamma_i) \alpha + \alpha + 1}=:\beta^*. %\label{eq: beta_cond_estimation_equals_randomization}
	 \end{align*}
  This means the optimal exploration probability sequence depends on the smoothness of the targets and the intrinsic dimensionality of $\mathcal{H}$, which is controlled by $\alpha$. Clearly, if $\beta>\beta^*$, i.e., exploitation is favored over exploration, the estimation error dominates the randomization error and the converse happens when $\beta<\beta^*$, i.e., exploration is favored over exploitation. Note that for any choice of $\gamma_i$ and $\alpha$, we have $\beta^*\le \frac{1}{5}$, which implies the growth rate of regret is at least of the order  $T^{4/5}$. 
  % \textcolor{red}{I do not follow how we got $\beta^*\le\frac{1}{5}$. I am now getting $8/121$. I will check again.} \textcolor{cyan}{Here is how I got the 1/5. 
  % \begin{align}
  %     \beta^* &= \dfrac{(\min_{i\in\mathcal{A}}\gamma_i) \alpha}{3 (\min_{i \in \mathcal{A}} \gamma_i) \alpha + \alpha + 1} \nonumber \\
  %     & = \dfrac{1}{3 + \frac{1}{\min_{i \in \mathcal{A}} \gamma_i} + \frac{1}{(\min_{i\in\mathcal{A}}\gamma_i) \alpha}} \label{eq: beta_star_eq2}\\
  %     & < \lim_{\gamma_i \rightarrow 1/2} \lim_{\alpha \rightarrow \infty}\dfrac{1}{3 + \frac{1}{\min_{i \in \mathcal{A}} \gamma_i} + \frac{1}{(\min_{i\in\mathcal{A}}\gamma_i) \alpha}}, \nonumber
  % \end{align}
  % where I can write \eqref{eq: beta_star_eq2} because $\gamma_i > 0$ and $\alpha >1$. To maximize \eqref{eq: beta_star_eq2}, we have to minimize the denominator and for that we want $\gamma_i \rightarrow 1/2$ and $\alpha \rightarrow \infty$. Also, this gives me a strict inequality, $\beta^* < \frac{1}{5}$.
  % }
\end{remark}
Next, in order to compare with the regret rate in Section 7.1 of \cite{chen2021statistical} for the $\epsilon$-greedy strategy, we specialize Theorem~\ref{thm: regret_infinite_case_pre} to a finite-dimensional setting. While the randomization error bound can be obtained similarly as the second term in \eqref{eq: final_regret_bound_thm3}, we modify the bounding for the cumulative estimation error (see the proof in Section \ref{proof: proof_thm4} for details).
% restrict ourselves to a finite-dimensional space and establish an upper bound in that setting.

%\subsection{Regret analysis: finite-dimensional RKHS}
% In  this section, we find the regret rate for the case when we consider a finite-dimensional RKHS as we did for the estimation error bound in Section \ref{sec: estimation_error}. While the randomization error bound can be obtained similarly as the second term in \eqref{eq: final_regret_bound_thm3}, for the cumulative estimation error we modify the proof in Section \ref{sec: proofs}.  We get the following result.

\begin{theorem}%[Regret bound for kernel $\epsilon$-greedy in finite-dimensions]
\label{thm: regret_finite_case_pre}
Suppose \ref{assump:independenterrors}, \ref{assump: errorsindependentX} and \ref{assump: min_evalue} hold, and $\sup_{\substack{a, a^\prime \in \mathcal{A} \\ a \neq a^\prime}}\normx{f_{a} - f_{a^\prime}}_\mathcal{H}<\infty$. For any $\delta>0$, suppose $\{\epsilon_s\}^T_{s=1}$ satisfies $$\frac{1}{t^2}\sum^t_{s=1}\frac{1}{\epsilon_s}\le \frac{\delta\eta}{4L(L-1)d\kappa}$$ for all $0<t\le T$, where $d:=\emph{dim}(\mathcal{H})$. Then, for %for the following choice of $\lambda_{i,t}$ for $i \in \{1,\hdots,L\}$, $0<t\le T$,
%Then for any $\delta > 0$ and $t>1$,  %exploration probability sequence $\{\epsilon_s\}_s$, 
%with the following choice of $\lambda_{i,t}$ for $i=1,\ldots,L$, 
% \textcolor{red}{Is there no dependence of $\lambda_t$ on $\eta$ as before?} \textcolor{cyan}{there is, I will make sure I change it everywhere.}
 % Under Assumptions \ref{assump:independenterrors} - \ref{assump: Levaluedecay_rep}, \ref{assump: min_evalue} and for the following choice of $\lambda$:
\begin{align*}
\lambda_{i,t} =\left[ \dfrac{L}{ t^2}  \left(\sum_{s=1}^t \dfrac{1}{\epsilon_s}\right) \right]^{1/2},\,\,i \in \{1,\hdots,L\},\,\, 0<t\le T, %\label{eq:choice_of_lam_regret_finite}
\end{align*}
and for $\delta>0$, $1\le p\le\infty$, $T\ge 1$, the following holds with probability at least $1-2\delta$:
% and for $\{\epsilon_s\}_s$ such that $\sum_{s=1}^t \frac{1}{\epsilon_s} = o(t^2)$, we get with probability at least $1-\delta$ for some $\delta > 0$:
\begin{align}
R_T &\leq  8 \frac{\kappa}{\sqrt{\delta}} \max\{\tilde{C}_0, \tilde{C}_* \}\sum_{t=1}^T  \left[\dfrac{L}{t^2} \Big(\sum_{s=1}^t \dfrac{1}{\epsilon_s} \Big) \right]^{1/2}  \nonumber \\
& \quad + \kappa T^{1-\frac{1}{p}}\left[\sum_{t=1}^T \dfrac{\epsilon_t}{L-1} + \left\{\dfrac{1}{\delta} \sum_{t=1}^T \dfrac{\epsilon_t}{L-1}  \right\}^{1/2} \right]^{1/p} \sup_{\substack{a, a^\prime \in \mathcal{A} \\ a \neq a^\prime}}\normx{f_{a} - f_{a^\prime}}_\mathcal{H}, \label{eq: finiteRKHS_regret_bound}
\end{align}
where $\tilde{C}_0 = \sqrt{\frac{(L-1)d \sigma^2}{\eta}}$ and $\tilde{C}_* = \max_{1\leq i \leq L} \frac{\normx{f_i}_\mathcal{H}}{\eta}$.
% for $p \in (0,1]$, where $\{\epsilon_t\}_t$ is the exploration probability sequence.
\end{theorem}

\begin{remark}
\label{rem: Regret_upper_bound_finite_polyeps}
As in Remark~\ref{rem:1}, the choice of $\epsilon_t=t^{-\beta}$ for some $0<\beta<1$ reduces \eqref{eq: finiteRKHS_regret_bound} to
     % Under Assumptions \ref{assump:independenterrors} - \ref{assump: Levaluedecay_rep} and \ref{assump: min_evalue}, for the exploration probability given by $\epsilon_t = t^{-\beta}, 0 < \beta < 1, t \in \mathbb{N}$, and corresponding $\lambda_t = \text{O}(t^{(\beta - 1)/2})$, with probability at least $1 - \delta$, the regret for the proposed kernel $\epsilon$-greedy algorithm in Section \ref{sec: algorithm} is bounded by:
  \begin{align*}
  R_T \lesssim  T^{\frac{\beta + 1}{2}} + T^{1-\beta},  %\label{eq: regret_break_approx_finiteRKHS}
% R_T \lesssim 2\kappa \max\{\tilde{C}_0, \tilde{C}_*\} \left(\dfrac{L}{\delta}\right)^{1/2} T^{\frac{\beta + 1}{2}} + \sup_{\substack{a, a^\prime \in \mathcal{A} \\ a \neq a^\prime}}\normx{f_{a} - f_{a^\prime}}_\mathcal{H} \left(\dfrac{1 + \sqrt{\delta}}{\sqrt{\delta} (1-\beta) (L-1)}\right) T^{1-\beta},  \label{eq: regret_break_approx_finiteRKHS}
\end{align*}
% for $\lambda_{i,t}$ being chosen according to \eqref{eq:choice_of_lam_regret_finite}, $\tilde{C}_0 = \sqrt{{(4(L-1)d \sigma^2)}/{\eta}}$ and $\tilde{C}_* = \max_{1\leq i \leq L} {(2\normx{f_i}_\mathcal{H}}/{\eta})$. 
Balancing these terms yields that $R_T$ has a growth order of at least $T^{2/3}$ with the choice of $\beta=\frac{1}{3}$. Note that the estimation error (resp. randomization error) dominates the randomization error (resp. estimation error) if $\beta>\frac{1}{3}$ (resp. $\beta<\frac{1}{3}$). We would like to highlight that this rate of $T^{2/3}$ is optimal for contextual bandits using an $\epsilon$-greedy strategy, i.e., this strategy cannot achieve regret rates slower than $T^{2/3}$ for contextual bandits \citep[Theorem 3]{dann2022guarantees}.
% \textcolor{red}{please ensure that the constants inside the parenthesis exactly reflect the general theorem. This comment is for previous results too. The reason I say is there is an $\eta$ in (22) which is not reflected in (23). If you prefer we can say $R_T\lesssim ...$ where we defined $\lesssim$ as a symbol which says $a\lesssim b$ means there exists a constant $c$ such that $a\leq cb$. Either of these fixes is fine.} \textcolor{blue}{I think I have fixed this in the results now, all the remarks now have $a \lesssim b$ since there I am doing an integral approximation. I have included all the constants in the results.}
 % \textcolor{red}{for what class of problems? linear?} \textcolor{cyan}{Please see the screenshot to the paragraph from their manuscript below. Also added the paper to Papers folder on overleaf. The result is a remark after Theorem 3, Section 6.1}
 % \begin{figure}[h!]
 %     \centering
 %     \includegraphics{Papers/dann_et_al_para.png}
 % \end{figure}
\end{remark}

\subsection{Regret analysis with the margin condition for finite dimensional RKHS}
\label{sec: margincondition}
In this section, we make an additional assumption known as the `margin condition' on the underlying reward functions as is commonly assumed in the MABC/contextual bandit's literature \citep{chen2021statistical,goldenshluger2013linear}. Under this assumption, we can achieve significant improvement in the expected regret rate as compared to the previous results. Note that, here we construct regret upper bounds in expectation unlike the previous results in Section \ref{sec: regret_analysis} where we constructed upper bounds on the regret in probability. As in the literature, we focus our analysis on $L = 2$ (two arms) though it can be extended to $L$ arms.  
\begin{assumption}
\label{assump: AssumptionMargin}
\textbf{Margin Condition:} There exists $C > 0$ such that
%\begin{align*}
$P_{X \sim \mathcal P_X}\left(0 < |\langle f_1 - f_0, k(\cdot, X) \rangle_\mathcal{H}| \leq l \right) \leq C l,\,\, \forall l > 0.$
%\end{align*}
\end{assumption}
The assumption is related to the behavior of the distribution of the covariates near the decision boundary $\{x: f_1(x) = f_0(x) \}$. As it is difficult to distinguish between the arms near the boundary, imposing such an assumption helps to control the contribution of incorrect decisions being made near the decision boundary. In the following theorem, we present an upper bound on the expected regret for the kernel $\epsilon$-greedy algorithm when the true mean reward functions lie in a finite-dimensional RKHS $\mathcal{H}$ and show it is almost minimax optimal.

\begin{theorem}
\label{theorem: finite_dim_with_Margin_condition}
%Assume that $\mathcal{H}$ is finite-dimensional and Assumption
Let $L=2$. Suppose 
\ref{assump:independenterrors}, \ref{assump: errorsindependentX}, \ref{assump: min_evalue}, and \ref{assump: AssumptionMargin} hold, $\normx{f_1 - f_0}_\mathcal{H} < \infty$, and for any $\delta>0$, $T\ge 1$, $\{\epsilon_s\}^T_{s=1}$ satisfies $$\frac{1}{t^2}\sum^t_{s=1}\frac{1}{\epsilon_s}\le \frac{\delta\eta}{4(L-1)d\kappa},\,\,0<t\le T,$$ where $d:=\emph{dim}(\mathcal{H})$. Then for $\delta>0$, $T\ge 1$, $\theta>0$, $0\le \zeta<1$, and 
% suppose $(\epsilon_s)_s$ satisfies $$\frac{1}{t^2}\sum^t_{s=1}\frac{1}{\epsilon_s}\le \frac{\delta\eta}{4(L-1)d\kappa},\,\,0<t\le T,$$ where $d:=\emph{dim}(\mathcal{H})$, and 
\begin{align}
\lambda_{t} = \left[\dfrac{1}{ t^2} \left(\sum_{s=1}^t \dfrac{1}{\epsilon_s}\right) \right]^{1/2},\,\, 0 < t \leq T,\label{eq: lamt_Thm6}
\end{align}
the following holds:
\begin{align}
   \E R_T & \leq \kappa \normx{f_1 - f_0}_\mathcal{H} \sum_{t=1}^T \dfrac{\epsilon_t}{2} \nonumber\\
   &\qquad+ \tilde{A}(\zeta, \kappa, \tilde{C}_0, C,  \tilde{C}_*) \left[T^{-\theta} \sum_{t=1}^T \left[\dfrac{1}{t^2} \sum_{s=1}^t \dfrac{1}{\epsilon_s}\right]^{1/2} +  T^{\theta \zeta}  \sum_{t=1}^T \left(\dfrac{1}{t^2} \sum_{s=1}^t \dfrac{1}{\epsilon_s}\right)^{(1+\zeta)/2}\right],\label{eq:finite_dim_regret_margin_con}
\end{align}
where $\tilde{A}(\zeta, \kappa, \tilde{C}_0, C, \tilde{C}_*)$ is a constant that depends only on its arguments and not on $T$ and $\{\epsilon_s\}_s$, $\tilde{C}_0$, and $\tilde{C}_*$ are defined in Theorem~\ref{thm: regret_finite_case_pre}, and $C$ is defined in \ref{assump: AssumptionMargin}. 
% = \sqrt{{(4(L-1)d \sigma^2)}/{\eta}}$,  $\tilde{C}_* = \max_{1\leq i \leq L} {(2\normx{f_i}_\mathcal{H}}/{\eta})$, and
% \begin{align*}
%     \tilde{A}(C, \kappa, \tilde{C}_0, \tilde{C}_*, \zeta) =   \frac{4}{1-\zeta} \max \left\{ C \kappa \max\{\tilde{C}_0, \tilde{C}_*\}, \kappa^{1+\zeta} \max\{\tilde{C}_0, \tilde{C}_*\}^{1+\zeta}\right\}.
% \end{align*}

% \textcolor{red}{Do you need all these assumtions? Do we need A3?}

\end{theorem}
% Then, for the choice of $\epsilon_t = t^{-\beta}$, $\frac{1}{3} < \beta < \frac{1}{2}$ and $\lambda_t = t^{(\beta-1)/2}$,  the proposed kernel $\epsilon$-greedy algorithm achieves $\E R_T = O(T^{1-\beta})$ as $T\rightarrow\infty$.
  % When restricting the RKHS to finite-dimensions, upon choosing the exploration probability to be $\epsilon_t = t^{-\beta}$ for $ 1/3 < \beta < 1/2$, note that, as $\beta \rightarrow \frac{1}{2}$, we get a regret rate of  $\tilde{O}_P(\sqrt{T})$.  
 \begin{remark}  
 %As in remark \ref{rem: Regret_upper_bound_finite_polyeps},
 The choice of $\epsilon_t = t^{-\beta}$ for some $0 < \beta < 1$ and $0 \le \zeta < 1$, reduces \eqref{eq:finite_dim_regret_margin_con} to
 \begin{align}
 \E R_T \lesssim  T^{1-\beta} +  T^{-\theta + \frac{\beta + 1}{2}} + T^{\frac{(\beta -1)(1+\zeta)}{2} + \theta \zeta + 1}. \label{eq: regret_finite_time_rem5}
     % \E R_T \lesssim \dfrac{\kappa \normx{f_1 - f_0}_\mathcal{H}}{2(1-\beta)} T^{1-\beta} + \tilde{A}(C, \kappa, \tilde{C}_0, \tilde{C}_*, \zeta) \left[T^{-\theta + \frac{\beta + 1}{2}} + T^{\frac{(\beta -1)(1+\zeta)}{2} + \theta \zeta + 1} \right]. \label{eq: regret_finite_time_rem5}
 \end{align}
We first balance the second and the third term in \eqref{eq: regret_finite_time_rem5}, and then balance the resulting rate with the first term in \eqref{eq: regret_finite_time_rem5}. As a result we obtain $\E R_T = {O}(T^{1-\beta})$ where $\frac{3}{7} \le \beta < \frac{1}{2}$. Therefore, the best regret we can achieve in the setting of Theorem~\ref{theorem: finite_dim_with_Margin_condition} is $T^{\frac{1}{2}+\varepsilon}$, $\varepsilon>0$, which is almost minimax optimal with $T^{1/2}$ being the minimax optimal rate for linear MABC/contextual bandits \citep{chen2021statistical}. 
 % Details for the proof of \eqref{eq: regret_finite_time_rem5} can be found in Section \ref{proof: remark_finitedimregret_margin}.
     % Note that this $\tilde{O}_P(\sqrt{T})$ is the same rate as \cite{chen2021statistical} for linear bandits, which is the only work that studies the regret rate for $\epsilon$-greedy strategy for linear bandits. 
 \end{remark}

%%%%%%%%%%%%%%%%%%%%%%%%%%%%%%%%%%%%%% Implementation and simulations %%%%%%%%%%%%%%%%%%%%%%%%%%%%%%%%%%%%%%%%%%%%%%%%%

\section{Numerical experiments}
\label{sec: simulations}
In this section, we compare the performance of the proposed kernel $\epsilon$-greedy strategy with other MABC algorithms through numerical experiments. We use a Gaussian kernel parameterized by $\gamma>0$, i.e., 
$$K(x, x^\prime) = \exp\left(-\gamma^2 \normx{x - x^\prime}^2\right),$$
for the kernel bandit algorithms. We compare the following four strategies (with parameters defined the parentheses to be tuned and selected) based on the cumulative regret incurred until time horizon $T = 1000$: 
\begin{enumerate}
    \item Kernel $\epsilon$-greedy algorithm with Gaussian kernel (regularization parameter $\lambda_t$, length-scale parameter $\gamma$)
    \item \begin{enumerate}
        \item Kernel $\epsilon$-greedy algorithm with linear kernel,
        \item Weighted linear $\epsilon$-greedy algorithm of \cite{chen2021statistical} with ridge regression estimator,
   \end{enumerate}
     \item Weighted linear $\epsilon$-greedy algorithm of \cite{chen2021statistical} (i.e., without regularization), 
    \item Kernel Upper Confidence Bound (Kernel UCB) algorithm of \cite{valko2013finite} modified for the MABC framework with Gaussian kernel (exploration parameter $\tau$, regularization parameter $\lambda_t$, length-scale parameter $\gamma$).
\end{enumerate}
Note that 2(a) and 2(b) are essentially the same algorithm since our estimator \eqref{sec: IPWR_estimator} with a linear kernel is just a dual representation of the ridge regression version of the weighted linear $\epsilon$-greedy estimator of \cite{chen2021statistical}. This is also reflected in the regret curves in Figures \ref{fig:mean_reward_funcs1D}(b),(d) and Figure \ref{fig: reward_func_values_setting4}(a),(b). For all the $\epsilon$-greedy based algorithms (1-3), we choose the exploration probability sequence $\epsilon_t = \max\{\frac{t^{-1/2}\log{(t)}}{10}, 0.02\}$, which is the same choice as used by \cite{chen2021statistical} in their simulation setup. For algorithms 1 and 4, i.e., the kernel $\epsilon$-greedy and kernel UCB with Gaussian kernel, we do cross-validation as described in Section \ref{subsec: choice of kernel} to determine the right choice of the parameters in the parentheses. 

We consider four simulated data experiments for $d$-dimensional context space for $d \in \{1,2,3\}$ and for $L =2$ arms. All algorithms are run until the time horizon $T = 1000$ with initial random exploration time, $t_0 = 50$. For the initialization phase until $t_0 = 50$, we randomly assign both arms 25 times each. In  Figures \ref{fig:mean_reward_funcs1D}(b), \ref{fig:mean_reward_funcs1D}(d) and Figures \ref{fig: reward_func_values_setting4}(a), \ref{fig: reward_func_values_setting4}(b), we plot the  average cumulative regret (averaged over 25 runs) over time for the four strategies. It can be seen that in all four settings, the kernel $\epsilon$-greedy with Gaussian kernel outperforms the linear strategies by significantly reducing the regret incurred.
 
 \textbf{Setting 1}: We let $d = 1$ and sample $X_t \overset{\text{i.i.d.}}{\sim} \text{Unif}(-1,1)$ for $t = 1,\hdots, T$.  The mean reward functions considered are:
 $f_1(x) = \sin(\pi x)$ and $f_2(x) = \cos(\pi x)$ for $-1 < x < 1$ as in  Figure \ref{fig:mean_reward_funcs1D}(a). In Figure \ref{fig:mean_reward_funcs1D}(b), note that kernel UCB performs the poorest amongst all four strategies while kernel $\epsilon$-greedy performs the best.

\textbf{Setting 2}: We consider a `chessboard' like setup (see Figure \ref{fig:mean_reward_funcs1D}(c)) similar to the experimental setup of \cite{zenati2022efficient}, where $d = 2$, and the mean reward functions are $f_1(x) = 1$ and $f_2(x) = 1$ in the green and red regions, respectively, and $0$ elsewhere. Here, we sample each component of the covariates $X_t \in \mathbb{R}^2, t = 1,\hdots, T$ independently from Unif(-1,1) distribution. Note that, in Figure \ref{fig:mean_reward_funcs1D}(d), kernel UCB and kernel $\epsilon$-greedy perform better than the weighted $\epsilon$-greedy linear algorithms. Both these algorithms give comparable performance with the former being slightly better. 

\textbf{Settings 3 and 4:} We consider two arms, $L = 2$, and covariate dimension, $d = 3$. For both these settings, we follow the data generation process of \cite{chen2021statistical}, wherein the covariates $X_t$ are sampled i.i.d.~from a truncated normal distribution supported on $[-10, 10]$ with mean zero and scale parameter one. In Setting 3, we use a discretized version of the `Bump' synthetic environment as in the experimental setup of \cite{zenati2022efficient}.  The rewards are generated using the functions,
$f_a(x) = \max(0, 1 - \normx{a - a^*}_1 - \langle w^*, x - x^* \rangle_2), a = 1,2$ for some fixed $a^*, x^*$ and $w^*$. We fix $a^* = 2$ and randomly generate $d$-dimensional vectors $x^*$ and $w^*$.
In Setting 4, we consider the following function:
$f_a(x) = I\{\normx{x - a + 0.5}_1 < 4 \} + 0.5I\{\normx{x - (a - 1)}_1 < 4\}$  for $a = 1, 2$. For setting 3, as can be seen in Figure \ref{fig: reward_func_values_setting4}(a), both kernel UCB and kernel $\epsilon$-greedy perform at par with each other and result in significantly lower regret than the linear algorithms. For setting 4 as can be seen in Figure \ref{fig: reward_func_values_setting4}(b), kernel $\epsilon$-greedy performs better than the kernel UCB algorithm and significantly outperforms the linear algorithms. 

\begin{figure}[t]
    \centering 
\begin{tabular}{c c}
    \includegraphics[scale = 0.36]{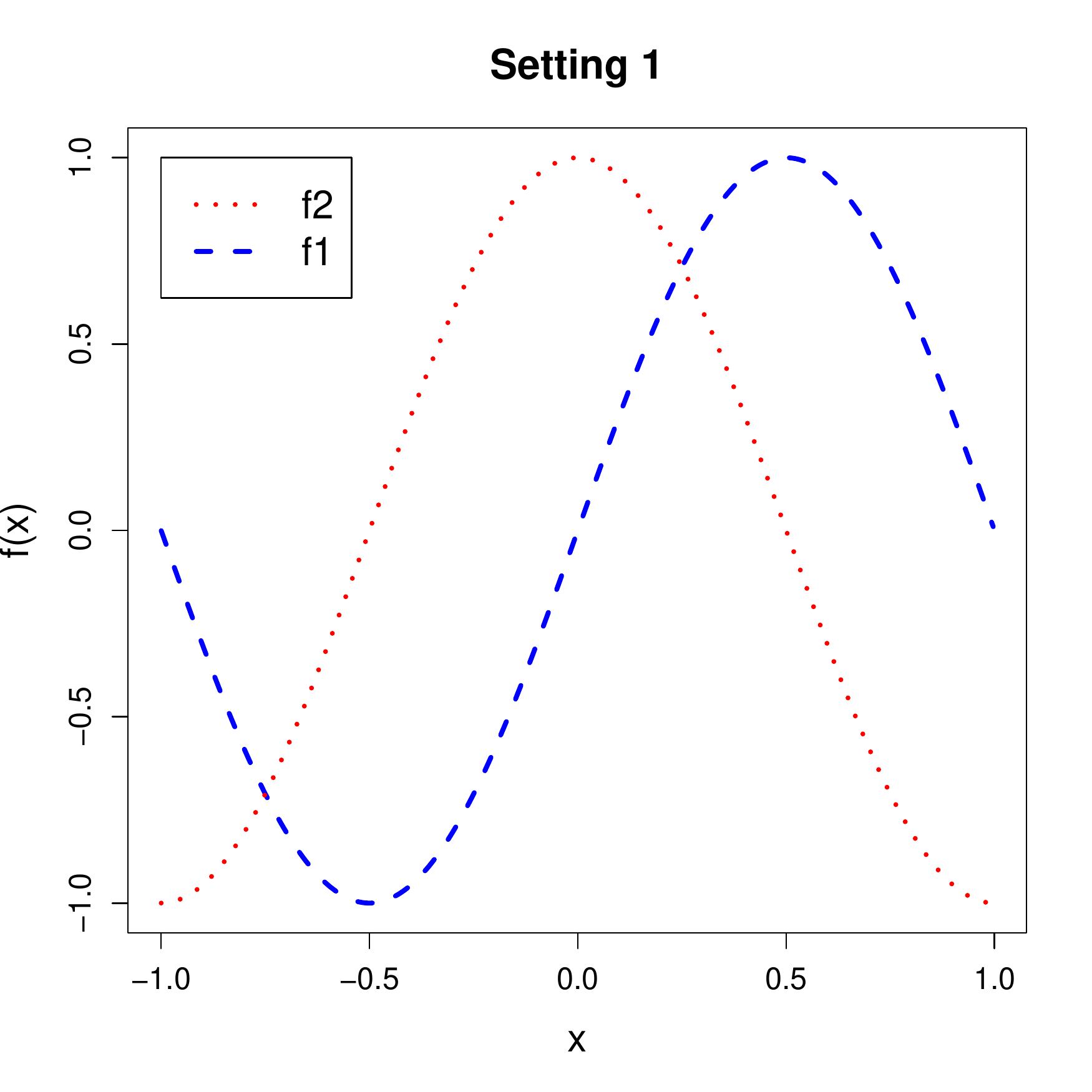}     &     \includegraphics[scale = 0.27]{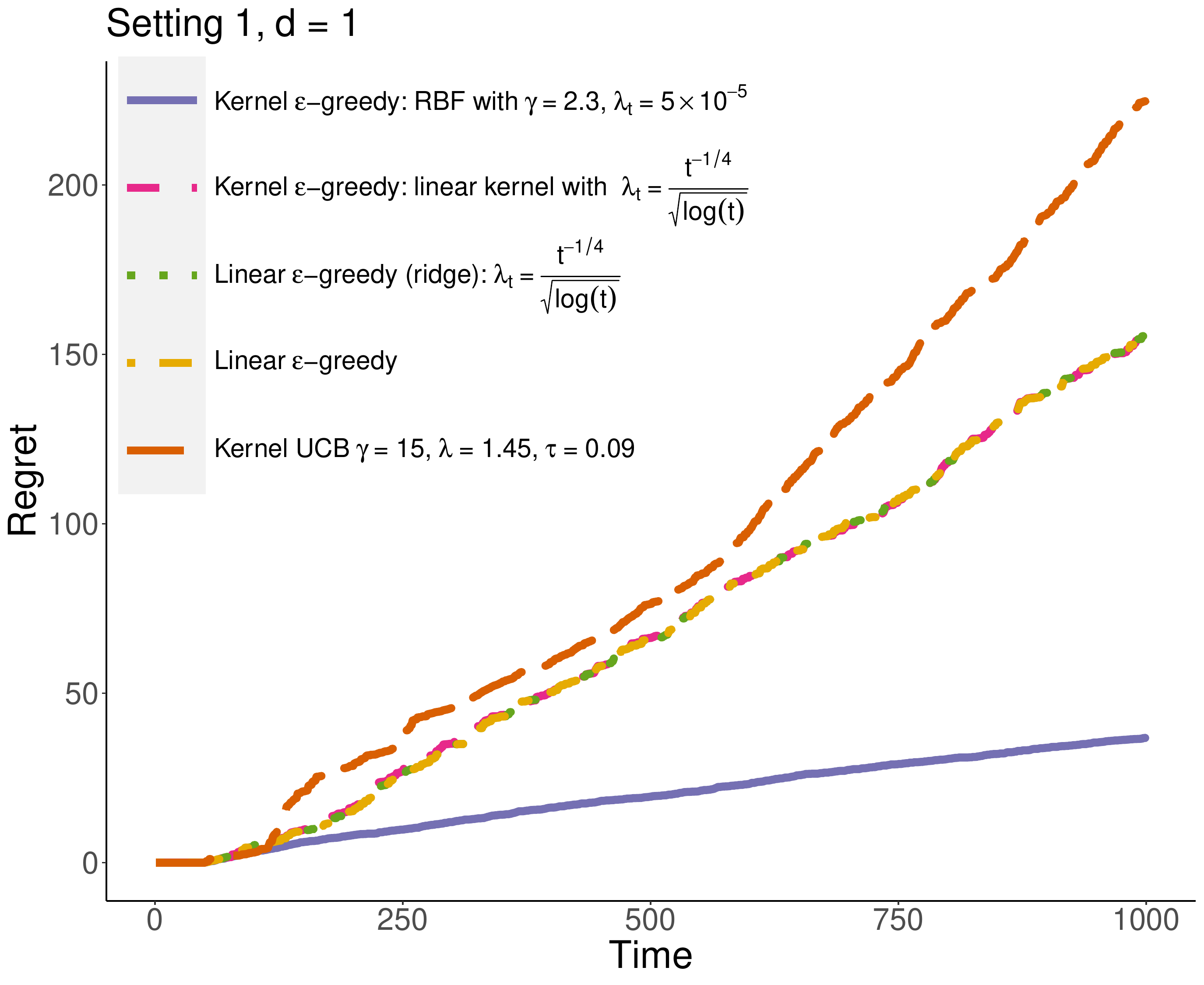} \\
    (a) Setting 1: Mean Reward functions & (b) Average cumulative regret over time\\
    \vspace{4mm}\\
 \includegraphics[scale = 0.36]{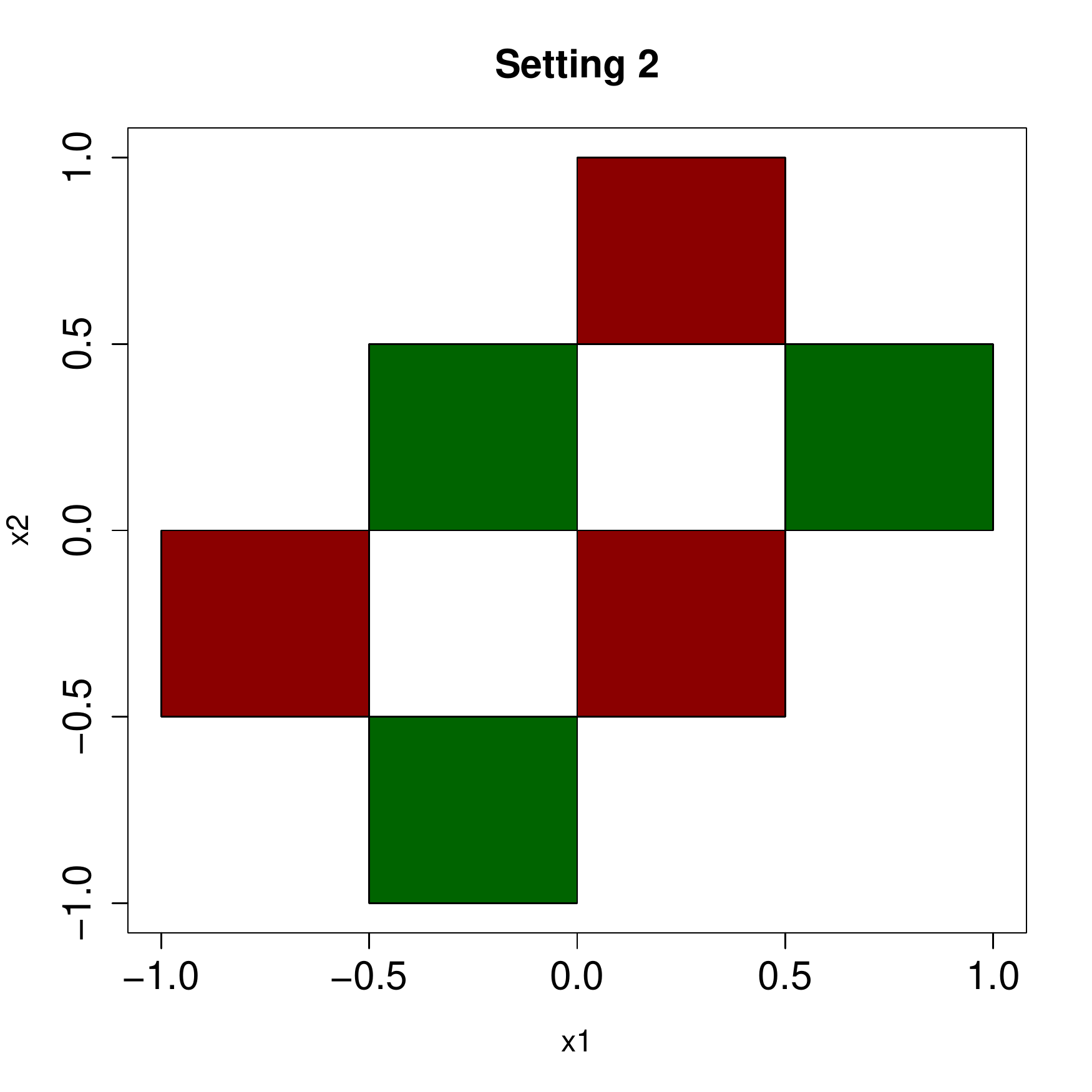} &
    \includegraphics[scale = 0.27]{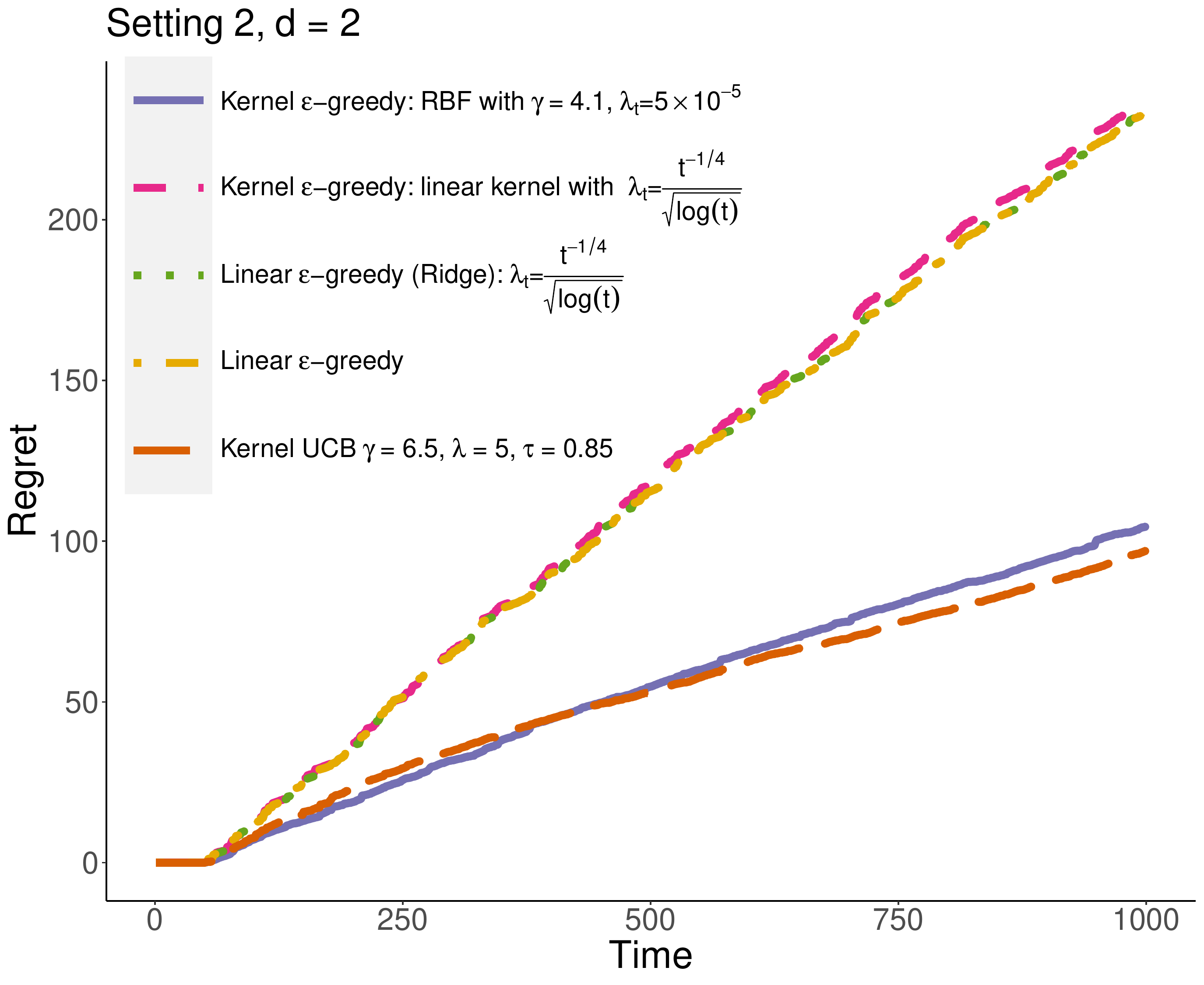}\\
    (c) Setting 2: Mean reward functions & (d) Average cumulative regret over time.\\
\end{tabular}
     \caption{Left: Mean reward functions for settings 1 and 2. Right: Average cumulative regret over 25 runs for the five strategies over time for $T = 1000$.}
    \label{fig:mean_reward_funcs1D}
    \end{figure}

% \begin{figure}[H]
%     \centering
%     \includegraphics[scale = 0.36]{Figures/d1/21March_setting1_rewardFunc.pdf}
%     \includegraphics[scale = 0.27]{Figures/d1/June18_d1_sig_p2_T1k_KernelLinEpsGreedy_sinusoidal_regret_T_25.pdf}
%     \caption{Left: Mean reward functions for setting 1, Right: Average regret for kernel $\epsilon$-greedy and linear $\epsilon$-greedy over time.}
%     \label{fig:mean_reward_funcs1D}
% \end{figure}

% \begin{figure}[H]
%     \centering
%     \includegraphics[scale = 0.36]{Figures/d2/26March_setting2_true.pdf}
%     \includegraphics[scale = 0.27]{Figures/d2/June26_d2_sig_p2_T1k_KernelLinEpsGreedy_chessboard_regret_T_25.pdf}
%     \caption{Left: Mean reward function for setting 2 where $f(x) = 1$ in the green and red regions and zero elsewhere. Right: Average regret for kernel $\epsilon$-greedy and linear $\epsilon$-greedy algorithms over time.}
%     \label{fig:mean_reward_funcs2D}
% \end{figure}

\begin{figure}[t]
    \centering
    \begin{tabular}{c c}
           \includegraphics[scale = 0.27]{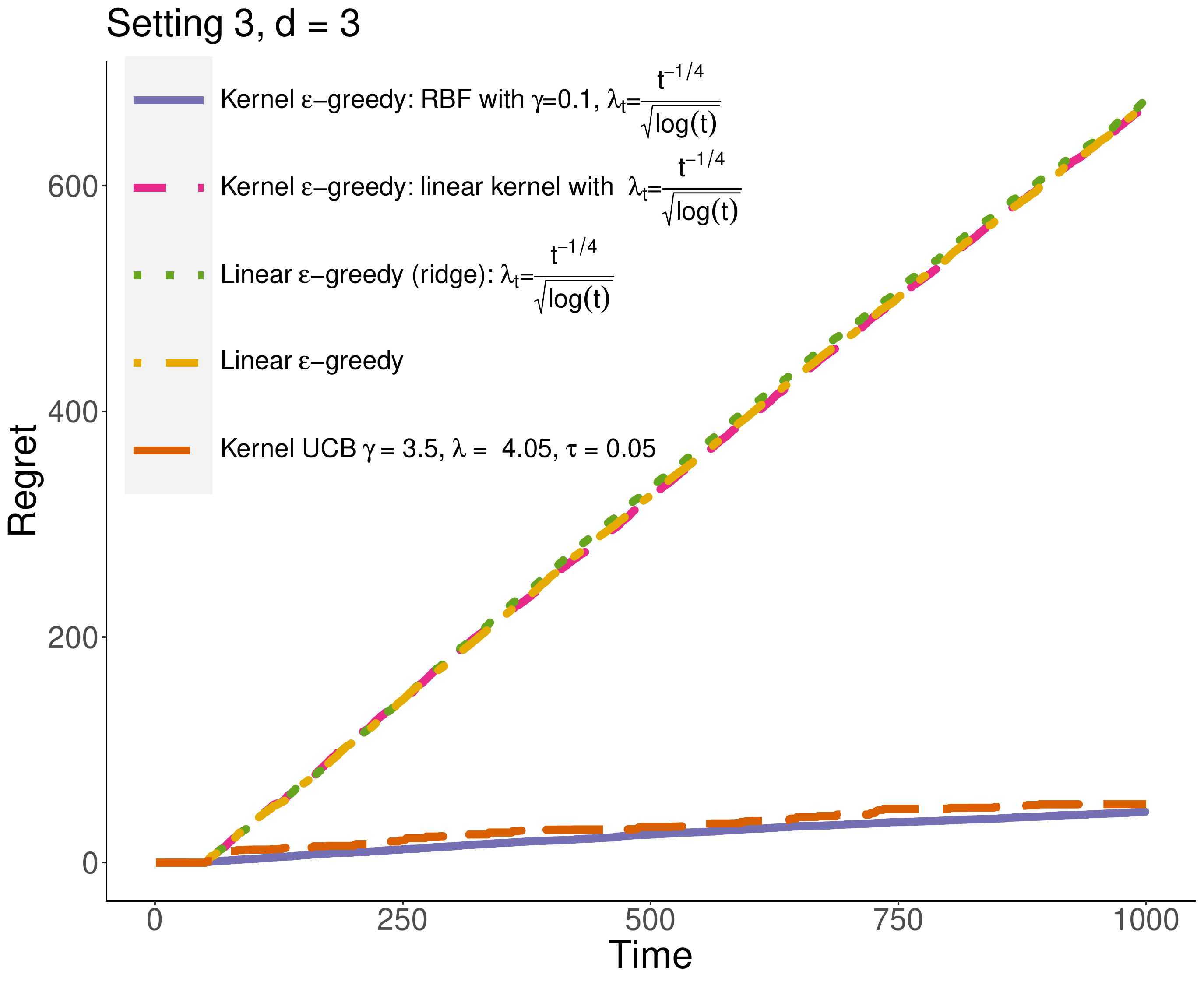} &
           \includegraphics[scale = 0.27]{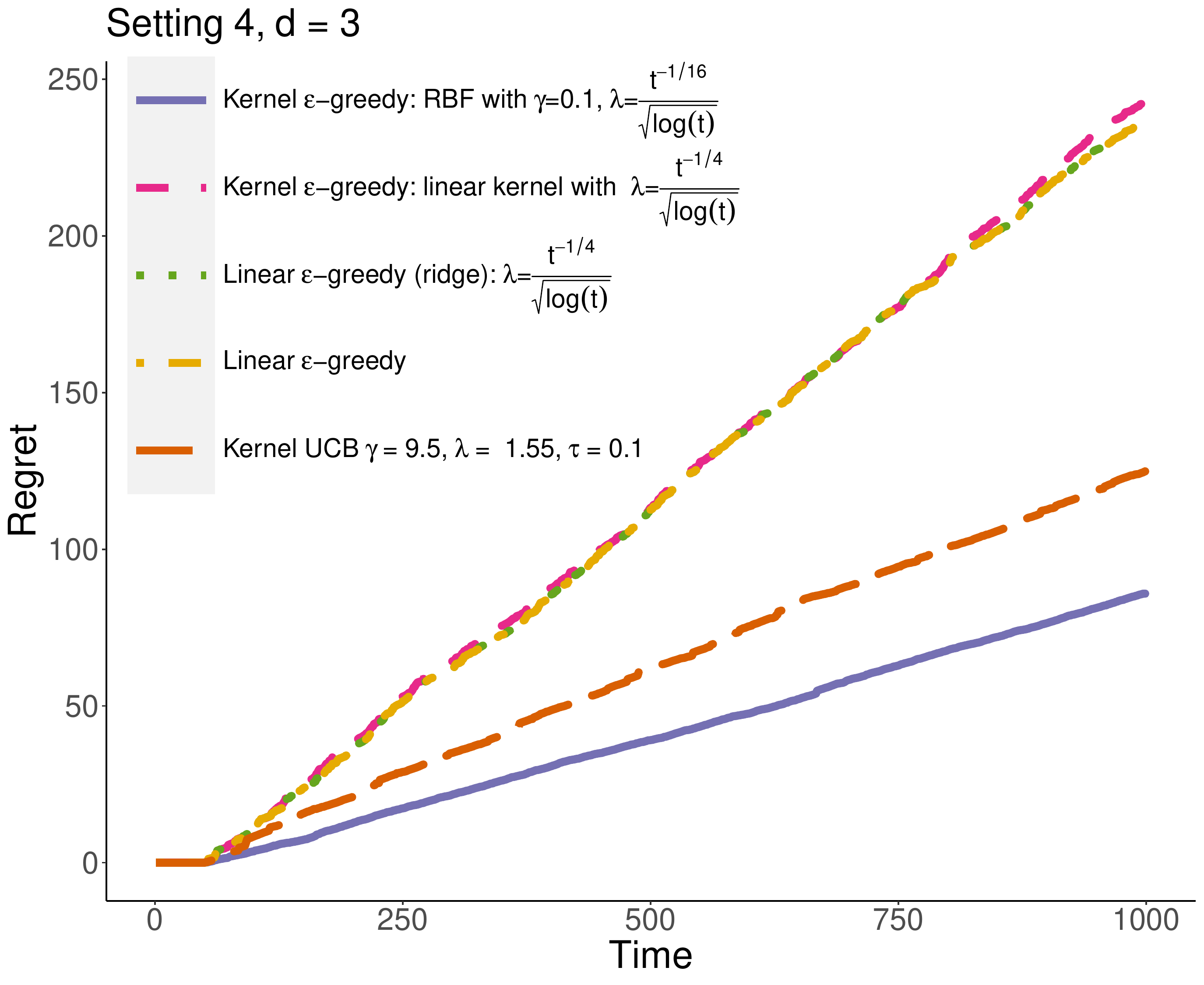}\\
           (a) Setting 3: Average cumulative regret  & (b) Setting 4: Average cumulative regret
      \end{tabular}     
    \caption{Average cumulative regret over time for kernel $\epsilon$-greedy algorithm with Gaussian kernel, linear kernel (with and without regularization), and kernel UCB with Gaussian kernel.}
    \label{fig: reward_func_values_setting4}
\end{figure}

\subsection{Choice of kernel parameters} \label{subsec: choice of kernel}
In this section, we describe the methodology we use to tune and select the parameters in the proposed kernel $\epsilon$-greedy algorithm and the kernel UCB algorithm of \cite{valko2013finite}. Note that for the kernel $\epsilon$-greedy algorithm with Gaussian kernel, we need to tune two parameters, $\lambda_t$ and $\gamma$, while for the kernel UCB algorithm, we need to tune three parameters, $\lambda$, $\gamma$, and  the exploration parameter, $\tau$. Below, we describe the steps for tuning the two parameters in the former, while the same methodology is used to tune the three parameters in the latter.

For selecting the two parameters $\gamma$ and $\lambda$ in implementing the kernel $\epsilon$-greedy algorithm (corresponding to the purple line in Figures \ref{fig:mean_reward_funcs1D}(b),(d) and Figure \ref{fig: reward_func_values_setting4}), we use the following cross-validation approach.  We consider the following choices for $\lambda_t$ and $\gamma$, respectively: 
\begin{align*}
 \lambda_t & \in \left\{\frac{t^{-1/2}}{\sqrt{\log t}}, \frac{t^{-1/4}}{\sqrt{\log t}}, \frac{t^{-1/6}}{\sqrt{\log t}}, \frac{t^{-1/8}}{\sqrt{\log t}}, \frac{t^{-1/16}}{\sqrt(\log t)}, 5\times10^{-5}, 0.005, 0.5\right\},\,\,\text{and}\\
\gamma & \in \{0.1,0.3, 0.5, \hdots, 5\}.   
\end{align*}
Then, we use the following $k$-fold cross-validation approach with $k = 10$:
\begin{enumerate}
    \item Sample $X \in \mathbb{R}^{T(k+1) \times d}$.
    \item Split this data into $(k+1)$ subsets, each of size $T$. First $k$ subsets are used as training datasets and the $(k+1)$th dataset is used as the test data.
    \item For each pair of $(\lambda_t, \gamma)$, run the algorithms independently on each of the training sets (first $k$ subsets) and note the regret incurred. Take the average cumulative regret across all the $k$ subsets.
    \item Choose $(\lambda_t^*, \gamma^*)$ that minimized the average regret at time $T$. 
    \item Run the kernel $\epsilon$-greedy algorithm using the Gaussian kernel with length-scale $\gamma^*$ and regularization parameter $\lambda_t^*$ on the test dataset and repeat the experiment on this dataset 25 times. The results reported in Figures \ref{fig:mean_reward_funcs1D}(b),(d) and \ref{fig: reward_func_values_setting4}(a),(b) are the averages of the cumulative regret over these 25 runs.
\end{enumerate}
% \textcolor{blue}{No, we use different ranges, 1) $\lambda$ is a constant in their algorithm, but in ours we have it as function of $t$, 2) I realized that we needed larger values of $\gamma$ for kernel UCB, so considered a wider grid. I will mention these ranges here. }
For Kernel UCB, we follow the same cross-validation approach but with the following grid choices for $\lambda, \gamma$ and $\tau$:
\begin{align*}
    \lambda & \in \{0.05, 0.15, 0.25, \hdots, 5\},\,\,
    \gamma  \in \{0.5, 1.5, 2.5, \hdots, 15\}, \text{and}\,\,
    \tau \in \{0.05, 0.1, 0.15, \hdots, 0.9\},
\end{align*}
where it has to be noted that Kernel UCB uses $\lambda$ that does not vary with time (as suggested in \citealp{valko2013finite}) in contrast to ours which is time dependent. 

For cases where a linear approximation would lead to a good classification of arms, our algorithm still performs on par with the linear algorithms. However if one knows that the true model is linear, it might be computationally efficient to opt for the linear $\epsilon$-greedy algorithm. If that is not the case, we can say that using the kernel $\epsilon$- greedy can be beneficial in most settings with the right choice of kernel.

% \section{Real-data example: Warfarin dosage}

\section{Discussion}
In this work, we propose the kernel $\epsilon$-greedy algorithm for the multi-armed bandit problem with covariates with finitely many arms. To our knowledge, this is the first work considering kernel methods in a nonparametric MABC framework, which is different from the kernel bandits framework of \citep{valko2013finite, srinivas2010gaussian}, in terms of the underlying modeling framework.  We provided upper bounds on the estimation error for the proposed online regression estimator and provided sub-linear regret rates for the proposed algorithm. While the kernelized versions of UCB and Thompson sampling have been well-studied, to our knowledge, this is the first attempt at studying the kernelized $\epsilon$-greedy algorithm. The theoretical analysis presented is novel, as it utilizes the intrinsic properties of the RKHS and exploits the simplicity of the $\epsilon$-greedy algorithm, resulting in upper bounds that do not depend on quantities like the maximum information gain, like previous works. An advantage of the analysis is that we achieve sub-linear regret bounds for wide choices of kernels, along with achieving state-of-the-art regret bounds in a finite-dimensional setting, even when the regressors are not linear. Simple strategies like $\epsilon$-greedy are easy to implement and deeper theoretical understanding helps in supporting their application in real-life sequential decision-making problems. From a practical point of view, addressing computational challenges in implementing the $\epsilon$-greedy algorithm needs further research. One way to address time and computational complexity would be by using incremental Nyst\"rom approximations as done by \cite{zenati2022efficient}. While we employ cross-validation to tune the kernel parameters and regularization parameters, it can be time-consuming to do an exhaustive search. Therefore, new computational techniques need to be devised to help with parameter tuning. Another possible future direction is to study the effect of delayed feedback on the kernel $\epsilon$-greedy strategy, similar to \cite{vakili2023delayed}.
% A future direction is to develop sharp Bernstein-type concentration inequality for operator norm of a self-adjoint Hilbert-Schmidt operator-valued random element defined on a separable Hilbert space, which would provide similar results as in this paper but with sub-Gaussian concentration. } \textcolor{red}{Maybe we should remove this line? } \textcolor{green}{yes.}

%%%%%%%%%%%%%%%%%%%%%%%%%%%%%%%%%%%%%%%%%%%%%%%%%%% Proofs %%%%%%%%%%%%%%%%%%%%%%%%%%%%%%%%%%%%%%%%%%%%%%%%%%%%%%%%%%%%%%%
\section{Proofs}
\label{sec: proofs}
In this section, we present the proofs of the main results of the paper. Before we present the proofs, we present a result that is used in many of these proofs.
\begin{lemma}
\label{lem: Lamma_unbiased_covariance}
$\E(\hat{\Sigma}_{i,t}) = \Sigma$, where $\hat{\Sigma}_{i,t}$ is defined in \eqref{eq: LSigma_hat_it}.
\end{lemma}
\begin{proof}
Consider,
\begin{align*}
	\E(\hat{\Sigma}_{i,t}) &= \E \left[ \dfrac{1}{t} \sum_{s=1}^t \dfrac{I\{\hat{a}_s = i\}}{P(\hat{a}_s = i| \mathcal{F}_{s-1},X_s)} k(\cdot, X_s)\otimes k(\cdot, X_s) \right]\\
	&= \dfrac{1}{t} \sum_{s=1}^t \E \left[\dfrac{I\{\hat{a}_s = i\}}{P(\hat{a}_s = i| \mathcal{F}_{s-1},X_s)} k(\cdot, X_s)\otimes k(\cdot, X_s) \right]\\
	&= \dfrac{1}{t} \sum_{s=1}^t \E \left[ \E \left( \dfrac{I\{\hat{a}_s = i\}}{P(\hat{a}_s = i| \mathcal{F}_{s-1},X_s)} k(\cdot, X_s)\otimes k(\cdot, X_s) \middle| \mathcal{F}_{s-1},X_s\right) \right]\\
	&= \dfrac{1}{t} \sum_{s=1}^t \E \left[  \dfrac{P(\hat{a}_s = i| \mathcal{F}_{s-1},X_s)}{P(\hat{a}_s = i| \mathcal{F}_{s-1},X_s)} k(\cdot, X_s)\otimes k(\cdot, X_s)  \right]\\
 	&= \dfrac{1}{t} \sum_{s=1}^t \E (k(\cdot, X_s)\otimes k(\cdot, X_s))= \Sigma, \ \text{for} \ i = 1, \hdots, L,
\end{align*}
where the third equality follows from the law of iterated expectations.
\end{proof}

\subsection{Proof of Theorem~\ref{thm: theorem1}}
%\begin{proof}[\textbf{Proof of Theorem~\ref{thm: theorem1}}]
\label{proof: proof1}
Without loss of generality, we will assume $i = 1$. Let $\hat{\pi}_s := P(\hat{a}_s = 1| \mathcal{F}_{s-1}, X_s)$. Then, 
\begin{align*}
\fh_{1,t} - f_1 & = \left( \hat{\Sigma}_{1,t} + \lambda I \right)^{-1} \dfrac{1}{t} \sum_{s=1}^t \dfrac{I\{\hat{a}_s = 1\}}{\hat{\pi}_s} k(\cdot, X_s) y_s - f_1 \nonumber\\
& = \left( \hat{\Sigma}_{1,t} + \lambda I \right)^{-1} \dfrac{1}{t} \sum_{s=1}^t \dfrac{I\{\hat{a}_s = 1\}}{\hat{\pi}_s} k(\cdot, X_s) [\langle k(\cdot, X_s), f_1 \rangle_\mathcal{H} + e_s] - f_1 \nonumber\\
& = \left( \hat{\Sigma}_{1,t} + \lambda I \right)^{-1} \left[ -\lambda f_1 + \dfrac{1}{t} \sum_{s=1}^t \dfrac{I\{\hat{a}_s = 1\}}{\hat{\pi}_s} k(\cdot, X_s) e_s \right]\nonumber \\
& = \left(\hat{\Sigma}_{1,t} + \lambda I \right)^{-1/2} \left(\hat{\Sigma}_{1,t} + \lambda I \right)^{-1/2} \left({\Sigma} + \lambda I \right)^{1/2} \left({\Sigma} + \lambda I \right)^{-1/2} \nonumber\\
&\qquad\qquad\times\left[ -\lambda f_1 + \dfrac{1}{t} \sum_{s=1}^t \dfrac{I\{\hat{a}_s = 1\}}{\hat{\pi}_s} k(\cdot, X_s) e_s \right]. 
%\label{eq: est_error_decomp}
\end{align*}
This implies that, 
\begin{align}
\normx{\fh_{1,t} - f_1}_{\mathcal{H}} & \leq \norm{\left(\hat{\Sigma}_{1,t} + \lambda I \right)^{-1/2} \left(\hat{\Sigma}_{1,t} + \lambda I \right)^{-1/2} \left({\Sigma} + \lambda I \right)^{1/2}}_\infty \nonumber\\
& \quad \quad \times \left\| \left({\Sigma} + \lambda I \right)^{-1/2}  \left[ -\lambda f_1 + \dfrac{1}{t} \sum_{s=1}^t \dfrac{I\{\hat{a}_s = 1\}}{\hat{\pi}_s} k(\cdot, X_s) e_s \right]\right\|_{\mathcal{H}} \nonumber\\
& \leq \normx{(\Sigmah_{1,t} + \lambda I)^{-1/2}}_{\infty} \normx{ (\hat{\Sigma}_{1,t} + \lambda I)^{-1/2} (\Sigma + \lambda I)^{1/2}}_{\infty} \nonumber\\
& \quad \quad \times \left\| \left({\Sigma} + \lambda I \right)^{-1/2}  \left[ -\lambda f_1 + \dfrac{1}{t} \sum_{s=1}^t \dfrac{I\{\hat{a}_s = 1\}}{\hat{\pi}_s} k(\cdot, X_s) e_s \right]\right\|_{\mathcal{H}}\le\frac{\mathcal{S}_1\mathcal{S}_2}{\sqrt{\lambda}},\label{eq: norm_ineq_estimation_error}
\end{align}
% \textcolor{red}{Is the above step ok?}
where
\begin{align*}
	\mathcal{S}_1 &:= \left\Vert\left(\hat{\Sigma}_{1,t} + \lambda I \right)^{-1/2} \left({\Sigma} + \lambda I \right)^{1/2}\right\Vert_\infty,\,\,\text{and}\\
	\mathcal{S}_2 &:= \left\Vert\left({\Sigma} + \lambda I \right)^{-1/2} \left[ -\lambda f_1 + \dfrac{1}{t} \sum_{s=1}^t \dfrac{I\{\hat{a}_s = 1\}}{\hat{\pi}_s} k(\cdot, X_s) e_s \right]\right\Vert_\mathcal{H}.
\end{align*}
We now bound $\mathcal{S}_1$ and $\mathcal{S}_2$. By defining $B_t := (\Sigma + \lambda I)^{-1/2} (\Sigma - \hat{\Sigma}_{1,t}) (\Sigma + \lambda I)^{-1/2}$, we have
\begin{align*}
\mathcal{S}_1&=\normx{(\Sigmah_{1,t} + \lambda I)^{-1/2} (\Sigma + \lambda I)^{1/2}}_\infty = \normx{(\Sigma + \lambda I)^{1/2} (\Sigmah_{1,t} + \lambda I)^{-1} (\Sigma + \lambda I)^{1/2}}_\infty^{1/2} \nonumber \\
&= \normx{(I - B_t)^{-1}}_\infty^{1/2} 
 \leq (1 - \normx{B_t}_\infty)^{-1/2}, %\label{eq: part1_Btbound}
\end{align*}
{where the last inequality follows from the fact that $(I - B_t)^{-1} \preceq (1-\normx{B_t}_\infty)^{-1}I$ whenever $\normx{B_t}_\infty < 1$}.
Note, $\normx{B_t}_\infty \leq \normx{B_t}_{HS}$, where $\normx{\cdot}_{HS}$ denotes the Hilbert-Schmidt norm. Using Chebyshev's inequality we obtain,
\begin{align*}
P(\normx{B_t}_{HS} \geq \epsilon) & \leq \dfrac{\mathbb{E}\normx{B_t}_{HS}^2}{\epsilon^2}.
\end{align*}
It follows from Lemma~\ref{lem: Lamma_unbiased_covariance} that
\begin{align}
\E\normx{B_t}_{HS}^2 & = \E\normx{(\Sigma + \lambda I)^{-1/2} (\hat{\Sigma}_{1,t} - \Sigma) (\Sigma + \lambda I)^{-1/2}}_{HS}^2 \nonumber \\
&= \E \normx{(\Sigma + \lambda I)^{-1/2} \hat{\Sigma}_{1,t} (\Sigma + \lambda I)^{-1/2}}_{HS}^2 - \normx{(\Sigma + \lambda I)^{-1/2} \Sigma (\Sigma + \lambda I)^{-1/2}}_{HS}^2.  \label{eq: B_t_eq1}
\end{align}
Define $N_{\Sigma, 2} (\lambda):= \normx{(\Sigma + \lambda I)^{-1/2} \Sigma (\Sigma + \lambda I)^{-1/2}}_{HS}^2$ and consider,
\begin{align*}
\E&\normx{(\Sigma + \lambda I)^{-1/2} \hat{\Sigma}_{1,t} (\Sigma + \lambda I)^{-1/2}}_{HS}^2  \\
&= \E \left\langle (\Sigma + \lambda I)^{-1/2} \Sigmah_{1,t} (\Sigma + \lambda I)^{-1/2}, (\Sigma + \lambda I)^{-1/2} \hat{\Sigma}_{1,t} (\Sigma + \lambda I)^{-1/2}\right\rangle_{HS}\\
&= \E \tr \left[ (\Sigma + \lambda I)^{-1/2} \Sigmah_{1,t} (\Sigma + \lambda I)^{-1} \Sigmah_{1,t} (\Sigma + \lambda I)^{-1/2} \right]\\
&= \E \tr \left[ (\Sigma + \lambda I)^{-1} \Sigmah_{1,t} (\Sigma + \lambda I)^{-1} \Sigmah_{1,t}  \right].
\end{align*}
Now, plugging in the definition of $\Sigmah_{1,t}$ from \eqref{eq: LSigma_hat_it} and defining $\tauh_s = I\{\ahat_s = 1\}/\pih_s$, we obtain
% \begin{align*}
% \E&\normx{(\Sigma + \lambda I)^{-1/2} \hat{\Sigma}_{1,t} (\Sigma + \lambda I)^{-1/2}}_{HS}^2  \\
% &= \E \tr \left[(\Sigma + \lambda I)^{-1} \dfrac{1}{t} \sum_{s=1}^t \dfrac{I\{\hat{a}_s = 1\}}{\pih_s} k(\cdot, X_s)\otimes k(\cdot, X_s) (\Sigma + \lambda I)^{-1} \dfrac{1}{t} \sum_{\ell=1}^t \dfrac{I\{\hat{a}_\ell = 1\}}{\pih_\ell} k(\cdot, X_\ell)\otimes k(\cdot, X_\ell) \right].
% \end{align*}
% Let, $\tauh_s = I\{\ahat_s = 1\}/\pih_s$ where recall that $\pih_s = P(\hat{a}_s = 1| \mathcal{F}_{s-1}, X_s)$. Therefore,
\begin{align*}
\E&\normx{(\Sigma + \lambda I)^{-1/2} \hat{\Sigma}_{1,t} (\Sigma + \lambda I)^{-1/2}}_{HS}^2  \\
&= \E  \left[\dfrac{1}{t^2} \sum_{s=1}^t \sum_{\ell=1}^t \tauh_s \tauh_\ell \tr \left( (\Sigma + \lambda I)^{-1}  k(\cdot, X_s)\otimes k(\cdot, X_s) (\Sigma + \lambda I)^{-1}   k(\cdot, X_\ell)\otimes k(\cdot, X_\ell) \right)  \right].
\end{align*}
By breaking the double sum in the above expression into the cases when, (1) $s=\ell$, (2) $s > \ell$, and (3) $s < \ell$, yields
\begin{align}
\E&\normx{(\Sigma + \lambda I)^{-1/2} \hat{\Sigma}_{1,t} (\Sigma + \lambda I)^{-1/2}}_{HS}^2  \nonumber\\
&=\E  \left[\dfrac{1}{t^2} \sum_{s=1}^t \tauh_s^2 \tr \left( (\Sigma + \lambda I)^{-1}  k(\cdot, X_s)\otimes k(\cdot, X_s) (\Sigma + \lambda I)^{-1}   k(\cdot, X_s)\otimes k(\cdot, X_s) \right)  \right] \nonumber\\
& \quad  + \E  \left[\dfrac{1}{t^2} \sum_{\ell = 1}^{t-1} \sum_{s = \ell + 1}^t \tauh_s \tauh_\ell \tr \left( (\Sigma + \lambda I)^{-1}  k(\cdot, X_s)\otimes k(\cdot, X_s) (\Sigma + \lambda I)^{-1}   k(\cdot, X_\ell)\otimes k(\cdot, X_\ell) \right)  \right] \nonumber\\
&\quad \quad + \E  \left[\dfrac{1}{t^2} \sum_{\ell = 2}^t \sum_{s = 1}^{\ell - 1} \tauh_s \tauh_\ell \tr \left( (\Sigma + \lambda I)^{-1}  k(\cdot, X_s)\otimes k(\cdot, X_s) (\Sigma + \lambda I)^{-1}   k(\cdot, X_\ell)\otimes k(\cdot, X_\ell) \right)  \right]\nonumber\\
&=\circled{\footnotesize{1}}+\circled{\footnotesize{2}}+\circled{\footnotesize{3}},
\label{eq: part1_three_cases}
\end{align}
where we bound $\circled{\footnotesize{1}}$--$\circled{\footnotesize{3}}$ as follows.
% Let us now consider each of these summations separately.\\
% \textbf{Case 1. $\ell = s$:}
\begin{align}
% \E & \left[\dfrac{1}{t^2} \sum_{s=1}^t \tauh_s^2 \tr \left( (\Sigma + \lambda I)^{-1}  k(\cdot, X_s)\otimes k(\cdot, X_s) (\Sigma + \lambda I)^{-1}   k(\cdot, X_s)\otimes k(\cdot, X_s) \right)  \right] \nonumber\\
\circled{\footnotesize{1}}&= \dfrac{1}{t^2} \sum_{s=1}^t \E [\tauh_s^2 \tr ((\Sigma + \lambda I)^{-1} k(\cdot, X_s)\otimes k(\cdot, X_s) (\Sigma + \lambda I)^{-1} k(\cdot, X_s) \otimes k(\cdot, X_s)) ] \nonumber \\
&= \dfrac{1}{t^2} \sum_{s=1}^t \E \left[ \E[\tauh_s^2 \tr ((\Sigma + \lambda I)^{-1} k(\cdot, X_s)\otimes k(\cdot, X_s) (\Sigma + \lambda I)^{-1} k(\cdot, X_s)\otimes k(\cdot, X_s)) \middle| \mathcal{F}_{s-1}, X_s] \right] \nonumber \\
&= \dfrac{1}{t^2} \sum_{s=1}^t \E \left[ \E (\tauh_s^2| \mathcal{F}_{s-1}, X_s) \tr((\Sigma + \lambda I)^{-1} k(\cdot, X_s)\otimes k(\cdot, X_s) (\Sigma + \lambda I)^{-1} k(\cdot, X_s)\otimes k(\cdot, X_s)) \right],\label{eq: sum1_part1}
\end{align}
where
\begin{align*}
\E(\tauh_s^2| \mathcal{F}_{s-1}, X_s) &= \E \left[ \dfrac{I\{\ahat_s = 1\}}{\hat{\pi}_s^2} \middle| \mathcal{F}_{s-1}, X_s \right]
= \dfrac{\pih_s}{\pih_s^2} 
= \dfrac{1}{\pih_s}
 \leq \dfrac{L-1}{\epsilon_s},
\end{align*}
since $1-\epsilon_s \geq \frac{\epsilon_s}{L-1}$, 
% Since we assume that $\epsilon_s \leq \frac{L-1}{L}$, we obtain
% \begin{align*}
% \E(\tauh_s^2| \mathcal{F}_{s-1}, X_s) &= \E \left[ \dfrac{I\{\ahat_s = 1\}}{\hat{\pi}_s^2} \middle| \mathcal{F}_{s-1}, X_s \right]
% = \dfrac{\pih_s}{\pih_s^2} 
% = \dfrac{1}{\pih_s}
%  \leq \dfrac{L-1}{\epsilon_s},
% \end{align*}
resulting in
% Now \eqref{eq: sum1_part1} is bounded above by,
\begin{align}
\eqref{eq: sum1_part1} 
&\leq \dfrac{1}{t^2} \sum_{s=1}^t \E \left[ \dfrac{L-1}{\epsilon_s} \tr \left( (\Sigma + \lambda I)^{-1} k(\cdot, X_s)\otimes k(\cdot, X_s) \Slinv k(\cdot, X_s)\otimes k(\cdot, X_s) \right) \right] \nonumber \\
&\leq \dfrac{1}{t^2} \sum_{s=1}^t \dfrac{L-1}{\epsilon_s} \E \tr \left( \Slinv k(\cdot, X_s)\otimes k(\cdot, X_s)\right) \sup_{X_s} \normx{\Slinv k(\cdot, X_s)\otimes k(\cdot, X_s) }_{\infty} \nonumber \\
&\leq  \dfrac{1}{t^2} \sum_{s=1}^t \dfrac{L-1}{\epsilon_s}  \tr \left( \Slinv \Sigma \right) \normx{\Slinv}_{\infty} \sup_{X_s}\normx{k(\cdot, X_s)\otimes k(\cdot, X_s) }_{\infty} \label{Eq:tempoo} \\
& \leq \dfrac{1}{t^2} \left( \sum_{s=1}^t \dfrac{L-1}{\epsilon_s} \right) N_{\Sigma,1}(\lambda) \dfrac{\kappa}{\lambda}, \label{eq: part1_sum1_bound1}
\end{align}
where $N_{\Sigma,1}(\lambda) := \tr \left( \Slinv \Sigma \right)$.
% \textbf{Case 2: $\ell < s$.}
\begin{align}
 % \E & \left[\dfrac{1}{t^2} \sum_{\ell = 1}^{t-1} \sum_{s = \ell + 1}^t \tauh_s \tauh_\ell \tr \left( (\Sigma + \lambda I)^{-1}  k(\cdot, X_s)\otimes k(\cdot, X_s) (\Sigma + \lambda I)^{-1}   k(\cdot, X_\ell)\otimes k(\cdot, X_\ell) \right)  \right]\nonumber\\
 \circled{\footnotesize{2}}
 & = \dfrac{1}{t^2} \sum_{\ell = 1}^{t-1} \sum_{s = \ell + 1}^t \E \left[ \tauh_s \tauh_\ell \tr \left( \Slinv k(\cdot, X_s)\otimes k(\cdot, X_s) \Slinv k(\cdot, X_\ell)\otimes k(\cdot, X_\ell) \right) \right]\nonumber\\
 &  = \dfrac{1}{t^2} \sum_{\ell = 1}^{t-1} \sum_{s = \ell + 1}^t  \E \left[ \E[ \tauh_s \tauh_\ell \tr \left( \Slinv k(\cdot, X_s)\otimes k(\cdot, X_s) \Slinv \right.\right.\nonumber\\
 &\qquad\qquad\qquad\left.\left.\times k(\cdot, X_\ell)\otimes k(\cdot, X_\ell) \right) \middle| \mathcal{F}_{s-1}, X_s] \right]\nonumber\\
 &\stackrel{(\dagger)}{=} \dfrac{1}{t^2} \sum_{\ell = 1}^{t-1} \sum_{s = \ell + 1}^t  \E \left[ \tauh_\ell \E(\tauh_s| \mathcal{F}_{s-1},X_s) \tr \left((\Sigma + \lambda I)^{-1} k(\cdot, X_s)\otimes k(\cdot, X_s) \Slinv\right.\right.\nonumber\\
 &\qquad\qquad\qquad\left.\left.\times k(\cdot, X_\ell)\otimes k(\cdot, X_\ell) \right)  \right]\nonumber\\
  &= \dfrac{1}{t^2} \sum_{\ell = 1}^{t-1} \sum_{s = \ell + 1}^t  \E \left[ \tauh_\ell \tr \left((\Sigma + \lambda I)^{-1} k(\cdot, X_s)\otimes k(\cdot, X_s) \Slinv k(\cdot, X_\ell)\otimes k(\cdot, X_\ell) \right)  \right],\label{Eq:11}
 \end{align} 
where we used 
 \begin{align}
 \E(\tauh_s| \mathcal{F}_{s-1}, X_s) &= \E \left[ \dfrac{I\{\ahat_s = 1\}}{\pih_s}\Big\vert \mathcal{F}_{s-1}, X_s \right] = 1 \label{eq: Etauh_cond_1}
 \end{align}
 in $(\dagger)$.
 % \begin{align*}
 % &= \dfrac{1}{t^2} \sum_{\ell = 1}^{t-1} \sum_{s = \ell + 1}^t  \E \left[ \tauh_\ell \tr \left((\Sigma + \lambda I)^{-1} k(\cdot, X_s)\otimes k(\cdot, X_s) \Slinv k(\cdot, X_\ell)\otimes k(\cdot, X_\ell) \right)  \right].
 % \end{align*}
 Now using the law of iterated expectation,  
 % with inner expectation conditioned on all randomness until time $\ell$, 
 we obtain
 \begin{align}
\eqref{Eq:11} &= \dfrac{1}{t^2} \sum_{\ell = 1}^{t-1}  \E \left[ \sum_{s = \ell + 1}^t  \E[ \tauh_\ell \tr \left((\Sigma + \lambda I)^{-1} k(\cdot, X_s)\otimes k(\cdot, X_s) \Slinv k(\cdot, X_\ell)\otimes k(\cdot, X_\ell) \right)\vert \mathcal{F}_\ell] \right]\nonumber\\
 &= \dfrac{1}{t^2} \sum_{\ell = 1}^{t-1} \E \left[\tauh_\ell \tr \Big( \Slinv \Big(\sum_{s = \ell + 1}^t \E(k(\cdot, X_s)\otimes k(\cdot, X_s) | \mathcal{F}_\ell)\Big) \Slinv \right.\nonumber\\
 &\qquad\qquad\qquad\left.\times k(\cdot, X_\ell)\otimes k(\cdot, X_\ell)  \Big)  \right]\nonumber\\
 % \end{align}
 % \begin{align}
 &= \dfrac{1}{t^2} \sum_{\ell = 1}^{t-1} \E \left[\tauh_\ell \tr \Big( \Slinv (t- \ell)\Sigma \Slinv k(\cdot, X_\ell)\otimes k(\cdot, X_\ell)  \Big)  \right]\nonumber\\
 &= \dfrac{1}{t^2} \sum_{\ell = 1}^{t-1} (t- \ell) \E \left[\tauh_\ell \tr \Big( \Slinv \Sigma \Slinv k(\cdot, X_\ell)\otimes k(\cdot, X_\ell)  \Big)  \right]\nonumber\\
 &= \dfrac{1}{t^2} \sum_{\ell = 1}^{t-1} (t- \ell) \E \left[ \E\Big[ \tauh_\ell \tr \Big( \Slinv \Sigma \Slinv k(\cdot, X_\ell)\otimes k(\cdot, X_\ell)  \Big) \middle| \mathcal{F}_{\ell -1}, X_\ell \Big]\right]\nonumber\\
 &= \dfrac{1}{t^2} \sum_{\ell = 1}^{t-1} (t- \ell) \E \left[\E(\tauh_\ell| \mathcal{F}_{\ell-1}, X_\ell) \tr \left(\Slinv \Sigma \Slinv k(\cdot, X_\ell)\otimes k(\cdot, X_\ell)  \right) \right].\label{Eq:12}
 \end{align}
 Again, using \eqref{eq: Etauh_cond_1}, we get
 \begin{align}
\eqref{Eq:12} &= \dfrac{1}{t^2} \sum_{\ell = 1}^{t-1} (t- \ell) \E \left[\tr \left(\Slinv \Sigma \Slinv k(\cdot, X_\ell)\otimes k(\cdot, X_\ell)  \right) \right] \nonumber \\
 &=  \dfrac{1}{t^2} \sum_{\ell = 1}^{t-1} (t- \ell) \left[\tr \left(\Slinv \Sigma \Slinv \Sigma \right) \right] \nonumber \\
 &= \dfrac{1}{t^2} \sum_{\ell = 1}^{t-1} (t - \ell) N_{\Sigma, 2}(\lambda)
 = \dfrac{1}{t^2} \left[t(t-1) - \dfrac{(t-1)t}{2}\right] N_{\Sigma, 2}(\lambda) \nonumber \\
 &= \dfrac{N_{\Sigma, 2}(\lambda)}{2} - \dfrac{N_{\Sigma,2}(\lambda)}{2t}.   \label{eq: part1_sum2_bound2}
 \end{align}
%\textbf{Case 3: } $\ell > s$.
 \begin{align}
 % \E  &\left[\dfrac{1}{t^2} \sum_{\ell = 2}^t \sum_{s = 1}^{\ell - 1} \tauh_s \tauh_\ell \tr \left( (\Sigma + \lambda I)^{-1}  k(\cdot, X_s)\otimes k(\cdot, X_s) (\Sigma + \lambda I)^{-1}   k(\cdot, X_\ell)\otimes k(\cdot, X_\ell) \right)  \right] \nonumber\\
\circled{\footnotesize{3}} &= \dfrac{1}{t^2} \sum_{\ell = 2}^t \sum_{s=1}^{\ell -1} \E \left[\tauh_s \tauh_\ell \tr \left(\Slinv k(\cdot, X_s)\otimes k(\cdot, X_s) \Slinv k(\cdot, X_\ell)\otimes k(\cdot, X_\ell) \right) \right]\nonumber\\
 &= \dfrac{1}{t^2} \sum_{\ell = 2}^t \sum_{s=1}^{\ell -1} \E \left[ \E[\tauh_s \tauh_\ell \tr \left(\Slinv k(\cdot, X_s)\otimes k(\cdot, X_s) \Slinv \right.\right.\nonumber\\
 &\qquad\qquad\left.\left. \times k(\cdot, X_\ell)\otimes k(\cdot, X_\ell) \right)\middle| \mathcal{F}_{\ell-1},X_\ell]\right]\nonumber\\
 &= \dfrac{1}{t^2} \sum_{\ell = 2}^t \sum_{s=1}^{\ell -1} \E \left[ \tauh_s \E(\tauh_\ell| \mathcal{F}_{\ell - 1}, X_\ell) \tr \left(\Slinv k(\cdot, X_s)\otimes k(\cdot, X_s) \Slinv \right.\right.\nonumber\\
 &\qquad\qquad\left.\left.\times k(\cdot, X_\ell)\otimes k(\cdot, X_\ell) \right)\right]\nonumber\\
 &= \dfrac{1}{t^2} \sum_{\ell = 2}^t \sum_{s=1}^{\ell -1} \E \left[ \tauh_s \tr \left(\Slinv k(\cdot, X_s)\otimes k(\cdot, X_s) \Slinv k(\cdot, X_\ell)\otimes k(\cdot, X_\ell) \right)\right].\label{Eq:13}
 \end{align}
 Now, using law of iterated expectations, we have
 \begin{align}
 \eqref{Eq:13}&= \dfrac{1}{t^2} \sum_{\ell = 2}^t \sum_{s=1}^{\ell -1} \E \left[ \E[\tauh_s \tr \left(\Slinv k(\cdot, X_s)\otimes k(\cdot, X_s) \Slinv k(\cdot, X_\ell)\otimes k(\cdot, X_\ell) \right) \middle| \mathcal{F}_{s}]\right] \nonumber \\
  &= \dfrac{1}{t^2} \sum_{\ell = 2}^t \sum_{s=1}^{\ell -1} \E \left[ \tauh_s \tr \left(\Slinv k(\cdot, X_s)\otimes k(\cdot, X_s) \Slinv \E[k(\cdot, X_\ell)\otimes k(\cdot, X_\ell)  \middle| \mathcal{F}_{s}]\right)\right] \nonumber\\ 
 &= \dfrac{1}{t^2} \sum_{\ell = 2}^t \sum_{s=1}^{\ell -1} \E \left[ \E\left(\tauh_s \tr \left(\Slinv k(\cdot, X_s)\otimes k(\cdot, X_s) \Slinv \Sigma \right)| \mathcal{F}_{s-1}, X_s\right)\right] \nonumber \\
 &= \dfrac{1}{t^2} \sum_{\ell = 2}^t \sum_{s=1}^{\ell -1} \E \left[ \tr \left(\E(\tauh_s| \mathcal{F}_{s-1}, X_s)\Slinv k(\cdot, X_s)\otimes k(\cdot, X_s) \Slinv \Sigma \right)\right] \nonumber \\
  &= \dfrac{1}{t^2} \sum_{\ell = 2}^t \sum_{s=1}^{\ell -1} \E \left[ \tr \left(\E(\tauh_s| \mathcal{F}_{s-1}, X_s)\Slinv k(\cdot, X_s)\otimes k(\cdot, X_s) \Slinv \Sigma \right)\right] \nonumber \\
   &= \dfrac{1}{t^2} \sum_{\ell = 2}^t \sum_{s=1}^{\ell -1} \E \left[ \tr \left(\Slinv k(\cdot, X_s)\otimes k(\cdot, X_s) \Slinv \Sigma \right)\right] \nonumber \\
 &= \dfrac{1}{t^2} \sum_{\ell = 2}^t (\ell - 1) N_{\Sigma, 2}(\lambda)= \dfrac{1}{t^2} \left[\dfrac{t(t-1)}{2} \right] N_{\Sigma, 2}(\lambda)\nonumber \\
 &= \dfrac{N_{\Sigma, 2}(\lambda)}{2} - \dfrac{N_{\Sigma,2}(\lambda)}{2t}. \label{eq: part1_sum3_bound3}
 \end{align}
Putting  together \eqref{eq: part1_sum1_bound1}, \eqref{eq: part1_sum2_bound2} and \eqref{eq: part1_sum3_bound3} in \eqref{eq: part1_three_cases}, and the result in \eqref{eq: B_t_eq1}, we get,
\begin{align}
\E\normx{B_t}_{HS}^2 & \leq \dfrac{1}{t^2} \left(\sum_{s=1}^t \dfrac{L-1}{\epsilon_s}\right) N_{\Sigma, 1}(\lambda) \dfrac{\kappa}{\lambda}  + 2 \left(\dfrac{N_{\Sigma,2}(\lambda)}{2} - \dfrac{1}{2t} N_{\Sigma,2}(\lambda)\right) - N_{\Sigma,2}(\lambda)\nonumber\\
& \leq \dfrac{1}{t^2} \left(\sum_{s=1}^t \dfrac{L-1}{\epsilon_s} \right) N_{\Sigma,1}(\lambda) \dfrac{\kappa}{\lambda}. \label{eq: EBt_HS_sq}
\end{align}
Using Assumption \ref{assump: Levaluedecay_rep}, we obtain,
\begin{align*}
\E\normx{B_t}_{HS}^2 & \le A_1(\bar{C},\alpha) \dfrac{\kappa}{t^2} \left(\sum_{s=1}^t \dfrac{L-1}{\epsilon_s} \right) \lambda^{-(1 + 1/\alpha)}, \ \alpha > 1,
\end{align*}
where $A_1(\bar{C},\alpha)$ is a constant depending only on $\bar{C}$ and $\alpha$. 
% \textcolor{red}{in the above, we cannot have $\le$ but $\lesssim$ because $\eta_i\lesssim i^{-\alpha}$ and it will also involve an integral to bound $N_{\Sigma,1}(\lambda)$.}
Therefore,
\begin{align*}
P\left(\normx{B_t}_{HS} \geq \frac{1}{2}\right) & \leq 4 \E\normx{B_t}_{HS}^2
\le \dfrac{4 \kappa}{t^2} A_1(\bar{C},\alpha)\lambda^{-(1 + 1/\alpha)} \left(\sum_{s=1}^t \dfrac{L-1}{\epsilon_s} \right).
\end{align*}
Thus for $\delta > 0$, choosing:
\begin{align*}
\lambda \ge \left[\dfrac{4 \kappa A_1(\bar{C},\alpha)}{t^2 \delta} \left(\sum_{s=1}^t \dfrac{L-1}{\epsilon_s}\right)\right]^{\alpha/(1+\alpha)}, \ \alpha > 1, 
%\label{eq: lambda_cond1}
\end{align*}
yields
$$P\left(\normx{B_t}_{HS} \geq \frac{1}{2}\right) \leq \delta,$$
implying that with probability at least $1-\delta$, \begin{equation}\mathcal{S}_1\le \sqrt{2}.\label{Eq:s1}\end{equation}

We now bound $\mathcal{S}_2$ as
\begin{align}
\mathcal{S}_2 &= \left\Vert\left({\Sigma} + \lambda I \right)^{-1/2} \left[ -\lambda f_1 + \dfrac{1}{t} \sum_{s=1}^t \dfrac{I\{\hat{a}_s = 1\}}{\hat{\pi}_s} k(\cdot, X_s) e_s \right]\right\Vert_\mathcal{H}\nonumber\\
&\le \left\| (\Sigma + \lambda I)^{-1/2} \dfrac{1}{t} \sum_{s=1}^t \dfrac{I\{\ahat_s = 1\}}{\pih_s} k(\cdot, X_s) e_s \right\|_{\mathcal{H}} + \lambda\normx{(\Sigma + \lambda I)^{-1/2} f_1}_{\mathcal{H}}. \label{eq: esterror_decomp2}
\end{align}
For the second term in \eqref{eq: esterror_decomp2},  using Assumption \ref{assump: Lsmoothness_cond_rep}, we have
\begin{align}
\left\| (\Sigma + \lambda I)^{-1/2} f_1 \right\|_\mathcal{H} &\leq \left\|(\Sigma + \lambda I)^{-1/2} \Sigma^{\gamma_1} h \right\|_\mathcal{H}
 \leq \left\| (\Sigma + \lambda I)^{-1/2} \Sigma^{\gamma_1}\right\|_{\infty} \normx{\Sigma^{-\gamma_1} f_1}_{\mathcal{H}}\nonumber\\
&\leq \sup_i \dfrac{\eta_i^{\gamma_1}}{(\eta_i + \lambda)^{1/2}}  \normx{\Sigma^{-\gamma_1} f_1}_{\mathcal{H}}
 \leq \sup_{x \geq 0} \left[\dfrac{x^{2\gamma_1}}{x + \lambda}  \right]^{1/2} \normx{\Sigma^{-\gamma_1} f_1}_{\mathcal{H}}\nonumber\\
 &\le \lambda^{\gamma_1-\frac{1}{2}} \normx{\Sigma^{-\gamma_1} f_1}_{\mathcal{H}},\label{Eq:s22} 
\end{align}
where the last inequality follows by noting that
for $0 < \gamma_1 \leq 1/2$,
\begin{align*}
\left(\sup_{x \geq 0} \dfrac{x^{2\gamma_1}}{x + \lambda} \right)^{1/2} &= \left(\sup_{x \geq 0} \left(\dfrac{x}{x + \lambda} \right)^{2 \gamma_1} \dfrac{1}{(x + \lambda)^{1-2\gamma_1}}  \right)^{1/2} \leq  \lambda^{\gamma_1 - \frac{1}{2}}.
\end{align*}
For any $\xi>0$, applying Chebyshev's inequality to the first term in \eqref{eq: esterror_decomp2} yields
\begin{align}
P &\left( \left\| (\Sigma + \lambda I)^{-1/2} \dfrac{1}{t} \sum_{s=1}^t \dfrac{I\{\ahat_s = 1\}}{\pih_s} k(\cdot, X_s) e_s  \right\|_\mathcal{H} \geq \xi \right) \nonumber\\
&\leq \dfrac{1}{\xi^2}\E \left\| (\Sigma + \lambda I)^{-1/2} \dfrac{1}{t} \sum_{s=1}^t \dfrac{I\{\ahat_s = 1\}}{\pih_s} k(\cdot, X_s) e_s  \right\|_\mathcal{H}^2 \nonumber\\
&= \dfrac{1}{\xi^2}\E \left\langle 	(\Sigma + \lambda I)^{-1/2} \dfrac{1}{t} \sum_{s=1}^t \dfrac{I\{\ahat_s = 1\}}{\pih_s} k(\cdot, X_s) e_s, (\Sigma + \lambda I)^{-1/2} \dfrac{1}{t} \sum_{\ell=1}^t \dfrac{I\{\ahat_\ell = 1\}}{\pih_\ell} k(\cdot, X_\ell) e_\ell\right\rangle_\mathcal{H} \nonumber\\
&= \dfrac{1}{\xi^2}\E  \left\langle 	(\Sigma + \lambda I)^{-1}, \dfrac{1}{t^2} \sum_{s=1}^t \sum_{\ell = 1}^t \dfrac{I\{\ahat_s = 1\}}{\pih_s}  \dfrac{I\{\ahat_\ell = 1\}}{\pih_\ell}e_s e_\ell k(\cdot, X_s) \otimes k(\cdot, X_\ell) \right\rangle_{HS}\nonumber\\
&=   \dfrac{1}{\xi^2}\left\langle 	(\Sigma + \lambda I)^{-1}, \E \left(  \dfrac{1}{t^2} \sum_{s=1}^t \sum_{\ell = 1}^t \dfrac{I\{\ahat_s = 1\}}{\pih_s}  \dfrac{I\{\ahat_\ell = 1\}}{\pih_\ell}e_s e_\ell k(\cdot, X_s) \otimes k(\cdot, X_\ell) \right)\right\rangle_{HS}. \label{eq: part2_term1_bound}
\end{align}
In the following, we simplify the expectation term in \eqref{eq: part2_term1_bound} by considering three cases for the double sum: (1) $\ell = s$, (2) $\ell > s$ and (3) $\ell < s$. Recall $\tauh_s = I\{\ahat_s = 1\}/\pih_s$, where $\pih_s = P(\ahat_s = 1| \mathcal{F}_{s-1}, X_s)$. Consider
\begin{align*}
\E &\left(  \dfrac{1}{t^2} \sum_{s=1}^t \sum_{\ell = 1}^t \dfrac{I\{\ahat_s = 1\}}{\pih_s}  \dfrac{I\{\ahat_\ell = 1\}}{\pih_\ell}e_s e_\ell k(\cdot, X_s) \otimes k(\cdot, X_\ell) \right) \\
&= \E \left(  \dfrac{1}{t^2} \sum_{s=1}^t  \dfrac{I\{\ahat_s = 1\}}{\pih_s^2}  e_s^2 k(\cdot, X_s) \otimes k(\cdot, X_s) \right)
 + \E \left(  \dfrac{1}{t^2} \sum_{\ell=1}^{t-1} \sum_{s = \ell + 1}^t \tauh_s \tauh_\ell e_s e_\ell k(\cdot, X_s) \otimes k(\cdot, X_\ell) \right)\\
& \quad \quad \quad + \E \left(\dfrac{1}{t^2} \sum_{\ell=2}^{t} \sum_{s = 1}^{\ell - 1} \tauh_s \tauh_\ell e_s e_\ell k(\cdot, X_s) \otimes k(\cdot, X_\ell) \right)\\
&=\circled{\footnotesize{4}}+\circled{\footnotesize{5}}+\circled{\footnotesize{6}},
\end{align*}
where 
% \textbf{Case 1: $s = \ell$: }\\
\begin{align}
% \E &\left(  \dfrac{1}{t^2} \sum_{s=1}^t  \dfrac{I\{\ahat_s = 1\}}{\pih_s^2}  e_s^2 k(\cdot, X_s) \otimes k(\cdot, X_s) \right)\nonumber\\
\circled{\footnotesize{4}}& = \E \left(  \dfrac{1}{t^2} \sum_{s=1}^t  \dfrac{I\{\ahat_s = 1\}}{\pih_s^2}  e_s^2 k(\cdot, X_s) \otimes k(\cdot, X_s) \right)\nonumber\\
&=\dfrac{1}{t^2} \sum_{s=1}^t \E \left[ \E \left(\dfrac{I\{\ahat_s = 1\}}{\pih_s^2}  e_s^2 k(\cdot, X_s) \otimes k(\cdot, X_s) \middle| \mathcal{F}_{s-1}, X_s, \hat{a}_s\right) \right]\nonumber
  \end{align}
 \begin{align}
 & = \dfrac{1}{t^2} \sum_{s=1}^t \E \left[ \dfrac{I\{\ahat_s = 1\}}{\pih_s^2}  k(\cdot, X_s) \otimes k(\cdot, X_s) \E \left( e_s^2  \middle| \mathcal{F}_{s-1}, X_s, \hat{a}_s \right) \right]\nonumber\\
 %\label{eq: three_conditionals_cheby}\\
 & \preceq \dfrac{\sigma^2}{t^2} \sum_{s=1}^t \E \left[ \dfrac{I\{\ahat_s = 1\}}{\pih_s^2}  k(\cdot, X_s) \otimes k(\cdot, X_s) \right] \nonumber \\
&= \dfrac{\sigma^2}{t^2} \sum_{s=1}^t \E \left[ \E \Bigg[ \dfrac{I\{\ahat_s = 1\}}{\pih_s^2}  k(\cdot, X_s) \otimes k(\cdot, X_s)\middle| \mathcal{F}_{s-1}, X_s \Bigg]\right] \nonumber\\
&= \dfrac{\sigma^2}{t^2} \sum_{s=1}^t \E\left[ \dfrac{1}{\pih_s} k(\cdot, X_s) \otimes k(\cdot, X_s) \right]\preceq \dfrac{\sigma^2}{t^2} \left(\sum_{s=1}^t \dfrac{L-1}{\epsilon_s} \right) \Sigma, \label{eq: part2_term1_case1_bound}
\end{align}
where the above inequality follows from  \ref{assump:independenterrors}, and the fact that $\pih_s \geq \epsilon_s/(L-1)$.
$D_{s\ell} := \{\hat{a}_s = \hat{a}_\ell = 1\}$. Then 
the complement of this event consists of situations when the two arms are not the same or either (or both) are not arm $1$, i.e., one or both of $\hat{\tau_s}$ and $\hat{\tau_\ell}$ will be zero. Therefore,
% . Note that, if $\hat{a}_s \neq \hat{a}_\ell$, then one or both of $ \hat{\tau_s} $ and $\hat{\tau_\ell}$ will be zero.
\begin{align}
\circled{\footnotesize{5}}&=\E \left(  \dfrac{1}{t^2} \sum_{\ell=1}^{t-1} \sum_{s = \ell + 1}^t \tauh_s \tauh_\ell e_s e_\ell k(\cdot, X_s) \otimes k(\cdot, X_\ell) \right)\nonumber\\
&= \dfrac{1}{t^2} \sum_{\ell = 1}^{t-1} \sum_{s = \ell + 1}^t \E \left[ E\left( \tauh_s \tauh_\ell e_s e_\ell k(\cdot, X_s) \otimes k(\cdot, X_\ell)\middle|\mathcal{F}_{s-1}, X_s, D_{s\ell} \right) \right]\nonumber\\
& = \dfrac{1}{t^2} \sum_{\ell = 1}^{t-1} \sum_{s = \ell + 1}^t \E \left[ \dfrac{1}{\pih_s \hat{\pi}_{\ell}} \E (e_s e_\ell|\mathcal{F}_{s-1}, X_s,  D_{s\ell}) k(\cdot, X_s) \otimes k(\cdot, X_\ell)  \right] \nonumber\\
&= \dfrac{1}{t^2} \sum_{\ell = 1}^{t-1} \sum_{s = \ell + 1}^t \E \left[\dfrac{1}{\pih_s \hat{\pi}_{\ell}}\E (e_s |\mathcal{F}_{s-1}, X_s,  \hat{a}_s = 1) \E (e_\ell |\mathcal{F}_{s-1}, X_s,  \hat{a}_\ell = 1) k(\cdot, X_s) \otimes k(\cdot, X_\ell)  \right] \label{eq: part2_A2_etindxt}\\
&= 0, \label{eq: part2_term1_case2_bound}
\end{align}
where \eqref{eq: part2_A2_etindxt} follows from \ref{assump: errorsindependentX}, i.e., errors and covariates at time $t$ are independent for a given arm. 
% \textbf{Case 3: $\ell > s$,}
Similar to $\circled{\footnotesize{5}}$, we obtain
\begin{align}
\circled{\footnotesize{6}} &=\E \left(\dfrac{1}{t^2} \sum_{\ell=2}^{t} \sum_{s = 1}^{\ell - 1} \tauh_s \tauh_\ell e_s e_\ell k(\cdot, X_s) \otimes k(\cdot, X_\ell) \right)\nonumber\\
&= \dfrac{1}{t^2} \sum_{\ell=2}^{t} \sum_{s = 1}^{\ell - 1} \E \left[ E\left( \tauh_s \tauh_\ell e_s e_\ell k(\cdot, X_s) \otimes k(\cdot, X_\ell)\middle|\mathcal{F}_{\ell-1}, X_\ell, D_{s\ell} \right) \right]\nonumber\\
& = \dfrac{1}{t^2} \sum_{\ell=2}^{t} \sum_{s = 1}^{\ell - 1} \E \left[ \dfrac{1}{\pih_s \hat{\pi}_{\ell}} \E (e_s e_\ell|\mathcal{F}_{\ell-1}, X_\ell,  D_{s\ell}) k(\cdot, X_s) \otimes k(\cdot, X_\ell)  \right] \nonumber\\
&=  \dfrac{1}{t^2}\sum_{\ell=2}^{t} \sum_{s = 1}^{\ell - 1} \E \left[\dfrac{1}{\pih_s \hat{\pi}_{\ell}}\E (e_s |\mathcal{F}_{\ell-1}, X_\ell,  \hat{a}_s = 1) \E (e_\ell |\mathcal{F}_{\ell-1}, X_\ell,  \hat{a}_\ell = 1) k(\cdot, X_s) \otimes k(\cdot, X_\ell)  \right]\nonumber\\
&= 0. \label{eq: part2_term1_case3_bound}
\end{align}
% \begin{align*}
%  &= \dfrac{1}{t^2} \sum_{\ell = 2}^t \sum_{s = 1}^{\ell - 1} \E \left[\tauh_s \tauh_\ell e_s e_\ell k(\cdot, X_s) \otimes k(\cdot, X_\ell)\right]\\
%  &= \dfrac{1}{t^2} \sum_{\ell = 2}^t \sum_{s = 1}^{\ell - 1} \E \left[ \E \left(\tauh_s \tauh_\ell e_s e_\ell k(\cdot, X_s) \otimes k(\cdot, X_\ell)| \mathcal{F}_{\ell - 1}, X_\ell, \hat{a}_\ell \right) \right].
% \end{align*}
% Again, using the same logic as \eqref{eq: case2_threeconditionals_cheby}, we get:
% \begin{align}
% &= \dfrac{1}{t^2} \sum_{\ell = 2}^t \sum_{s = 1}^{\ell - 1} \E \left[ \tauh_s \tauh_\ell \E(e_\ell|\mathcal{F}_{\ell - 1}, X_\ell, \hat{a}_\ell) \E(e_s|\mathcal{F}_{\ell - 1}, X_\ell, \hat{a}_\ell)  k(\cdot, X_s) \otimes k(\cdot, X_\ell) \right]   \nonumber \\
% &= 0. \label{eq: part2_term1_case3_bound}
% \end{align}
Combining \eqref{eq: part2_term1_case1_bound}, \eqref{eq: part2_term1_case2_bound} and \eqref{eq: part2_term1_case3_bound} in \eqref{eq: part2_term1_bound}, we obtain
\begin{align}
P \left( \left\| (\Sigma + \lambda I)^{-1/2} \dfrac{1}{t} \sum_{s=1}^t \dfrac{I\{\ahat_s = 1\}}{\pih_s} k(\cdot, X_s) e_s  \right\|_\mathcal{H} \geq \xi \right) %\nonumber\\
 \leq \dfrac{1}{\xi^2}\left\langle 	(\Sigma + \lambda I)^{-1}, \dfrac{\sigma^2}{t^2} \left(\sum_{s=1}^t \dfrac{L-1}{\epsilon_s} \right) \Sigma \right\rangle_{HS} \nonumber
   \end{align}
 \begin{align}
& = \dfrac{\sigma^2}{t^2 \xi^2} \left(\sum_{s=1}^t \dfrac{L-1}{\epsilon_s}\right) \tr \left[(\Sigma + \lambda I)^{-1} \Sigma\right]%\nonumber\\
= \dfrac{\sigma^2}{t^2 \xi^2} \left(\sum_{s=1}^t \dfrac{L-1}{\epsilon_s}\right) N_{\Sigma,1}(\lambda)\label{Eq:tempu}\\
&\le \dfrac{\sigma^2 A_1(\bar{C},\alpha)}{t^2 \xi^2} \left( \sum_{s=1}^t \dfrac{L-1}{\epsilon_s} \right) \lambda^{-1/\alpha}, \label{eq: term1_final_bound_expectation}
\end{align}
where we used Assumption \ref{assump: Levaluedecay_rep} in the last inequality. Combining \eqref{Eq:s22} and \eqref{eq: term1_final_bound_expectation} in \eqref{eq: esterror_decomp2}, and choosing \begin{align*}
\xi = \left[  \dfrac{\sigma^2 A_1(\bar{C},\alpha)}{\delta t^2} \left( \sum_{s=1}^t \dfrac{L-1}{\epsilon_s} \right) \lambda^{- 1/\alpha}\right]^{1/2},% \label{eq: part2_term1_final_bound1}
\end{align*}
yields that with probability at least $1-\delta$,
\begin{align}\mathcal{S}_2\le \left[  \dfrac{\sigma^2 A_1(\bar{C},\alpha)}{\delta t^2} \left( \sum_{s=1}^t \dfrac{L-1}{\epsilon_s} \right) \lambda^{- 1/\alpha}\right]^{1/2}+\lambda^{\gamma_1+\frac{1}{2}}\Vert \Sigma^{-\gamma_1}f_1\Vert_\mathcal{H}.\label{Eq:s2}\end{align}
Using \eqref{Eq:s1} and \eqref{Eq:s2} in \eqref{eq: norm_ineq_estimation_error} yields that with probability at least $1-2\delta$, we have
\begin{align}
 \Vert \hat{f}_{1,t}-f_1\Vert_\mathcal{H}&\le \left[  \dfrac{2\sigma^2 A_1(\bar{C},\alpha)}{\delta t^2} \left( \sum_{s=1}^t \dfrac{L-1}{\epsilon_s} \right) \lambda^{- (1+1/\alpha)}\right]^{1/2}+\sqrt{2}\lambda^{\gamma_1}\Vert \Sigma^{-\gamma_1}f_1\Vert_\mathcal{H}\nonumber\\
% & \textcolor{red}{\le \left[  \dfrac{2\sigma^2 A_1(\bar{C},\alpha)}{\delta t^2} \left( \sum_{s=1}^t \dfrac{L-1}{\epsilon_s} \right) \lambda^{- (1+1/\alpha)}\right]^{1/2}+\sqrt{\frac{2}{\delta}}\lambda^{\gamma_1}\Vert \Sigma^{-\gamma_1}f_1\Vert_\mathcal{H}}\nonumber\\
 &\le \sqrt{2}\max\{C_0, C_1\} \left[ \left(\dfrac{1}{{\delta} t^2} \left(\sum_{s=1}^t \dfrac{1}{\epsilon_s}\right) \lambda^{-(1+1/\alpha)} \right)^{1/2} + \lambda^{\gamma_1}\right],   \label{eq: balancing_term_thm1}
\end{align}
where
 $C_0 = \sqrt{\sigma^2 (L-1) A_1(\bar{C},\alpha)}$ and $C_1 = \normx{\Sigma^{-\gamma_1} f_1}_\mathcal{H}$. The result follows by choosing 
%  then we have
% \begin{align}
% \normx{\hat{f}_{1,t} - f_1}_\mathcal{H} &\leq \max\{C_0, C_1\} \left[ \left(\dfrac{1}{\delta t^2} \left(\sum_{s=1}^t \dfrac{1}{\epsilon_s}\right) \lambda^{-(1+1/\alpha)} \right)^{1/2} + \lambda^{\gamma_1}\right].   \label{eq: balancing_term_thm1}
% \end{align}
% Note that, for $\alpha > 1$, and $0<\gamma_1 \leq 1/2$, $\lambda^{-(1+1/\alpha)}$ and $\lambda^{\gamma_1}$ are decreasing and increasing functions of $\lambda$, respectively. Therefore we can balance the two terms to obtain
\begin{align}
{\lambda=\lambda_{1,t} := \left[\dfrac{1}{\delta t^2} \left(\sum_{s=1}^t \dfrac{1}{\epsilon_s}\right) \right]^{\alpha/(2\gamma_1 \alpha + \alpha + 1)}} \nonumber
% \textcolor{red}{\lambda=\lambda_{1,t} := \left[\dfrac{1}{ t^2} \left(\sum_{s=1}^t \dfrac{1}{\epsilon_s}\right) \right]^{\alpha/(2\gamma_1 \alpha + \alpha + 1)} \nonumber}
% %\label{eq: lambda_thm1_infinite}
\end{align}
in \eqref{eq: balancing_term_thm1}.
%Note that, the requirement of $\frac{1}{\delta t^2}\sum_{s=1}^t \frac{1}{\epsilon_s} \le \tilde{C}$ for all $t$ ensures that  $\lambda$ satisfies \eqref{eq: lambda_cond1}.
% With this choice for $\lambda$ and with probability at least $1-\delta$, we get an upper bound on the rate of convergence for the estimation error,
% \begin{align}
% \normx{\hat{f}_{1,t} - f_1}_\mathcal{H} &\leq  \max\{C_0, C_1\}\lambda_{1,t}^{\gamma_1}  = \max\{C_0, C_1\}\left(\dfrac{1}{\delta t^2} \left( \sum_{s=1}^t \frac{1}{\epsilon_s}\right)  \right)^{\gamma_1 \alpha/ (2\gamma_1 \alpha + \alpha + 1)}, \label{eq: est_bound_thm1_infinite}
% \end{align}
% where $C_0 = \sqrt{(2\sigma^2 (L-1))}$ and $C_1 = \normx{\Sigma^{-\gamma_1} f_1}_\mathcal{H}$.
% \textcolor{red}{throughout the proof of theorem 1, there were many places like the above where there is no proper continuity of the equation. The above representation looks very informal. Please look through the proof of Theorem 1 and see how I have modified it. Pleaes incorporate such changes for the rest of the proofs.}
Also, the same proof works for all arms $i = 1,\hdots,L$, by defining $\hat{\pi}_s = P(\hat{a}_s = i| \mathcal{F}_{s-1}, X_s)$.

\subsection{Proof of Theorem~\ref{thm: theorem_finiteRKHS}}
\label{proof: proof2}
% \begin{proof}[\textbf{}]
% In this section, we highlight the changes we make in the proof of Theorem~\ref{thm: theorem1}, when 
% % estimation error bound analysis for the kernel $\epsilon$-greedy algorithm when 
% $\mathcal{H}$ is restricted to finite-dimensions.
%  Note that 
Since most steps in the proof of Theorem \ref{thm: theorem_finiteRKHS} follow that of the proof of Theorem~\ref{thm: theorem1}, we only highlight the differences when $\mathcal{H}$ is finite-dimensional. Again, without loss of generality, we assume that $i = 1$ and the proof for other arms follows similarly.  Recall, $\hat{\pi}_s := P(\hat{a}_s = 1| \mathcal{F}_{s-1}, X_s)$. Note that \ref{assump: min_evalue} implies $\mathcal{H}$ is finite-dimensional. Define $d:=\text{dim}(\mathcal{H})$. Then
\begin{align}
N_{\Sigma, 1}(\lambda) &= \text{Tr}((\Sigma + \lambda I)^{-1} \Sigma) = \sup_i \dfrac{\eta_i(\Sigma) d}{(\eta_i(\Sigma) + \lambda)} \leq   d,\,\,\,\text{and}\,\,\, \normx{(\Sigma + \lambda I)^{-1}}_{\infty} \leq  \dfrac{1}{\eta}.\label{eq: effectivedim_boundFinite}
\end{align}
% Likewise, $\normx{(\Sigma + \lambda I)^{-1}}_{\infty} \leq  \dfrac{1}{\eta}$. 
Therefore \eqref{eq: norm_ineq_estimation_error} modifies to
%%%%%%%%%%%%%%%%%%%%%%%%%%%%%%%%%%%%%%%%%%%%%%%%%%%%%%%%%%%%%%%%%%%%%%%%%%%%%%%%%%%%%%%%%%%%%%%%
% We modify the r.h.s. in \eqref{eq: norm_ineq_estimation_error} as follows.
\begin{align}
\normx{\fh_{1,t} - f_1}_{\mathcal{H}} 
& \leq \normx{(\Sigmah_{1,t} + \lambda I)^{-1/2}}_{\infty} \normx{ (\hat{\Sigma}_{1,t} + \lambda I)^{-1/2} (\Sigma + \lambda I)^{1/2}}_{\infty} \nonumber\\
& \quad \quad \times \norm{ \left({\Sigma} + \lambda I \right)^{-1/2}  \left[ -\lambda f_1 + \dfrac{1}{t} \sum_{s=1}^t \dfrac{I\{\hat{a}_s = 1\}}{\hat{\pi}_s} k(\cdot, X_s) \epsilon_s \right]}_{\mathcal{H}}\nonumber\\
& \leq \normx{(\Sigma + \lambda I)^{-1/2}}_{\infty} \normx{ (\hat{\Sigma}_{1,t} + \lambda I)^{-1/2} (\Sigma + \lambda I)^{1/2}}_{\infty}^2 \nonumber\\
& \quad \quad \times \norm{ \left({\Sigma} + \lambda I \right)^{-1/2}  \left[ -\lambda f_1 + \dfrac{1}{t} \sum_{s=1}^t \dfrac{I\{\hat{a}_s = 1\}}{\hat{\pi}_s} k(\cdot, X_s) \epsilon_s \right]}_{\mathcal{H}} \nonumber \\
& \leq \frac{\mathcal{S}^2_1\mathcal{S}_2}{\sqrt{\eta}}. 
% \normx{ (\hat{\Sigma}_{1,t} + \lambda I)^{-1/2} (\Sigma + \lambda I)^{1/2}}_{\infty}^2 \nonumber\\
% & \quad \quad \times \norm{ \left({\Sigma} + \lambda I \right)^{-1/2}  \left[ -\lambda f_1 + \dfrac{1}{t} \sum_{s=1}^t \dfrac{I\{\hat{a}_s = 1\}}{\hat{\pi}_s} k(\cdot, X_s) \epsilon_s \right]}_{\mathcal{H}}.
\label{eqF: estimation_error_decomp}
\end{align}
We now bound $\mathcal{S}_1$ and $\mathcal{S}_2$. Bounding $\mathcal{S}_1$ proceeds exactly as in the proof of Theorem~\ref{thm: theorem1} until \eqref{Eq:tempoo} yielding
\begin{align}
\circled{\footnotesize{1}}&\le  \dfrac{1}{t^2} \sum_{s=1}^t \dfrac{L-1}{\epsilon_s}  \tr \left( \Slinv \Sigma \right) \normx{\Slinv}_{\infty} \sup_{X_s}\normx{k(\cdot, X_s)\otimes k(\cdot, X_s) }_{\infty}\nonumber\\
&\le \dfrac{1}{t^2} \left(\sum_{s=1}^t \dfrac{L-1}{\epsilon_s}  \right) \frac{d \kappa}{\eta},\label{eqF: part1_sum1_bound11}
\end{align}
% \begin{equation}
% \normx{\fh_{1,t} - f_1}_{\mathcal{H}} \le \frac{\mathcal{S}_1\mathcal{S}_2}{\sqrt{\eta}}.
% \label{eqF: estimation_error_decomp}    
% \end{equation}
% After this, all the steps follow from the proof of Theorem \ref{thm: theorem1} in Section \ref{proof: proof1} until \eqref{Eq:tempoo}, which becomes,
% \textbf{Case 1: $\ell = s$}
% \begin{align}
% \E & \left[\dfrac{1}{t^2} \sum_{s=1}^t \tauh_s^2 \tr \left( (\Sigma + \lambda I)^{-1}  k(\cdot, X_s)\otimes k(\cdot, X_s) (\Sigma + \lambda I)^{-1}   k(\cdot, X_s)\otimes k(\cdot, X_s) \right)  \right] \nonumber\\
% & \leq \dfrac{1}{t^2} \left( \sum_{s=1}^t \dfrac{L-1}{\epsilon_s} \right) N_{\Sigma,1}(\lambda) \dfrac{\kappa}{\eta} \leq \dfrac{1}{t^2} \left(\sum_{s=1}^t \dfrac{L-1}{\epsilon_s}  \right) \frac{d \kappa}{\eta},\label{eqF: part1_sum1_bound1}
% \end{align}
where we use \eqref{eq: effectivedim_boundFinite} in the last line of the above inequality. Putting together \eqref{eqF: part1_sum1_bound11}, \eqref{eq: part1_sum2_bound2} and \eqref{eq: part1_sum3_bound3} in \eqref{eq: part1_three_cases}, and the result in \eqref{eq: B_t_eq1}, we get the following analog of \eqref{eq: EBt_HS_sq}:
\begin{align*}
\E\normx{B_t}_{HS}^2 \leq \dfrac{(L-1)d \kappa}{\eta} \left(\frac{1}{t^2}\sum_{s=1}^t \dfrac{1}{\epsilon_s}\right). 
\end{align*}
Therefore, using Chebyshev's inequality,
\begin{align*}
P\left( \normx{B_t}_{HS} \geq \frac{1}{2} \right) &\leq 4\E\normx{B_t}_{HS}^2
\leq \dfrac{4(L-1)d\kappa}{\eta}\left(\frac{1}{t^2}\sum_{s=1}^t \dfrac{1}{\epsilon_s}\right)\le\delta.
\end{align*}
Thus for $\delta > 0$, 
% if $(\epsilon_s)_s$ is such that
% \begin{align}
% \frac{1}{t^2}\sum_{s=1}^t \dfrac{1}{\epsilon_s} \leq \dfrac{\delta \eta}{4(L-1)\kappa d}, \label{eqF: eps_delta_cond}
% \end{align}
with probability at least $1- \delta$, we obtain
\begin{align*}
\mathcal{S}^2_1 \leq 2. 
\end{align*}
To bound $\mathcal{S}_2$, we bound the second term in \eqref{eq: esterror_decomp2} as 
\begin{align}
\lambda \normx{(\Sigma + \lambda I)^{-1/2} f_1}_{\mathcal{H}} \le \lambda\normx{(\Sigma + \lambda I)^{-1/2} }_{\infty} \normx{ f_1}_{\mathcal{H}}
 \leq \dfrac{\lambda}{\sqrt{\eta}} \normx{f_1}_\mathcal{H}.\label{eq:s2-2}
\end{align}
For bounding the first term of \eqref{eq: esterror_decomp2}, we follow the same steps as in the proof of Theorem~\ref{thm: theorem1} until \eqref{Eq:tempu}. By using \eqref{eq: effectivedim_boundFinite} in \eqref{Eq:tempu}, we obtain
\begin{align*}
P\left( \left\| (\Sigma + \lambda I)^{-1/2} \dfrac{1}{t} \sum_{s=1}^t \dfrac{I\{\ahat_s = 1\}}{\pih_s} k(\cdot, X_s) e_s  \right\|_\mathcal{H} \geq \xi \right) \leq \dfrac{(L-1) \sigma^2 d}{\xi^2} \left(\frac{1}{t^2}\sum_{s=1}^t \dfrac{1}{\epsilon_s}\right),
\end{align*}
i.e., with probability at least $1-\delta$,
\begin{align}
\left\| (\Sigma + \lambda I)^{-1/2} \dfrac{1}{t} \sum_{s=1}^t \dfrac{I\{\ahat_s = 1\}}{\pih_s} k(\cdot, X_s) e_s  \right\|_\mathcal{H} \le \sqrt{\dfrac{(L-1) \sigma^2 d}{\delta}} \left(\frac{1}{t^2}\sum_{s=1}^t \dfrac{1}{\epsilon_s}\right)^{\frac{1}{2}}.\label{Eq:s2-1}
\end{align}
Combining \eqref{eq:s2-2} and \eqref{Eq:s2-1} in \eqref{eq: esterror_decomp2}, we obtain that with probability at least $1-\delta$,
$$\mathcal{S}_2\le \sqrt{\dfrac{(L-1) \sigma^2 d}{\delta}} \left(\frac{1}{t^2}\sum_{s=1}^t \dfrac{1}{\epsilon_s}\right)^{\frac{1}{2}}+\frac{\lambda}{\sqrt{\eta}}\Vert f_1\Vert_\mathcal{H}.$$
Using these bounds on $\mathcal{S}_1$ and $\mathcal{S}_2$ in \eqref{eqF: estimation_error_decomp}, we obtain that with probability at least $1-2\delta$,
\begin{align}
\| \hat{f}_{1,t} - f_1 \|_\mathcal{H} & \leq \left[ \dfrac{4(L-1) \sigma^2d}{\delta \eta} \left( \frac{1}{t^2}\sum_{s=1}^t \dfrac{1}{\epsilon_s} \right) \right]^{1/2} + \dfrac{2\lambda}{\eta}\normx{f_1}_\mathcal{H} \nonumber\\
&\le \frac{2}{{\sqrt{\delta}}}\max\{\tilde{C}_0, \tilde{C}_1\} \left[\left(\dfrac{1}{ t^2} \left( \sum_{s=1}^t \dfrac{1}{\epsilon_s} \right) \right)^{1/2} + \lambda\right],\nonumber
\end{align}
where
$\tilde{C}_0 := \sqrt{(L-1)\sigma^2 d/\eta }$ and $\tilde{C}_1:= \normx{f_1}_\mathcal{H}/\eta$. The result, therefore, follows by choosing $$\lambda=\lambda_t := \left[\dfrac{1}{t^2} \left(\sum_{s=1}^t \dfrac{1}{\epsilon_s}\right) \right]^{1/2}.$$

Next, we derive the bound for the estimation error in expectation. From the above, note that, for $i = 1,\hdots, L$ and any $\delta > 0$,

\begin{align*}
    P(\normx{\hat{f}_{i,t} - f_i}_\mathcal{H} \le\theta_t) \geq 1 - 2\delta,
\end{align*}
where,
\begin{align*}
    \theta_t = \frac{2}{\sqrt{\delta}}\max\{\tilde{C}_0, \tilde{C}_i\} \left[\dfrac{1}{t^2} \left(\sum_{s=1}^t \dfrac{1}{\epsilon_s} \right)\right]^{1/2}.
\end{align*}
\begin{align*}
    P(\normx{\hat{f}_{i,t} - f_i}_\mathcal{H} > \theta)  
   &\le \min\left\{1,\left(\dfrac{4 (\max\{\tilde{C}_0, \tilde{C}_i\})^2}{\theta^2} \right) \left[\dfrac{1}{t^2} \left(\sum_{s=1}^t \dfrac{1}{\epsilon_s} \right) \right]\right\}.
\end{align*}
The expectation result therefore follows by using  $\E(\normx{\hat{f}_{i,t} - f_i}^{1+\zeta}_\mathcal{H}) = \int_0^\infty P(\normx{\hat{f}_{i,t} - f_i}_\mathcal{H} > \theta)\theta^\zeta d\theta$.
% \begin{align}
%  \normx{\hat{f}_{1,t} - f_1}_\mathcal{H} & \leq \max\{\tilde{C}_0, \tilde{C}_1\} \left[\left(\dfrac{1}{\delta t^2} \left( \sum_{s=1}^t \dfrac{1}{\epsilon_s} \right) \right)^{1/2} + \lambda\right].   \label{eq: estimation_error_breakup}
% \end{align}
% Balancing the two terms in the r.h.s. of \eqref{eq: estimation_error_breakup}, we obtain
% \begin{align}
% \lambda_t = \left[\dfrac{1}{\delta t^2} \left(\sum_{s=1}^t \dfrac{1}{\epsilon_s}\right) \right]^{1/2}. \label{eq: lambda_finitedim}
% \end{align}
% % Note that, if we choose $1/\epsilon_s = o(s)$, this choice of $\lambda$ satisfies \eqref{eqF: lambda_cond1}.
% With the choice for $\lambda_t$ as in \eqref{eq: lambda_finitedim}, we obtain that with probability at least $1-\delta$,
% \begin{align*}
% \| \hat{f}_{1,t} - f_1 \|_\mathcal{H} \leq  \max\{\tilde{C}_0, \tilde{C}_1\} \left(\dfrac{1}{\delta t^2} \left( \sum_{s=1}^t \frac{1}{\epsilon_s}\right)  \right)^{1/2}.
% \end{align*}
Note, that same proof works for all arms $i = 1,\hdots,L$, by defining $\hat{\pi}_s = P(\hat{a}_s = i| \mathcal{F}_{s-1}, X_s)$. 

\subsection{Proof of Theorem~\ref{thm: regret_infinite_case_pre}}
\label{proof: proof_regret_infinite}
The randomization error in the regret decomposition in \eqref{eq: regret_breakup} can be bounded as
\begin{align}
&\sum_{t=1}^T \left|f_{A_t}(X_t) - f_{\ahat_t}(X_t)\right| = \sum_{t=1}^T \left|\left\langle f_{A_t} - f_{\ahat_t}, k(\cdot, X_t)	\right\rangle_\mathcal{H}\right| \nonumber\\
&= \sum_{t=1}^T I\{\ahat_t \neq A_t \} |\left\langle  f_{A_t} - f_{\ahat_t} , k(\cdot, X_t)	\right\rangle_\mathcal{H}|
 \leq \kappa \sum_{t=1}^T I\{\ahat_t \neq A_t \} \normx{ f_{A_t} - f_{\ahat_t}}_{\mathcal{H}}\nonumber\\
& \leq \kappa \left(\sum_{t=1}^T I\{\ahat_t \neq A_t\} \right)^{1/p} \left(\sum_{t=1}^T \normx{f_{A_t} - f_{\ahat_t}}_\mathcal{H}^q\right)^{1/q}, \label{eq: holders_ineq}\\
 &\leq \kappa \left(\sum_{t=1}^T I\{\ahat_t \neq A_t\} \right)^{1/p} \left(T \sup_{\substack{a, a^\prime \in \mathcal{A} \\ a \neq a^\prime}}\normx{f_{a} - f_{a^\prime}}_\mathcal{H}^q\right)^{1/q} \nonumber\\
 & = \kappa T^{1/q}\left(\sum_{t=1}^T I\{\ahat_t \neq A_t\} \right)^{1/p} \sup_{\substack{a, a^\prime \in \mathcal{A} \\ a \neq a^\prime}}\normx{f_{a} - f_{a^\prime}}_\mathcal{H}, \label{eq: holders_randomization}
 \end{align} 
 where we obtain \eqref{eq: holders_ineq} by H\"older's inequality with $p$ and $q$ such that $\frac{1}{p} + \frac{1}{q} = 1$, for $p, q \in [1,\infty]$. 
 %Note that the sup term in \eqref{eq: holders_randomization} does not depend on the data and can be assumed to be a fixed quantity. 
 Next, we bound the first term in \eqref{eq: holders_randomization} below. By the law of iterated expectations, we obtain
\begin{align}
 \E[I\{\ahat_t \neq A_t\}] & = \E \left[ \E(I\{\ahat_t \neq A_t\} | \mathcal{F}_{t-1}, X_t) \right] 
 = \E[P(\ahat_t \neq A_t| \mathcal{F}_{t-1}, X_t)] = \dfrac{\epsilon_t}{L-1}. \label{eq: expectation_indicator}
 \end{align}
 Therefore, for any $\xi > 0$, Chebyshev's inequality yields
\begin{align}
& P\left(\left|\sum_{t=1}^T I\{\ahat_t  \neq A_t\}  - \sum_{t=1}^T \dfrac{\epsilon_t}{L-1}\right| \geq \xi \right) 
\leq \dfrac{1}{\xi^2} \E \left[ \sum_{t=1}^T \left(I\{\ahat_t  \neq A_t\} - \dfrac{\epsilon_t}{L-1} \right)\right]^2\nonumber\\ %\label{eq: chebychev_regret_proof3}\\
& = \dfrac{1}{\xi^2} \E \left[\sum_{t=1}^T \sum_{s=1}^T \left(I\{\ahat_t \neq A_t\} I\{\ahat_s \neq A_s\} - \dfrac{\epsilon_t}{L-1} I\{\ahat_s \neq A_s\} - \dfrac{\epsilon_s}{L-1} I\{\ahat_t \neq A_t\} + \dfrac{\epsilon_t \epsilon_s}{(L-1)^2} \right) \right]. \nonumber%\label{eq: full_mult_cheby}
\end{align}
Now, we simplify the expectation in the r.h.s. of %\eqref{eq: full_mult_cheby} 
the above inequality by considering cases: (i) $s = t$, (ii) $s < t$ and (iii) $s > t$ as 
\begin{align*}
P \left(\left|\sum_{t=1}^T I\{\ahat_t  \neq A_t\}  - \sum_{t=1}^T \dfrac{\epsilon_t}{L-1}\right| \geq \xi \right)\le \circled{\footnotesize{7}}+\circled{\footnotesize{8}}+\circled{\footnotesize{9}},
\end{align*}
where
\begin{align}
\circled{\footnotesize{7}}& := \dfrac{1}{\xi^2} \E \left[\sum_{t=1}^T  \left(1 - \dfrac{2 \epsilon_t}{L-1} \right) I\{\ahat_t \neq A_t\}  + \dfrac{\epsilon_t^2}{(L-1)^2}  \right] \nonumber\\
& \stackrel{\eqref{eq: expectation_indicator}}{=} \dfrac{1}{\xi^2}\sum_{t=1}^T \left[  \left(1 - \dfrac{2 \epsilon_t}{L-1} \right) \left(\dfrac{\epsilon_t}{L-1}\right) + \dfrac{\epsilon_t^2}{(L-1)^2}  \right] \nonumber\\
%\label{eq: Case1_rand_intermediate} \\
& = \dfrac{1}{\xi^2} \sum_{t=1}^T \dfrac{\epsilon_t}{L-1} - \dfrac{\epsilon_t^2}{(L-1)^2} = \dfrac{1}{\xi^2} \sum_{t=1}^T \dfrac{\epsilon_t}{L-1} \left[1 - \dfrac{\epsilon_t}{L-1} \right],  \nonumber%\label{eq:firstpart_Rand_cheby_break}
\end{align}
and
\begin{align}
\circled{\footnotesize{8}}& := \dfrac{1}{\xi^2} \sum_{t=2}^T \sum_{s=1}^{t-1} \left[ \E\big[ I\{\ahat_t \neq A_t\} I\{\ahat_s \neq A_s\} \big] - \dfrac{\epsilon_t}{L-1} \E \big[I\{\ahat_s \neq A_s\} \big]  - \dfrac{\epsilon_s}{L-1} \E\big[I\{\ahat_t \neq A_t\}\big]\right.\nonumber\\
&\qquad\qquad\qquad\left.+ \dfrac{\epsilon_t \epsilon_s}{(L-1)^2}  \right]\nonumber\\
&=0,\nonumber%\label{eq: secondpart_Rand_cheby_break}. 
\end{align}    
which follows from \eqref{eq: expectation_indicator} and for $s<t$,
\begin{align*}
\E\big[ I\{\ahat_t \neq A_t\} I\{\ahat_s \neq A_s\} \big] &= \E\big[\E( I\{\ahat_t \neq A_t\} I\{\ahat_s \neq A_s\} \mid \mathcal{F}_{t-1}, X_t)\big]\\
& = \E\big[ I\{\ahat_s \neq A_s\} P(\ahat_t \neq A_t| \mathcal{F}_{t-1}, X_t)\big]\\
& = \E \Big[I\{\ahat_s \neq A_s\} \times \dfrac{\epsilon_t}{L-1} \Big]\\
&\stackrel{\eqref{eq: expectation_indicator}}{=} \left(\dfrac{\epsilon_s}{L-1}\right) \left(\dfrac{\epsilon_t}{L-1}\right) = \E[I\{\ahat_t \neq A_t\}] \E[I\{\ahat_s \neq A_s\}].
\end{align*}
Through similar calculation, it follows that
\begin{align*}
\circled{\footnotesize{9}}&:=\dfrac{1}{\xi^2}  \left[\sum_{t=1}^{T-1} \sum_{s=t+1}^{T} \E[I\{\ahat_t \neq A_t\} I\{\ahat_s \neq A_s\}] - \dfrac{\epsilon_t}{L-1} \E[I\{\ahat_s \neq A_s\}] - \dfrac{\epsilon_s}{L-1} \E[I\{\ahat_t \neq A_t\}]\right.\\
&\qquad\qquad\qquad\left.+ \dfrac{\epsilon_t \epsilon_s}{(L-1)^2} \right]\\
&=0.
\end{align*}
Therefore, we obtain
% using the value of $\rm (I)$ in \eqref{eq:firstpart_Rand_cheby_break}, $\rm (II) = \rm (III) = 0$ in \eqref{eq: Rand_cheby_break}, we obtain
\begin{align*}
&P \left(\left|\sum_{t=1}^T I\{\ahat_t  \neq A_t\}  - \sum_{t=1}^T \dfrac{\epsilon_t}{L-1}\right| \geq \xi \right)  \leq \dfrac{1}{\xi^2} \sum_{t=1}^T \dfrac{\epsilon_t}{L-1} \left[1 - \dfrac{\epsilon_t}{L-1} \right] \leq \dfrac{1}{\xi^2} \sum_{t=1}^T \dfrac{\epsilon_t}{L-1},
\end{align*}
which implies for any $\delta>0$, with probability at least $1-\delta$,
% Setting the r.h.s. in the last inequality equal to $\delta$, we get that with probability at least $1-\delta$,
\begin{align}
\sum_{t=1}^T \dfrac{\epsilon_t}{L-1} - \left[\dfrac{1}{\delta} \sum_{t=1}^T \dfrac{\epsilon_t}{L-1} \right]^{1/2}\le \sum_{t=1}^T I\{\ahat_t  \neq A_t\} \leq \sum_{t=1}^T \dfrac{\epsilon_t}{L-1} + \left[\dfrac{1}{\delta} \sum_{t=1}^T \dfrac{\epsilon_t}{L-1} \right]^{1/2}. \label{eq: bound_first_term_Randy}
\end{align}
Using \eqref{eq: bound_first_term_Randy} in \eqref{eq: holders_randomization} yields that with probability at least $1-\delta$,
\begin{align}
\sum_{t = 1}^T |f_{A_t}(X_t) - f_{\ahat_t}(X_t)| 
% &\leq \kappa \left(\sum_{t=1}^T I\{\ahat_t \neq A_t\} \right)^{1/p} \left(T^{1/q} \sup_{\substack{a, a^\prime \in \mathcal{A} \\ a \neq a^\prime}}\normx{f_{a} - f_{a^\prime}}_\mathcal{H}\right) \nonumber \\
&\leq \kappa T^{1-\frac{1}{p}}\left[\sum_{t=1}^T \dfrac{\epsilon_t}{L-1} + \left\{\dfrac{1}{\delta} \sum_{t=1}^T \dfrac{\epsilon_t}{L-1}  \right\}^{1/2} \right]^{1/p}\sup_{\substack{a, a^\prime \in \mathcal{A} \\ a \neq a^\prime}}\normx{f_{a} - f_{a^\prime}}_\mathcal{H}. \label{eq:randomization_error}
\end{align}

Next, we construct an upper bound for the cumulative estimation error in the regret decomposition in \eqref{eq: regret_breakup}. By recalling $\mathcal{A}:=\{1,\ldots,L\}$,   consider
\begin{align}
\sup_{i\in \mathcal{A}}|(f_i(X_t) - \fh_{i}(X_t)) | & = \sup_{i\in\mathcal{A}}| \langle f_i - \fh_{i,t}, k(\cdot, X_t)\rangle_\mathcal{H}| 
\leq \kappa \sup_{i\in\mathcal{A}}\normx{f_i - \fh_{i,t}}_{\mathcal{H}}, \label{eq: estimation_to_instantregret}
\end{align}
where we will use Theorem~\ref{thm: theorem1} to bound 
\eqref{eq: estimation_to_instantregret}. To this end, by union bounding, we have
%using Theorem~\ref{thm: theorem1}. 
% We get that for each $i = 1,\hdots, L$, and for $\delta_i > 0$, if we choose $\{\epsilon_s\}_s$ such that $\sum_{s=1}^t \frac{1}{\epsilon_s} = o(t^2)$ and  $\lambda_{i,t}$ satisfies 
% \begin{align*}
% \lambda_{i,t} &= \left[ \dfrac{1}{\delta_i t^2} \left(\sum_{s=1}^t \dfrac{1}{\epsilon_s}\right) \right]^{\alpha/(2\gamma_i \alpha + \alpha + 1)},
% \end{align*}
%  then, with probability at least $1-\delta_i$,
% \begin{align}
% |(f_i(X_t) - \fh_{i}(X_t)) | \leq \kappa \max\{C_0, C_i\} \left(\dfrac{1}{\delta_i t^2} \left( \sum_{s=1}^t \frac{1}{\epsilon_s}\right)  \right)^{\gamma_i \alpha/ (2\gamma_i \alpha + \alpha + 1)}, \ i = 0,1, \hdots, L, \label{eq: thm_1_result}
% \end{align}
% where $C_0 = \sqrt{2\sigma^2 (L-1)}$ and $C_i = \normx{\Sigma^{-\gamma_i} f_i}$.
%  For a given time point $t$ and $\tilde{\delta} > 0$, we want to find $b_t$ such that the following holds.
% \begin{align*}
%      P\left(\sup_{i \in \{1,\hdots,L\}} |(f_i(X_t) - \fh_{i}(X_t)) | \geq  b_t \right) \leq \tilde{\delta}.
% \end{align*}
% We use a union bound to get,
\begin{align}
    & P\left(\sup_{i \in \mathcal{A}} \Vert f_i - \fh_{i,t} \Vert_\mathcal{H} \ge  b_t \right) \leq \sum_{i=1}^L P\left(\Vert f_i - \fh_{i,t}\Vert_\mathcal{H} \ge b_t \right) \nonumber\\
    &\le\left(\frac{1}{t^2}\sum^t_{s=1}\frac{1}{\epsilon_s}\right)\sum^L_{i=1}\left(\frac{2\sqrt{2}\max\{C_0,C_i\}}{b_t}\right)^{1/w_i}  \nonumber\\
    &\le L\left(\frac{1}{t^2}\sum^t_{s=1}\frac{1}{\epsilon_s}\right)\max_{i\in\mathcal{A}} \left(\frac{2\sqrt{2}\max\{C_0,C_i\}}{b_t}\right)^{1/w_i}  \nonumber\\
    &\le L\left(\frac{1}{t^2}\sum^t_{s=1}\frac{1}{\epsilon_s}\right) \max_{i\in\mathcal{A}} \left(\frac{\Theta}{b_t}\right)^{1/w_i}\nonumber\\
&\le L\left(\frac{1}{t^2}\sum^t_{s=1}\frac{1}{\epsilon_s}\right) \times\begin{cases}
\left(\frac{\Theta}{b_t}\right)^{\max_{i\in\mathcal{A}}1/w_i},& b_t< \Theta\\
\left(\frac{\Theta}{b_t}\right)^{\min_{i\in\mathcal{A}}1/w_i},& b_t\ge \Theta
\end{cases},\nonumber
    % \leq \sum_{i=1}^L \delta_i := \tilde{\delta}. \label{eq: union_bound_regret}
\end{align}
where $\Theta:=\max_{i\in\mathcal{A}} 2\sqrt{2}\max\{C_0,C_i\}$ and $w_i = {\gamma_i \alpha}/{(2 \gamma_i \alpha + \alpha + 1)}$. This means
with probability at least $1-\delta$,
\begin{align}
\sup_{i \in \mathcal{A}} \Vert f_i - \fh_{i,t} \Vert_\mathcal{H}\le \begin{cases}
\Theta \Delta_t^{\min_{i\in\mathcal{A}}w_i},& \Delta_t<1\\
\Theta \Delta_t^{\max_{i\in\mathcal{A}}w_i},& \Delta_t\ge 1\\
    \end{cases}=\begin{cases}
\Theta \Delta_t^{\frac{(\min_{i \in \mathcal{A}} \gamma_i) \alpha}{2(\min_{i \in \mathcal{A}} \gamma_i) \alpha + \alpha + 1}},& \Delta_t<1\\
\Theta \Delta_t^{\frac{(\max_{i \in \mathcal{A}} \gamma_i) \alpha}{2(\max_{i \in \mathcal{A}} \gamma_i) \alpha + \alpha + 1}},& \Delta_t\ge 1
    \end{cases},\label{Eq:tempii}
\end{align}
where $$\Delta_t:=\dfrac{L}{\delta t^2} \sum_{s=1}^t \dfrac{1}{\epsilon_s}$$ and used the fact that $h(x) = \frac{x \alpha}{2x\alpha + \alpha + 1}$ is a strictly  increasing function of $x$ for all $\alpha > 0$.
The result follows by using \eqref{Eq:tempii} in \eqref{eq: estimation_to_instantregret} and combining it with \eqref{eq:randomization_error} in \eqref{eq: regret_breakup}, while noting that $\lambda_{i,t}$ is given by the choice in \eqref{eq: Lthm_cond_lambda} in Theorem~\ref{thm: theorem1} but with $\delta$ replaced by $\delta/L$.

\subsection{Proof of Theorem~\ref{thm: regret_finite_case_pre}}
\label{proof: proof_thm4}
% Next, we restrict $\mathcal{H}$ to finite-dimensions and construct an upper bound for the regret for the kernel $\epsilon$-greedy algorithm.
%Note that, 
We bound the randomization error exactly as in the proof of Theorem~\ref{thm: regret_infinite_case_pre} in Section \ref{proof: proof_regret_infinite}. For bounding the cumulative estimation error, instead of Theorem~\ref{thm: theorem1}, we use the bound from Theorem~\ref{thm: theorem_finiteRKHS} in \eqref{eq: estimation_to_instantregret}. Then using the same union bounding idea as in the proof of Theorem~\ref{thm: regret_infinite_case_pre}, we obtain
\begin{align}
    & P\left(\sup_{i \in \mathcal{A}} \Vert f_i - \fh_{i,t} \Vert_\mathcal{H} \ge  b_t \right) \leq \sum_{i=1}^L P\left(\Vert f_i - \fh_{i,t}\Vert_\mathcal{H} \ge b_t \right) \nonumber\\
    &\le\left(\frac{1}{t^2}\sum^t_{s=1}\frac{1}{\epsilon_s}\right)\sum^L_{i=1}\left(\frac{4\max\{\tilde{C}_0,\tilde{C}_i\}}{b_t}\right)^{2}  
    \le L\left(\frac{1}{t^2}\sum^t_{s=1}\frac{1}{\epsilon_s}\right)\left(\frac{4\max\{\tilde{C}_0,\tilde{C}_*\}}{b_t}\right)^{2},  \nonumber
    % \leq \sum_{i=1}^L \delta_i := \tilde{\delta}. \label{eq: union_bound_regret}
\end{align}
which implies that with probability at least $1-\delta$,
\begin{align}
\sup_{i \in \mathcal{A}} \Vert f_i - \fh_{i,t} \Vert_\mathcal{H}\le 4\max\{\tilde{C}_0,\tilde{C}_*\} \left(\frac{L}{ t^2}\sum^t_{s=1}\frac{1}{\epsilon_s}\right)^{1/2}.\nonumber
% \begin{cases}
% \Theta \Delta_t^{\min_{i\in\mathcal{A}}w_i},& \Delta_t<1\\
% \Theta \Delta_t^{\max_{i\in\mathcal{A}}w_i},& \Delta_t\ge 1\\
%     \end{cases}=\begin{cases}
% \Theta \Delta_t^{\frac{(\min_{i \in \mathcal{A}} \gamma_i) \alpha}{2(\min_{i \in \mathcal{A}} \gamma_i) \alpha + \alpha + 1}},& \Delta_t<1\\
% \Theta \Delta_t^{\frac{(\max_{i \in \mathcal{A}} \gamma_i) \alpha}{2(\max_{i \in [L}] \gamma_i) \alpha + \alpha + 1}},& \Delta_t\ge 1
%     \end{cases},\label{Eq:tempii}
\end{align}
The result follows by using the above bound in \eqref{eq: estimation_to_instantregret} and combining it with \eqref{eq:randomization_error} in \eqref{eq: regret_breakup}.

\subsection{Proof of Theorem~\ref{theorem: finite_dim_with_Margin_condition}}
\label{proof: thm6}
 Since $L =2$, $a \in \{0,1\}$. It is important to note that we are working in the setting of a finite-dimensional RKHS. Recall, $A_s = \argmax_{a \in \{0,1\}} \hat{f}_{a, s-1}(X_s)$. Note that the regret in Definition \ref{def: Regret} can be written as
\begin{align*}
 R_T & = \sum_{s=1}^T |f_1(X_s) - f_0(X_s)| I\{\hat{a}_s \neq a^*_s\},
\end{align*}
where
\begin{align*}
I\{\hat{a}_s \neq a^*_s\} = I\{\hat{a}_s \neq a^*_s, a^*_s = A_s\} + I\{\hat{a}_s \neq a^*_s, a^*_s \neq A_s\}
\leq I\{\hat{a}_s \neq A_s\} + I \{A_s \neq a^*_s\}.
\end{align*}
Therefore,
\begin{align}
\E R_T &\leq \E \sum_{s=1}^T |f_1(X_s) - f_0(X_s)| I\{ \hat{a}_s \neq A_s\} + \E \sum_{s=1}^T |f_1(X_s) - f_0(X_s)| I\{ A_s \neq a^*_s\}\nonumber\\
&=R_T^{(1)}+R_T^{(2)}, \label{eq: expectedregretbreakupintoR1andR2}
\end{align}
where $$R_T^{(1)} := \E \sum_{s=1}^T |f_1(X_s) - f_0(X_s)| I\{ \hat{a}_s \neq A_s\}$$ is the error due to exploration (or randomization) and $$R_T^{(2)} := \E \sum_{s=1}^T |f_1(X_s) - f_0(X_s)| I\{ A_s \neq a^*_s\}$$ is the cumulative estimation error. Next, we bound these two terms as follows.
% Let us define
% \begin{align*}
% R_T^{(1)} := \E \sum_{s=1}^T |f_1(X_s) - f_0(X_s)| I\{ \hat{a}_s \neq A_s\}\\
% R_T^{(2)} := \E \sum_{s=1}^T |f_1(X_s) - f_0(X_s)| I\{ A_s \neq a^*_s\},
% \end{align*}
% where $R_T^{(1)}$ is the error due to exploration (or randomization) and $R_T^{(2)}$ is the cumulative estimation error. Let us first consider $R_T^{(2)}$.
\begin{align}
R_T^{(2)} &= \E \sum_{s=1}^T |f_1(X_s) - f_0(X_s)| I\{ A_s \neq a^*_s\} \nonumber\\
& = \sum_{s=1}^T \E \left[|f_1(X_s) - f_0(X_s)| I\{ A_s \neq a^*_s\}  \right] \nonumber\\
& = \sum_{s=1}^T \E\left[ \E \left[|f_1(X_s) - f_0(X_s)| I\{ A_s \neq a^*_s\} \Big| \mathcal{F}_{s-1} \right] \right]. \label{eq: R_T(2)_expression}
\end{align}
We only consider the inner expectation from here onwards and find an upper bound for that.
Let $\hat{f}_{s-1} = \hat{f}_{1, s-1} - \hat{f}_{0, s-1}$ and $f_{-} = f_1 - f_0$. We have that
\begin{align}
\E & \left[|f_1(X_s) - f_0(X_s)| I\{ A_s \neq a^*_s\} \Big| \mathcal{F}_{s-1} \right] \nonumber\\
&=  -\E \left[\left( I\{ \hat{f}_{s-1}(X_s) \geq 0 \} - I\{f_{-}(X_s) \geq 0 \}  \right) I\{f_{-}(X_s) \neq 0\} f_{-}(X_s) \Big|  \mathcal{F}_{s-1}\right]\ge 0. \label{eq: Margin_cond_exp_term1}
 \end{align}
% Note that the expectation in \eqref{eq: Margin_cond_exp_term1} is with respect to $X_s$ and is always non-negative. 
Similarly, we have that
\begin{align}
 \E\left[\left( I\{ \hat{f}_{s-1}(X_s) \geq 0 \} - I\{f_{-}(X_s) \geq 0 \}  \right) I\{f_{-}(X_s) \neq 0\} \hat{f}_{s-1}(X_s) \Big| \mathcal{F}_{s-1} \right] \geq 0. \label{eq: nonnegative_exp_term2}
\end{align}
Therefore using the fact that both \eqref{eq: Margin_cond_exp_term1} and \eqref{eq: nonnegative_exp_term2} are non-negative, we get
\begin{align}
  &\sum_{s=1}^T \E  \left[|f_1(X_s) - f_0(X_s)| I\{ A_s \neq a^*_s\} \Big| \mathcal{F}_{s-1} \right] \nonumber \\
  % &=\frac{1}{2}\sum_{s=1}^T \E \left[\left( I\{ \hat{f}_{s-1}(X_s) \geq 0 \} - I\{f_{-}(X_s) \geq 0 \}  \right) I\{f_{-}(X_s) \neq 0\}(\hat{f}_{s-1}(X_s) - f_{-}(X_s)) \Big|\mathcal{F}_{s-1}  \right]\nonumber\\
  &\leq \sum_{s=1}^T \E \left[\left( I\{ \hat{f}_{s-1}(X_s) \geq 0 \} - I\{f_{-}(X_s) \geq 0 \}  \right) I\{f_{-}(X_s) \neq 0\}(\hat{f}_{s-1}(X_s) - f_{-}(X_s)) \Big|\mathcal{F}_{s-1}  \right]\nonumber\\
  &= S_1+S_2,\label{eq: R2_margin_simplified}
\end{align}
% Now, we construct an upper bound for the r.h.s.~of \eqref{eq: R2_margin_simplified}.
% % \begin{align}
% % \sum_{s=1}^T  \E \left[\left( I\{ \hat{f}_{s-1}(x) \geq 0 \} - I\{f_{-}(x) \geq 0 \}  \right) I\{f_{-}(x) \neq 0\}(\hat{f}_{s-1} - f_{-})(x) \Big|\mathcal{F}_{s-1}  \right]. 
% % \end{align}
where for $\theta > 0$, 
% \begin{align}
% &\sum_{s=1}^T \E \left[\left( I\{ \hat{f}_{s-1}(x) \geq 0 \} - I\{f_{-}(x) \geq 0 \}  \right) I\{f_{-}(x) \neq 0\}(\hat{f}_{s-1} - f_{-})(x) \Big|\mathcal{F}_{s-1}  \right]=S_1+S_2, %\nonumber \\
% \label{eq: R2_breakuptoS1S2}
% \end{align}
% where
\begin{align*}S_1&:=\sum_{s=1}^T \E\left[ I\{0 < |f_{-}(X_s)| \leq T^{-\theta}\} \left( I\{ \hat{f}_{s-1}(X_s) \geq 0 \} - I\{f_{-}(X_s) \geq 0 \}  \right)\right.\\
&\qquad\qquad\qquad \left.\times I\{f_{-}(X_s) \neq 0\}(\hat{f}_{s-1} - f_{-})(X_s) \Big| \mathcal{F}_{s-1}\right],\end{align*}
and 
\begin{align*}
S_2&:=\sum_{s=1}^T  \E \left[ I\{ |f_{-}(X_s)| > T^{-\theta}\} \left( I\{ \hat{f}_{s-1}(X_s) \geq 0 \} - I\{f_{-}(X_s) \geq 0 \}  \right)\right.\\
&\qquad\qquad\qquad \left.\times I\{f_{-}(X_s) \neq 0\}(\hat{f}_{s-1} - f_{-})(X_s) \Big| \mathcal{F}_{s-1} \right].
\end{align*}
% &=  \underbrace{\sum_{s=1}^T \E\left[ I\{0 < |f_{-}(x)| \leq T^{-\theta}\} \left( I\{ \hat{f}_{s-1}(x) \geq 0 \} - I\{f_{-}(x) \geq 0 \}  \right) I\{f_{-}(x) \neq 0\}(\hat{f}_{s-1} - f_{-})(x) \Big| \mathcal{F}_{s-1}\right]}_{S_1} \nonumber \\
%  & \quad +  \underbrace{\sum_{s=1}^T  \E \left[ I\{ |f_{-}(x)| > T^{-\theta}\} \left( I\{ \hat{f}_{s-1}(x) \geq 0 \} - I\{f_{-}(x) \geq 0 \}  \right) I\{f_{-}(x) \neq 0\}(\hat{f}_{s-1} - f_{-})(x) \Big| \mathcal{F}_{s-1} \right]}_{S_2}
Note that
\begin{align}
S_1 
&\leq \sum_{s=1}^T  \E\left[ I\{0 < |f_{-}(X_s)| \leq T^{-\theta}\} (\hat{f}_{s-1} - f_{-})(X_s) \Big| \mathcal{F}_{s-1} \right] \nonumber \\
& \leq \sum_{s=1}^T C \kappa T^{-\theta} \normx{ \hat{f}_{s-1} - f_{-} }_\mathcal{H} \label{eq: applyMarginConditionThm5}\\
&\leq C \kappa T^{-\theta} \sum_{s=1}^T \normx{ \hat{f}_{s-1} - f_{-} }_\mathcal{H} 
\leq   C \kappa T^{-\theta} \sum_{s=1}^T \left[\normx{\hat{f}_{1,s-1} - f_1 }_\mathcal{H} + \normx{\hat{f}_{0,s-1} - f_0}_\mathcal{H}\right], \label{eq: S1_split}
\end{align}
where \eqref{eq: applyMarginConditionThm5} follows from \ref{assump: AssumptionMargin} and the fact that $\sup_{x\in\mathcal{X}}k(x,x) \leq \kappa$. The last inequality follows from the definition of $\hat{f}_{s-1}$ and $f_{-}$. 
Now, taking the expectation of $S_1$ in \eqref{eq: S1_split} we get that,
%  yielding
%  \begin{align*}
%  \E S_1 &\leq 2C\kappa T^{-\theta} \max_{i\in\{0,1\}}B(\tilde{C}_0,\tilde{C}_i,0,\eta)
%  %4 C \kappa \max\{\tilde{C}_0, \tilde{C}_*\} T^{-\theta} 
%  \sum_{t=1}^T \left[\dfrac{1}{t^2} \sum_{s=1}^t \dfrac{1}{\epsilon_s}\right]^{1/2},
% \end{align*}
% where the last inequality from \eqref{eq: estimation_bound_zeta_inexpectation} in Theorem \ref{thm: theorem_finiteRKHS}.
\begin{align}
    \E S_1 &\leq C \kappa T^{-\theta} \sum_{s=1}^T \left[\E\normx{\hat{f}_{1,s-1} - f_1 }_\mathcal{H} + \E \normx{\hat{f}_{0,s-1} - f_0}_\mathcal{H}\right]\nonumber\\
    & \le 2C\kappa T^{-\theta} \max_{i\in\{0,1\}}B(\tilde{C}_0,\tilde{C}_i,0,\eta)
 %4 C \kappa \max\{\tilde{C}_0, \tilde{C}_*\} T^{-\theta} 
 \sum_{t=1}^T \left[\dfrac{1}{t^2} \sum_{s=1}^t \dfrac{1}{\epsilon_s}\right]^{1/2},\label{eq: boundforS1}
\end{align}
where the last inequality from \eqref{eq: estimation_bound_zeta_inexpectation} in Theorem \ref{thm: theorem_finiteRKHS}.
% Now using \eqref{eq: estimation_bound_expectation}, we can bound each of the summands in the r.h.s. of the above inequality. For $\delta > 0$, we get that if we choose,
% \begin{align}
%  \lambda_{t} = \left[\dfrac{1}{\delta t^2} \left(\sum_{s=1}^t \dfrac{1}{\epsilon_s}\right) \right]^{\alpha/(2\gamma \alpha + \alpha + 1)},  \label{eq: choiceofLambdaS1S2}
% \end{align}   
% we obtain 
% \begin{align}
%     \E S_1 \leq 2 C \kappa  \max\{C_0, C_*\} \left(\dfrac{1}{1-w}\right) T^{-\theta} \sum_{t=1}^T \left[\dfrac{1}{t^2} \sum_{s=1}^t \dfrac{1}{\epsilon_s} \right]^{w}, \label{eq: boundforS1}
% \end{align}
% where $w = \gamma \alpha/(2\gamma \alpha + \alpha + 1)$, $C_0 = \sqrt{2\sigma^2}$ and $C_* = \max_{i \in \{0,1\}} \normx{\Sigma^{-\gamma} f_i}$.

% We know from \eqref{eq: estimation_bound_expectation} that each of these terms is $O_P(s^{(\beta - 1) w}), \ \text{where} \ w = \gamma \alpha/ (2 \gamma \alpha + \alpha + 1)$. Therefore we get that 
% \begin{align}
% S_1 = O_P(T^{-\theta + (\beta - 1) w + 1}). \label{eq: S1_bound}
% \end{align}
% Note that, for the integral approximation to hold, we want,
% \begin{align}
%     (\beta - 1) w + 1 > 0.  \label{eq: integral_approx_cond_S1}
% \end{align}
Next, to construct an upper bound for $S_2$, we use the fact that
\begin{align*}
I\{|(\hat{f}_{s-1} - f_{-}) (x) | > |f_{-}(x)| \} \geq I\{\hat{f}_{s-1}(x) \geq 0\} - I\{f_{-}(x) \geq 0\}.
\end{align*}
For $\zeta \ge 0$, we obtain
% We then have that, for $\zeta = \frac{\beta - (1-\beta)w}{2w (1- \beta) - \beta}$, for $ \frac{w}{1+w} < \beta < \frac{2w}{1 + 2w}$ (or $\frac{\gamma \alpha}{3 \gamma \alpha + \alpha + 1} < \beta < \frac{2 \gamma \alpha}{4 \gamma \alpha + \alpha + 1}$),
\begin{align}
S_2 
&\leq \sum_{s=1}^T \E\left[ I\{|f_{-}(X_s)| > T^{-\theta}\} I\{|(\hat{f}_{s-1} - f_{-})(X_s)| > |f_{-}(X_s)| \} | (\hat{f}_{s-1} - f_{-})(X_s)| \Big|\mathcal{F}_{s-1} \right] \nonumber\\
& \leq  \sum_{s=1}^T \E\left[ I\{|f_{-}(X_s)| > T^{-\theta}\} \dfrac{| (\hat{f}_{s-1} - f_{-})(X_s)|^{1+\zeta}}{|f_{-}(X_s)|^{\zeta}} \Big|\mathcal{F}_{s-1} \right] \nonumber\\
& \leq T^{\theta \zeta} \sum_{s=1}^T \E\left[ | (\hat{f}_{s-1} - f_{-})(X_s)|^{1 + \zeta} \Big|\mathcal{F}_{s-1} \right]\leq \kappa^{1 + \zeta} T^{\theta \zeta} \sum_{s=1}^T \normx{\hat{f}_{s-1} - f_{-}}_\mathcal{H}^{1+\zeta}. \label{eq: S2_bound}
\end{align}
Taking expectation of $S_2$, for the choice of $\lambda_{i,t}$ as in \eqref{eq: lamt_Thm6} and using \eqref{eq: estimation_bound_zeta_inexpectation},  we get that for $0 \le \zeta < 1$, 
\begin{align}
\E S_2  &\leq \kappa^{1+\zeta} T^{\theta \zeta} \sum_{s=1}^T \mathbb{E}\normx{\hat{f}_{s-1} - f_{-}}_\mathcal{H}^{1+\zeta}
\leq 2\max_{i\in\{0,1\}}B(\tilde{C}_0,\tilde{C}_i,\zeta,\eta) \kappa^{1 + \zeta} T^{\theta \zeta} \sum_{t=1}^T \left[\dfrac{1}{t} \sum_{s=1}^t \dfrac{1}{\epsilon_s}\right]^{(1+\zeta)/2}, \label{eq: boundonES2}
% \left(\dfrac{4}{1-\zeta}\right) \kappa^{1 + \zeta} T^{\theta \zeta} \max\{C_0, C_*\}^{1+\zeta} \sum_{t=1}^T \left(\dfrac{1}{t} \sum_{s=1}^t \dfrac{1}{\epsilon_s}\right)^{(1+\zeta)/2},
\end{align}
% \begin{align}
%     \E S_2 \leq 2\max_{i\in\{0,1\}}B(C_0,\Vert \Sigma^{-\gamma}f_i\Vert_\mathcal{H},\gamma,\zeta,\alpha) \kappa^{1+\zeta}  T^{\theta \zeta} \sum_{t=1}^T \left[\dfrac{1}{t^2} \sum_{s=1}^t \dfrac{1}{\epsilon_s} \right]^{w(1+\zeta)}, \label{eq: boundonES2}
%     % \E S_2 \leq 2^{1+\zeta} \kappa^{1+\zeta}  \max\{C_0, C_*\}^{1+\zeta} \left(\dfrac{1}{1-w(1+\zeta)}\right)T^{\theta \zeta} \sum_{t=1}^T \left[\dfrac{1}{t^2} \sum_{s=1}^t \dfrac{1}{\epsilon_s} \right]^{w(1+\zeta)}, \label{eq: boundonES2}
% \end{align}
% for $0\le \zeta < \frac{\gamma \alpha + \alpha + 1}{\gamma \alpha}$. 
% Note, since $\normx{\hat{f}_{s-1} - f_{-}}_\mathcal{H} = O_P(s^{(\beta - 1)w})$, we have that $\normx{\hat{f}_{s-1} - f_{-}}_\mathcal{H}^{1 + \zeta} = O_P(s^{(\beta - 1) w (1 + \zeta)})$. Then, with an integral approximation of the finite sum we get 
% \begin{align}
%     S_2 = O_P(T^{{(\beta -1)(1 + \zeta)w} + \theta \zeta + 1}). \label{eq: S2_bound}
% \end{align}
% Note that, for this integral approximation to hold we want,
% \begin{align}
%     (\beta-1)(1+\zeta) w + 1 > 0. \label{eq: integral_approx_cond_S2}
% \end{align}
Combining \eqref{eq: boundforS1}, \eqref{eq: boundonES2}, and \eqref{eq: R2_margin_simplified} in \eqref{eq: R_T(2)_expression}, we obtain 
%for the choice of $\lambda_t$ as in \eqref{eq: choiceofLambdaS1S2},
\begin{align}
    R^{(2)}_T& \le  \E S_1 + \E S_2 \nonumber\\
    & \lesssim 
    %A_0(\zeta, \kappa, C_0, C, C_*, \gamma, \alpha) 
    T^{-\theta} \sum_{t=1}^T  \left[\dfrac{1}{t^2} \sum_{s=1}^t \dfrac{1}{\epsilon_s} \right]^{1/2} + T^{\theta \zeta} \sum_{t=1}^T \left[\dfrac{1}{t^2} \sum_{s=1}^t \dfrac{1}{\epsilon_s} \right]^{(1+\zeta)/2}. \label{R2T_bound}
\end{align}
% where $A_0$ is a constant depending on the terms in the parenthesis, given by:
% \begin{align*}
%     A_0(\zeta, \kappa, C_0, C, C_*, \gamma, \alpha)& =  \Big(\frac{2^{1+\zeta}}{1-w(1+\zeta)}\Big) \max\Big\{C \kappa  \max\{C_0, C_*\}, \kappa^{1+\zeta} \max\{C_0, C_*\}^{1+\zeta}  \Big\}.
%\end{align*}
Now, we bound $R_T^{(1)}$ as
\begin{align}
R_T^{(1)} &= \E \sum_{s=1}^T |f_1(X_s) - f_0(X_s)| I\{\hat{a}_s \neq A_s\} \nonumber \\
& \leq \kappa \normx{f_1 - f_0}_\mathcal{H} \sum_{s =1 }^T \E I\{ \hat{a}_s \neq A_s\}
 \leq \kappa  \normx{f_1 - f_0}_\mathcal{H} \sum_{s =1 }^T P(\hat{a}_s \neq A_s) \nonumber\\
& = \kappa  \normx{f_1 - f_0}_\mathcal{H} \sum_{s =1 }^T  \dfrac{\epsilon_s}{2}. \label{eq: RT1_margin}
\end{align}
% Now, assuming that $\normx{f_1 - f_0}_\mathcal{H} < \infty$, and putting 
Combining \eqref{R2T_bound} and \eqref{eq: RT1_margin} in \eqref{eq: expectedregretbreakupintoR1andR2} yields the result.

\section*{Acknowledgments}
BKS is partially supported by National Science Foundation (NSF) CAREER Award DMS-1945396. 
\bibliography{KernelEpsGreedy, Refs}

\end{document}